\documentclass{article}
\usepackage[utf8]{inputenc}
\usepackage{geometry}
\usepackage{soul}
\usepackage{xfrac}
\usepackage{url}
\usepackage[small]{caption}
\usepackage{graphicx}
\usepackage{natbib}
\usepackage{amsmath}
\usepackage{booktabs}
\usepackage{algorithm}
\usepackage{algorithmic}
\usepackage{breakcites}
\usepackage{circledsteps}
\usepackage{bbm}
\usepackage{bm}

\usepackage{graphicx} % takes care of graphic including machinery
\usepackage{tcolorbox}
\usepackage{multirow}
\usepackage{amssymb}  
\usepackage{amsmath}  % improve math presentation
\usepackage{amsthm}
\usepackage{mathtools}
\usepackage[colorlinks, citecolor=blue]{hyperref}

\newcommand{\x}{{\bf x}}
\newcommand{\y}{{\bf y}}

\newcommand{\w}{{\bf w}}

\newcommand{\bu}{{\bf u}}
\newcommand{\bz}{{\bf z}}

\newcommand{\bepsilon}{{\bm \varepsilon}}
\newcommand{\bLambda}{{\bm \Lambda}}

\newcommand{\balpha}{{\bm \alpha}}
\newcommand{\bupsilon}{{\bm \upsilon}}
\newcommand{\bkappa}{{\bm \kappa}}
\newcommand{\bUpsilon}{{\bm \Upsilon}}
\newcommand{\bDelta}{{\bm \Delta}}
\newcommand{\btLambda}{{\widetilde{\bm \Lambda}}}
\newcommand{\btmu}{{\widetilde{\bm \mu}}}
\newcommand{\btu}{{\widetilde{\bf u}}}
\newcommand{\bttheta}{{\widetilde{\bm \theta}}}
\newcommand{\bPhi}{{\bm \Phi}}
\newcommand{\btPhi}{{\widetilde{\bm \Phi}}}
\newcommand{\bEpsilon}{\bm{\mathcal E}}

\newcommand{\bU}{{\bf U}}

\newcommand{\bX}{{\bf X}}
\newcommand{\bK}{{\bf K}}
\newcommand{\btK}{{\widetilde{\bf K}}}
\newcommand{\bV}{{\bf V}}
\newcommand{\bQ}{{\bf Q}}

\newcommand{\bA}{{\bf A}}
\newcommand{\bR}{{\bf R}}

\newcommand{\bv}{{\bf v}}

\newcommand{\br}{{\bf r}}

\newcommand{\ba}{{\bf a}}
\newcommand{\bone}{\mathbbm{1}}
\newcommand{\bI}{{\bf I}}
\newcommand{\bS}{{\bf S}}
\newcommand{\btS}{{\widetilde{\bf S}}}

\newcommand{\bsigma}{{\bm \sigma}}
\newcommand{\bGamma}{{\bf \Gamma}}
\newcommand{\btheta}{{\bm \theta}}

\newcommand{\bpi}{{\bm \pi}}
\newcommand{\bPi}{{\bm \Pi}}
\newcommand{\bmu}{{\bm \mu}}
\newcommand{\bphi}{{\bm \phi}}
\newcommand{\btphi}{\widetilde{\bm \phi}}

\newcommand{\bnu}{{\bm \nu}}
\newcommand{\bbP}{\mathbb P}
\newcommand{\bbR}{\mathbb R}

\newcommand{\bbE}{\mathbb E}
\newcommand{\cX}{\mathcal X}
\newcommand{\cH}{\mathcal H}

\newcommand{\ctO}{{\widetilde{\mathcal O}}}
\newcommand{\cN}{\mathcal N}
\newcommand{\cF}{\mathcal F}
\newcommand{\cW}{\mathcal W}
\newcommand{\cL}{\mathcal L}
\newcommand{\cA}{\mathcal A}

\newcommand{\ctZ}{\widetilde{\mathcal Z}}

\newcommand{\cO}{\mathcal O}
\newcommand{\cM}{\mathcal M}
\newcommand{\cV}{\mathcal V}
\newcommand{\cZ}{\mathcal Z}
\newcommand{\cC}{\mathcal C}

\newcommand{\cU}{\mathcal U}

\newcommand{\cS}{\mathcal S}

\newcommand{\fR}{\mathfrak R}
\newcommand{\fs}{\mathfrak s}

\DeclareMathOperator*{\argmin}{arg\,min}

\newtheorem{theorem}{Theorem}

\newtheorem{lemma}{Lemma}
\newtheorem{assumption}{Assumption}
\newtheorem{proposition}{Proposition}
\newtheorem{definition}{Definition}
\newtheorem{remark}{Remark}

\DeclareMathOperator*{\argmax}{arg\,max}
\usepackage{mathtools}

\begin{document}

\title{Provably Efficient Cooperative Multi-Agent Reinforcement Learning with Function Approximation}

\author{Abhimanyu Dubey and Alex Pentland\thanks{Media Lab and Institute for Data, Systems and Society, Massachusetts Institute of Technology. Corresponding email: \texttt{dubeya@mit.edu}.}}
\date{}
\maketitle
\begin{abstract}%
Reinforcement learning in cooperative multi-agent settings has recently advanced significantly in its scope, with applications in cooperative estimation for advertising, dynamic treatment regimes, distributed control, and federated learning. In this paper, we discuss the problem of cooperative multi-agent RL with function approximation, where a group of agents communicates with each other to jointly solve an episodic MDP.  We demonstrate that via careful message-passing and cooperative value iteration, it is possible to achieve near-optimal no-regret learning even with a fixed constant communication budget. Next, we demonstrate that even in heterogeneous cooperative settings, it is possible to achieve Pareto-optimal no-regret learning with limited communication. Our work generalizes several ideas from the multi-agent contextual and multi-armed bandit literature to MDPs and reinforcement learning.
\end{abstract}
    
\section{Introduction}
Cooperative multi-agent reinforcement learning (MARL) systems are widely prevalent in many engineering systems, e.g., robotic systems~\citep{ding2020distributed}, power grids~\citep{yu2014multi}, traffic control~\citep{bazzan2009opportunities}, as well as team games~\citep{zhao2019multi}. Increasingly, federated~\citep{yang2019federated} and distributed~\citep{peteiro2013survey} machine learning is gaining prominence in industrial applications, and reinforcement learning in these large-scale settings is becoming of import in the research community as well~\citep{zhuo2019federated, liu2019lifelong}. 

Recent research in the statistical learning community has focused on cooperative multi-agent decision-making algorithms with provable  guarantees~\citep{zhang2018fully, wai2018multi, zhang2018networked}. However, prior work focuses on algorithms that, while are decentralized, provide guarantees on convergence (e.g.,~\citet{zhang2018fully}) but no finite-sample guarantees for regret, in contrast to efficient algorithms with function approximation proposed for single-agent RL (e.g.,~\citet{jin2018q, jin2020provably, yang2020provably}). Moreover, optimization in the decentralized multi-agent setting is also known to be non-convergent without assumptions~\citep{tan1993multi}. Developing no-regret multi-agent algorithms is therefore an important problem in RL.

For the (relatively) easier problem of multi-agent multi-armed bandits, there has been significant recent interest in decentralized algorithms involving agents communicating over a network~\citep{landgren2016distributed, landgren2018social, martinez2018decentralized, dubey2020private}, as well as in the distributed settings~\citep{hillel2013distributed, wang2019distributed}. Since several application areas for distributed sequential decision-making regularly involve non-stationarity and contextual information~\citep{polydoros2017survey}, an MDP formulation can potentially provide stronger algorithms for these settings as well. Furthermore, no-regret algorithms in the single-agent RL setting with function approximation~(e.g., \citet{jin2020provably}) build on analysis techniques for contextual bandits, which leads us to the question -- \textit{Can no-regret function approximation be extended to (decentralized) cooperative multi-agent reinforcement learning?}

\textbf{Contributions}. In this paper, we answer the above question affirmatively for \textit{cooperative} multi-agent learning. Specifically, we study cooperative multi-agent reinforcement learning with \textit{linear} function approximation in two practical scenarios - the first being \textit{parallel} reinforcement learning~\citep{kretchmar2002parallel}, where a group of agents simultaneously solve isolated MDPs that are ``similar'' to each other, and communicate to facilitate faster learning. This corresponds to \textit{heterogenous} federated learning~\citep{li2018federated}, and generalizes transfer learning in RL~\citep{taylor2009transfer} to multiple sources~\citep{yao2010boosting}.

The second scenario we study is the heterogeneous multi-agent MDP, where a group of agents interact in an MDP by playing moves simultaneously, with the objective being to recover \textit{Pareto-optimal} multi-agent policies~\citep{tuyls112005evolutionary}, a more general notion of performance, compared to the standard objective of maximizing cumulative reward~\citep{boutilier1996planning}. Multi-agent MDPs are present in cooperative and distributed applications such as multi-agent robotics~\citep{yang2004multiagent,gupta2017cooperative}.

For each setting, we propose decentralized algorithms that are provably efficient with limited communication. Existing regret bounds for single-agent episodic settings scale as $\ctO(H^2\sqrt{d^3T})$ for $T$ episodes of length $H$\footnote{The $\ctO$ notation ignores logarithmic factors and failure probability, and $d$ is the dimensionality of the ambient feature space. See, e.g., \citet{yang2020provably} and \citet{jin2018q} for bounds.}, leading to a cumulative regret of $\ctO(MH^2\sqrt{d^3T})$ if $M$ agents operate in isolation. Similarly, a~\textit{fully centralized} agent running for $MT$ episodes will consequently obtain $\ctO(H^2\sqrt{d^3MT})$ regret. In comparison, for the \textit{parallel} setting, we provide a least-squares value iteration (LSVI) algorithm \texttt{\textbf{CoopLSVI}} which obtains a group regret of $\ctO((d+k)H^2\sqrt{(d+\chi)MT})$, where $\chi$ is a measure of heterogeneity between different MDPs, and $k$ is the minimum number of dimensions required to model the heterogeneity. When the MDPs are homogenous, our rate matches the \textit{centralized} single-agent regret. Moreover, our algorithm only requires $\cO(HM^3)$ rounds of communication, i.e., independent of the number of episodes $T$. The algorithm is a multi-agent variant of the popular upper confidence bound (UCB) strategies for reinforcement learning, where our key contributions are to first design an estimator that takes into account the bias introduced via heterogeneity between the different MDPs, and second, to strategically select episodes in which to synchronize statistics across agents without excessive communication.

For multi-agent MDPs, we introduce a variant of \texttt{\textbf{CoopLSVI}}, which attempts to recover the set of \textit{cooperative} Pareto-optimal policies, i.e., policies that cannot improve any individual agent's reward without decreasing the reward of the other agents~\citep{desai2018negotiable}. \texttt{\textbf{CoopLSVI}} obtains a cumulative \textit{Bayes regret} of $\ctO(H^2\sqrt{d^3T})$ over $T$ episodes, which is the first no-regret bound on learning Pareto-optimal policies. We use the method of \textit{random scalarizations}, a popular approach in multi-objective decision-making~\citep{van2014multi} coupled with optimistic least-squares value iteration, to provide a no-regret algorithm. Moreover, a direct corollary of our analysis is the \textit{first} no-regret algorithm for multi-objective RL~\citep{mossalam2016multi} with function approximation. 

\section{Related Work}
Our work is related to several areas of multi-agent learning. We discuss connections sequentially.

\textbf{Multi-Agent Multi-Armed Bandits}. The multi-agent bandit literature has seen a lot of interest recently, and given that our techniques build on function approximation techniques initially employed by the bandit community, many parallels can be drawn to our work and the multi-agent bandit literature. Several recent works have proposed no-regret algorithms. One line of work utilizes consensus-based averaging~\citep{martinez2018decentralized, landgren2016distributed, landgren2016distributed2, landgren2018social} which provide regret depending on statistics of the communication graph. For contextual bandits, similar algorithms have been derived that alternatively utilize message-passing algorithms~\citep{dubey2020kernel} or server-synchronization~\citep{wang2019distributed}. In the competitive multi-agent bandit setting, where agents must avoid collisions, algorithms have been proposed for distributed~\citep{liu2010distributed, liu2010distributed2, hillel2013distributed} and limited-communication~\citep{bistritz2018distributed} settings. Differentially-private algorithms have also been proposed~\citep{dubey2020private}.

\textbf{Reinforcement Learning with Function Approximation}. Our work builds on the body of recent work in (single-agent) reinforcement learning with function approximation. Classical work in this line of research, e.g.,~\citet{bradtke1996linear, melo2007q} provide algorithms, however, with no polynomial-time sample efficiency guarantees. In the presence of a simulator~\citet{yang2020reinforcement} provide a sample-efficient algorithm under linear function approximation. For the linear MDP assumption studied in this paper, our algorithms build on the seminal work of~\citep{jin2020provably}, that present an efficient (i.e., no-regret) algorithm. This research was further extended to kernel and neural function approximation in the recent work of~\citet{yang2020provably, wang2020provably}. Other approaches in this approximation setting are either computationally intractable~\citep{krishnamurthy2016pac, dann2018oracle, dong2020root} or require strong assumptions on the transition model~\citep{wen2017efficient}.

\textbf{Cooperative Multi-Agent Reinforcement Learning}. Cooperative multi-agent reinforcement learning has a very large body of related work, beginning from classical algorithms in the \textit{fully-}cooperative setting~\citep{boutilier1996planning}, i.e., when all agents share identical reward functions. This setting has been explored as multi-agent MDPs in the AI community~\citep{lauer2000algorithm, boutilier1996planning} and as \textit{team Markov games} in the control community~\citep{yoshikawa1978decomposition, wang2003learning}. However, the more general \textit{heterogeneous} reward setting considered in our work, where each agent may have unique rewards, corresponds to the \textit{team average} games studied previously~\citep{kar2013cal, zhang2018fully, zhang2018networked}. While some of these approaches do provide tractable algorithms that are decentralized and convergent, none consider the setting of linear MDPs with polynomial regret. Moreover, as pointed out in prior work (e.g., as studied in~\citet{szepesvari1999unified}), a centralized agent controlling each agent can converge to the optimal joint policy, which leaves only the sparse communication theoretically interesting. In our paper, however, we study a more general form of regret in order to discover multiple policies on the \textit{Pareto frontier}, instead of the single policy that maximizes team-average reward. We refer the readers to the illuminating survey paper by~\citet{zhang2019multi} for a detailed overview of algorithms in this setting.

\textbf{Parallel and Federated Reinforcement Learning}. Parallel reinforcement learning is a very relevant practical setting for reinforcement learning in large-scale and distributed systems, studied first in~\citep{kretchmar2002parallel}. A variant of the SARSA was presented for parallel RL in~\citet{grounds2005parallel}, that provides an efficient algorithm but with no regret guarantees. Modern deep-learning based approaches (with no regret guarantees) have been studied recently as well (e.g.,~\citet{clemente2017efficient, espeholt2018impala, horgan2018distributed, nair2015massively}). In the federated setting, which corresponds to a decentralized variant of parallel reinforcement learning, there has been recent interested from application domains as well~\citep{yu2020deep, zhuo2019federated}.

\textbf{Multi-Objective Sequential Decision-Making}. Our algorithms for multi-agent MDP build on recent work in multi-objective sequential decision-making. We extend the framework of obtaining Pareto-optimal \textit{reinforcement learning} policies from multi-objective optimization, presented in the work of~\citep{paria2020flexible}. We utilize a novel vector-valued noise concentration result from~\citet{chowdhury2020no}, which is an extension of the self-normalized martingale concentration for scalar noise presented in~\citep{chowdhury2017kernelized}, adapted to multi-objective gaussian process optimization. The framework of scalarizations has been studied both in the context of gaussian process optimization~\citep{knowles2006parego, zhang2007moea, zhang2009expensive} and multi-objective reinforcement learning~\citep{van2014multi}. 

\textbf{Notation}. We denote vectors by lowercase solid letters, i.e., $\x$, matrices by uppercase solid letters $\bX$, and sets by calligraphic letters, i.e., $\cX$. We denote the $\bS-$ellipsoid norm of a vector $\x$ as $\lVert \x \rVert_\bS = \sqrt{\x^\top\bS\x}$. We denote the interval $a, ..., b$ for $b \geq a$ by $[a, b]$ and by $[b]$ when $a=1$.
\section{\texttt{\textbf{CoopLSVI}} for Parallel MDPs}
\subsection{Parallel Markov Decision Processes}
Parallel MDPs (PMDPs,~\citep{sucar2007parallel, kretchmar2002parallel}) are a set of discrete time Markov decision processes that are executed in parallel, where a different agent interacts with any single MDP within the parallel MDP. Each agent interacts with their respective MDPs, each with identical (but disjoint) action and state spaces, but possibly unique reward functions and transition probabilities. We have a group of $M$ agents (denoted by $\cM$), where the MDP for any agent $m \in \cM$ is given by $\text{MDP}(\cS, \cA, H, \bbP_m, \br_m)$, where the state and action spaces are given by $\cS$ and $\cA$ respectively, the reward functions $\br_m = \{r_{m, h}\}_{h \in [H]}, r_{m, h} : \cS \times \cA \rightarrow [0, 1]$\footnote{We consider $r_{m, h}$ to be deterministic and bounded for simplicity. Our results can easily be extended to random rewards with sub-Gaussian densities.}, and transition probabilities $\bbP_m = \left\{\bbP_{m, h}\right\}_{h \in [H]} , \bbP_{m, h} : \cS \times \cA \rightarrow \cS$, i.e., $\bbP_{m, h}(x' | x, a)$ denotes the probability of the agent moving to state $x'$ if at step $h$ it selects action $a$ from state $x$. We assume that $\cS$ is measurable with possibly infinite elements, and that $\cA$ is finite with some size $A$. For any agent $m$, the policy $\pi_m$ is a set of $H$ functions $\pi_m = \{\pi_{m, h}\}_{m \in [M]}, \pi_{m, h} : \cS \rightarrow \cA$ such that $\sum_{a \in \cA} \pi_{m, h}(a|x) = 1 \ \forall \ x \in \cS$ and $\pi_{m, h}(a|x)$ is the probability of agent $m$ taking action $a$ from state $x$ at step $h$.

The problem proceeds as follows. At every episode $t = 1, 2, ...$, each agent $m \in \cM$ fixes a policy $\pi^t_m = \{\pi^t_{m, h}\}_{h \in [H]}$, and starts in an initial state $x^t_{m, 1}$ picked arbitrarily by the environment. For each step $h \in [H]$ of the episode, each agent observes its state $x^t_{m, h}$, selects an action $a^t_{m, h} \sim \pi^t_{m, h}(\cdot | x^t_{m, h})$, obtains a reward $r_{m, h}(x^t_{m, h}, a^t_{m, h})$, and transitions to state $x^t_{m, h+1}$ sampled according to $\bbP_{m, h}(\cdot | x^t_{m, h}, a^t_{m, h})$. The episode terminates at step $H+1$ where agents receive 0 reward. After termination, the agents can communicate among themselves via a server, if required. The performance of any policy $\pi$ in the $m^{th}$ MDP is measured by the value function $V^\pi_{m, h}(x) : \cS \rightarrow \bbR$, defined $\forall\ x \in \cS, h \in [H], m \in \cM$ as,
\begin{equation*}V^\pi_{m, h}(x) \triangleq \bbE_\pi\left[ \sum_{i=h}^{H} r_{m, i}(x_i, a_i) \ \Big| \ x_{m, h} = x\right].
\end{equation*}
The expectation is taken with respect to the random trajectory followed by the agent in the $m^{th}$ MDP under policy $\pi$. A related function $Q^\pi_{m, h} : \cS \times \cA \rightarrow \bbR$ determines the total expected reward from any action-state pair at step $h$ for the $m^{th}$ MDP for any state $x \in \cS$ and action $a \in \cA$:
\begin{equation*}
    Q^\pi_{m, h}(x, a) \triangleq \bbE_\pi\left[ \overset{H}{\underset{i=h}{\sum}} r_{m, i}(x_i, a_i) \ \Big| \ (x_{m, h}, a_{m, h}) = (x, a)\right].
\end{equation*}
Let $\pi^\star_m$ denote the optimal policy for the $m^{th}$ MDP, i.e., the policy that gives the maximum value, $V^\star_{m, h}(x) = \sup_{\pi}V^\pi_{m, h}(x)$, for all $x \in \cS, h \in [H]$. We can see that with the current set of assumptions, the optimal policy for each agent is possibly unique. For $T$ episodes, the cumulative \textit{group} regret (in expectation), is defined as,
\begin{equation*}\fR(T) \triangleq \underset{m \in \cM}{\sum} \overset{T}{\underset{t=1}{\sum}} \left[V^\star_{m, 1}(x^t_{m, 1}) -  V^{\pi_{m, t}}_{m, 1}(x^t_{m, 1})\right].
\end{equation*}
\subsection{Cooperative Least-Squares Value Iteration}
Our algorithms are a cooperative variant of linear least-squares value iteration with optimism~\citep{jin2020provably}. The prmary motivation behind our algorithm design is to strategically allow for communication between the agents such that with minimal overhead, we achieve a regret close to the single-agent case. Value iteration proceeds by obtaining the optimal Q-values $\{Q^\star_{m, h}\}_{h \in [H], m \in \cM}$ by recursively applying the Bellman equation. Specifically, each agent $m \in \cM$ constructs a sequence of action-value functions $\{Q_{m, h}\}_{h \in [H]}$ as, for each $x \in \cS, a \in \cA, m \in \cM$,
\begin{equation*}
Q_{m, h}(x, a) \leftarrow \left[r_{m, h} + \bbP_hV_{m, h+1}\right](x, a),\ V_{m, h+1}(x, a) \leftarrow \max_{a' \in \cA} Q_{m, h+1}(x, a').
\end{equation*}
Where $\bbP_hV(x, a) = \bbE_{x'}\left[V(x')\bbP(x' | x, a)\right]$. Our approach is to solve a linear least-squares regression usingproblem based on \textit{multi-agent} historical data. For any function class $\cF$, assume that any agent $m \in [M]$ has observed $k$ transition tuples $\{x^\tau_h, a^\tau_h, x^\tau_{h+1}\}_{\tau \in [k]}$ for any step $h \in [H]$. Then, the agent estimates the optimal Q-value for any step by LSVI, solving the regularized least-squares regression:
\begin{equation*}\label{eqn:q}
\widehat Q^t_{m, h} \leftarrow \underset{f \in \cF}{\argmin}\left\{\underset{\tau \in [k]}{\sum} \left[r_h(x^\tau_h, a^\tau_h) + V^t_{m, h+1}(x^\tau_{h+1}) - f(x^\tau_h, a^\tau_h)\right]^2 + \lVert f \rVert^2\right\}.
\end{equation*}
Here, the targets $y^\tau_h = r_h(x^\tau_h, a^\tau_h) + V^t_{m, h+1}(x^\tau_{h+1})$ denote the empirical value from specific transitions possessed by the agent, and $\lVert f \rVert$ denotes an appropriate regularization term based on the capacity of $f$ and the class $\cF$. To foster exploration, an additional bonus $\sigma^t_{m, h} : \cS \times \cA \rightarrow \bbR$ term is added that is inspired by the principle of optimism in the face of uncertainty, giving the final $Q$-value as,
\begin{equation}\label{eqn:q_ind}
Q^t_{m, h}(\cdot, \cdot) = \min\left\{ \left[\widehat Q^t_{m, h}+ \beta^t_{m, h}\sigma^t_{m, h}\right](\cdot, \cdot), H-h+1\right\},
\end{equation}
\begin{equation}\label{eqn:v_ind}
V^t_{m, h}(\cdot) = \max_{a \in \cA}Q^t_{m, h}(\cdot, a).
\end{equation}
Here $\{\beta^t_{m, h}\}_{t \in [T], m \in \cM}$ is an appropriately selected sequence. For episode $t$, we denote $\bpi^t = \{\pi^t_m\}_{m \in \cM}$ as the (joint) greedy policy with respect to the $Q$-values $\{ Q^t_{m, h}\}_{h \in [H]}$ for each agent. While this describes the multiagent LSVI algorithm abstractly for any general function class $\cF$, we first describe the \textit{homogeneous} parallel setting, where we assume $\cF$ to be linear in $d$ dimensions, called the \textit{linear} MDP assumption~\citep{bradtke1996linear, jin2020provably, melo2007q}.
\begin{definition}[Linear MDP, \citet{jin2020provably}]
\label{def:linear_mdp}
An $\text{MDP}(\cS, \cA, H, \bbP, R)$ is a linear MDP with feature map $\bphi : \cS \times \cA \rightarrow \bbR^d$, if for any $h \in [H]$, there exist $d$ unknown (signed) measures $\bmu_h = (\mu_h^1, ..., \mu_h^d)$ over $\cS$ and an unknown vector $\btheta_h \in \bbR^d$ such that for any $(x, a) \in \cS \times \cA$,
\begin{align*}
    \bbP_h(\cdot | x, a) = \langle \bphi(x, a), \bmu_h(\cdot)\rangle, r_h(x,a) = \langle \bphi(x, a), \btheta_h\rangle
\end{align*}
We assume, without loss of generality, $\lVert \bphi(x, a) \rVert \leq 1 \ \forall\ (x, a) \in \cS \times \cA$; $\max\left\{\lVert\bmu_h(\cS)\rVert, \lVert \btheta_h \rVert\right\} \leq \sqrt{d}$.
\end{definition}
\begin{algorithm}[t]
\caption{\texttt{\textbf{Coop-LSVI}}}
\small
\label{alg:ind_homo}
\begin{algorithmic}[1] %[1] enables line numbers
\STATE \textbf{Input}: $T, \bphi, H, S$, sequence $\beta_h = \{(\beta^t_{m, h})_{m, t}\}$.
\STATE \textbf{Initialize}: $\bS^t_{m, h}, \delta\bS^t_{m, h} = {\bf 0}, \cU^m_h, \cW^m_h = \emptyset$.
\FOR{episode $t=1, 2, ..., T$}
\FOR{agent $m \in \cM$}
\STATE Receive initial state $x^t_{m, 1}$.
\STATE Set $V^t_{m, H+1}(\cdot) \leftarrow 0$.
\FOR{step $h = H, ..., 1$}
\STATE Compute $\bLambda^t_{m, h} \leftarrow \bS^t_{m, h} + \delta\bS^t_{m, h}$.
\STATE Compute $\widehat Q^t_{m, h}$ and $\sigma^t_{m, h}$ (Eqns.~\ref{eqn:q_ind_homo} and ~\ref{eqn:s_ind_homo}).
\STATE Compute $Q^t_{m, h}(\cdot, \cdot)$ (Eqn.~\ref{eqn:q_ind})
\STATE Set $V^t_{m, h}(\cdot) \leftarrow \max_{a \in \cA}Q^t_{m, h}(\cdot, a)$.
\ENDFOR
\FOR{step $h=1, ..., H$}
\STATE Take action $a^t_{m, h} \leftarrow \argmax_{a \in \cA} Q^t_{m, h}(x^t_{m, h}, a)$.
\STATE Observe $r^t_{m, h}, x^t_{m, h+1}$.
\STATE Update $\delta\bS^t_{m, h} \leftarrow \delta\bS^t_{m, h} + \bphi(z^t_{m, h})\bphi(z^t_{m, h})^\top$.
\STATE Update $\cW^m_h \leftarrow \cW^m_h \cup (m, x, a, x')$.
\IF{$\log\frac{\det\left(\bS^t_{m, h} + \delta\bS^t_{m, h}+ \lambda\bI\right)}{ \det\left(\bS^t_{m, h}+ \lambda\bI\right)} > \frac{S}{\Delta t_{m, h}}$} \label{step:alg_ind_homo}
\STATE \textsc{Synchronize}$ \leftarrow $ \textsc{True}. 
\ENDIF
\ENDFOR
\ENDFOR
\IF{\textsc{Synchronize}}
\FOR{step $h = H, ..., 1$}
\STATE [$\forall$ \textsc{Agents}] Send $\cW^h_m \rightarrow$\textsc{Server}.
\STATE [\textsc{Server}] Aggregate $\cW^h \rightarrow \cup_{m \in \cM} \cW^m_h$.
\STATE [\textsc{Server}] Communicate $\cW^h$ to each agent.
\STATE [$\forall$ \textsc{Agents}] Set $\delta\bS^t_h \leftarrow 0, \cW^m_h \leftarrow \emptyset$.
\STATE [$\forall$ \textsc{Agents}] Set $\bS^t_h \leftarrow \bS^t_h + \sum_{z \in \cW^h} \bphi(z)\bphi(z)^\top$.
\STATE [$\forall$ \textsc{Agents}] Set $\cU^m_h \leftarrow \cU^m_h \cup \cW^m_h$
\ENDFOR
\ENDIF
\ENDFOR
\end{algorithmic}
\end{algorithm}

\subsection{Warm Up: Homogenous Parallel MDPs}
As a warm up, we describe \texttt{\textbf{CoopLSVI}} in the homogenous setting for simplicity first. In this case, the transition functions $\bbP_{m, h}$ and reward functions $r_{m, h}$ are identical for all agents and can be given by $\bbP_h$ and $r_h$ respectively for any episode $h \in [H]$. This environment corresponds to distributed applications, e.g., federated and concurrent reinforcement learning. Corresponding to Eq.~\ref{eqn:q}, we assume $\cF$ to be the class of linear functions in $d$ dimensions over the feature map $\bphi$, i.e., $f(\cdot) = \w^\top\bphi(\cdot), \w \in \bbR^d$, and set the ridge norm $\lVert \w \rVert_2^2$ as the regularizer. Furthermore, we fix a threshold constant $S$ that determines the amount of communication between the agents. The algorithm is summarized in Algorithm~\ref{alg:ind_homo}. In a nutshell, the algorithm operates by each agent executing a local linear LSVI and then synchronizing observations between other agents if the threshold condition is met every episode. Specifically, for each $t \in [T]$, each agent $m \in \cM$ obtains a sequence of value functions $\{Q^t_{m, h}\}_{h \in [H]}$ by iteratively performing linear least-squares ridge regression from the \textit{multi-agent} history available from the previous $t-1$ episodes. Assume that for any step $h$, the previous synchronization round occured after episode $k_t$. Then, the set of transitions available to agent $m \in \cM$ for any step $h$ before episode $t$ can be given by,
\begin{align*}
\cU^m_h(t) = \left\{ \cup \left( x^\tau_{n, h}, a^\tau_{n, h}, x^\tau_{n, h+1}\right)_{n \in \cM, \tau \in [k_{t}]}\right\} 
\bigcup \left\{ \cup \left( x^\tau_{m, h}, a^\tau_{m, h}, x^\tau_{m, h+1}\right)_{\tau=k_t+1}^{t-1}\right\}.
\end{align*}
Let $\psi^m_h(t)$ be an ordering of $\cU^m_{h}(t)$, and $U^m_h(t) = | \cU^m_h(t) |$. We have that $\cU^m_h$ is a set of $U^m_h(t)$ elements, where each element is a set of the form $(n, x, a, x')$, where $n \in \cM$ specifies the agent, and $(x, a, x')$ specifies a transition occuring at step $h$. Each agent $m$ first sets $Q_{m, H+1}^t$ to be ${\bm 0}_d$, and for $h = H, ..., 1$, iteratively solves $H$ regressions:
\begin{equation}
\widehat Q^t_{m, h} \leftarrow \underset{\w}{\argmin}\left\{\underset{(n, x, a, x') \in \cU^m_{h}(t)}{\sum} \left[r_h(x, a) + V^t_{m, h+1}(x')- \w^\top\bphi(x, a)\right]^2 + \lambda\lVert \w \rVert^2_2\right\}.
\end{equation}
Here $\lambda > 0$ is a regularizer. Next, $Q^t_{m, h}$ and $V^t_{m, h}$ are obtained via Equations~\ref{eqn:q_ind} and~\ref{eqn:v_ind}. We denote the targets $y_\tau = y_{m, h}(x_\tau, a_\tau, a'_\tau)$, and the features $\bphi_\tau = \bphi(x_\tau, a_\tau), \forall\ \tau \in \psi^m_h(t)$. Next, we denote the covariance $\bLambda^t_{m, h} \in \bbR^{d \times d}$ and bias $\bu^t_{m, h} \in \bbR^{d}$ as,
\begin{equation}\label{eqn:gram_ind_homo}
\bLambda^t_{m, h} =\sum_{\tau \in \psi^m_h(t)} \left[\bphi_\tau\bphi_\tau^\top\right] + \lambda\bI_d, \bu^t_{m,h} = \sum_{\tau \in \psi^m_h(t)} \left[\bphi_\tau y_\tau\right]
\end{equation}
Denote $\cZ = \cS \times \cA$ and $z = (x, a)$. We have $\forall\ z \in \cZ$,
\begin{equation}\label{eqn:q_ind_homo}
\widehat Q^t_{m, h}(z) = \bphi(z)^\top\left(\bLambda^t_{m, h}\right)^{-1}\bu^t_{m, h}.
\end{equation}
And correspondingly, the UCB bonus is given as,
\begin{equation}\label{eqn:s_ind_homo}
\sigma^t_{m, h}(z) = \left\lVert\bphi(z)\right\rVert_{(\bLambda^t_{m, h})^{-1}} = \sqrt{ \bphi(z)^\top(\bLambda^t_{m, h})^{-1}\bphi(z)}.
\end{equation}
This exploration bonus is similar to that of Gaussian process (GP) optimization~\citep{srinivas2009gaussian} and linear bandit~\citep{abbasi2011improved} algorithms, and it can be interpreted as the posterior variance of a Gaussian process regression. The motivation for adding the UCB term is similar to that in the bandit and GP case, to adequately overestimate the uncertainty in the ridge regression solution. When $\beta^t_{m, h}$ is appropriately selected, the $Q-$values overestimate the optimal $Q$-values with high probability, which is the foundation to bounding the regret.

The algorithm essentially involves this synchronization of transitions at carefully chosen instances. To achieve this, each agent maintains two sets of parameters. The first is $\bS_{m, h}^t$ which refers to the parameters that have been updated with the other agents via the synchronization, and the second is $\delta\bS_{m, h}^t$, which refers to the parameters that are not synchronized. Now, in each step of any episode $t$, after computing $Q-$values (Eq.~\ref{eqn:q_ind}), each agent executes the greedy policy with respect to the $Q-$values, i.e., $a^t_{m, h} = \argmax_{a \in \cA} Q^t_{m, h}(x^t_{m, h}, a)$, and updates the unsyncrhonized parameters $\delta\bS_{m, h}^t$. If any agent's new unsynchronized parameters satisfy the determinant condition with threshold $S$, i.e., if
\begin{align}\label{eqn:determinant_threshold}
    \log\frac{\det\left(\bS^t_{m, h} + \delta\bS^t_{m, h} + \lambda\bI_d\right)}{\det\left(\bS^{t}_{m, h} + \lambda\bI_d\right)} \geq \frac{S}{(t-k_t)},
\end{align}
then the agent signals a synchronization with the server, and messages are exchanged. Here $k_t$ denotes the episode after which the previous round of synchronization took place (step~\ref{step:alg_ind_homo} of Algorithm~\ref{alg:ind_homo}).  We can demonstrate the complexity of communication as a function of $S$.
\begin{lemma}[Communication Complexity]
\label{prop:communication_complexity}
If Algorithm~\ref{alg:ind_homo} is run with threshold $S$, then the total number of episodes with communication $n \leq 2H\sqrt{d(T/S)\log(MT)}+4H$.
\end{lemma}
The above result demonstrates that when $S = o(T)$, it is possible to ensure that the agents communicate only in a constant number of episodes, regardless of $T$. We now present the regret guarantee for the homogenous setting.
\begin{theorem}[Homogenous Regret]
\label{thm:ind_homo}
Algorithm~\ref{alg:ind_homo} when run on $M$ agents with communication threshold $S$, $\beta_t = \cO(H\sqrt{d\log (tMH)})$ and $\lambda = 1$ obtains the following cumulative regret after $T$ episodes, with probability at least $1-\alpha$,
\begin{equation*} 
\fR(T)=\widetilde\cO\left(dH^2\left(dM\sqrt{S} + \sqrt{dMT}\right)\sqrt{\log\left(\frac{1}{\alpha}\right)}\right).
\end{equation*}
\end{theorem}
\begin{remark}[Regret Optimality]
\normalfont
Theorem~\ref{thm:ind_homo} claims that appropriately chosen $\beta$ and $\lambda$ ensures sublinear group regret. Similar to the single-agent analysis in linear~\citep{jin2020provably} and kernel~\citep{yang2020provably} function approximation settings, our analysis admits a dependence on the (linear) function class via the $\ell_\infty-$covering number, which we simplify in Theorem~\ref{thm:ind_homo} by selecting appropriate values of the parameters. Generally, the regret scales as $\cO\left(H^2\left(M\sqrt{S} + \sqrt{MT\log\cN_\infty(\epsilon^\star)}\right)\sqrt{\log\left(\frac{1}{\alpha}\right)}\right)$, where $\cN_\infty(\epsilon)$ is the $\epsilon$-covering number of the set of linear value functions under the $\ell_\infty$ norm, and $\epsilon^\star = \cO(dH/T)$. We elaborate on this connection in the full proof. 
\end{remark}

\begin{remark}[Multi-Agent Analysis]
\normalfont
The aspect central to the multi-agent analysis is the dependence on the communication parameter $S$. If the agents communicate every round, i.e., $S = \cO(1)$, we observe that the cumulative regret is $\widetilde\cO(d^{\frac{3}{2}}H^2\sqrt{MT})$, matching the centralized setting. With no communication, the agents simply operate independently and the regret incurred is $\widetilde\cO(M\sqrt{T})$, matching the group regret incurred by isolated agents. Furthermore, with $S = \cO(T\log(MT)/dM^2)$ we observe that with a total of $\cO(dHM^3)$ episodes with communication, we recover a group regret of $\cO(d^{\frac{3}{2}}H^2\sqrt{MT}(\log MT))$, which matches the optimal rate (in terms of $T$) up to logarithmic factors.
\end{remark}

\subsection{Heterogeneous Parallel MDPs}
The above algorithm assumes that the underlying MDPs each agent interacts with are identical. We now present variants of this algorithm that generalizes to the case when the underlying MDPs are different. Note that even for heterogeneous settings, our algorithms require assumptions on the nature of heterogeneity in order to benefit from cooperative estimation.

\subsubsection{Small Heterogeneity}
The first heterogeneous setting we consider is when the deviations between MDPs are much smaller than the horizon $T$, which allows Algorithm~\ref{alg:ind_homo} to be no-regret as long as an upper bound on the heterogeneity is known. We first present this ``small deviation'' assumption.
\begin{assumption}
\label{assumption:bellman_small_deviation}
For any $\xi = o(T^{-\alpha}) < 1, \alpha > 0$, a parallel MDP setting demonstrates ``small deviations'' if for any $m, m' \in \cM$, the corresponding linear MDPs defined in Definition~\ref{def:linear_mdp} obey the following for all $(x, a) \in \cS \times \cA$:
\begin{align*}
    \left\lVert (\bbP_{m, h} - \bbP_{m', h})(\cdot | x, a) \right\rVert_{\textsc{TV}} \leq \xi, \text{ and }\left| (r_{m, h} - r_{m', h})(x, a)\right| \leq \xi.
\end{align*}
\end{assumption}

Under this assumption, when an upper bound on $\xi$ is known, we do not have to modify Algorithm~\ref{alg:ind_homo} as the confidence intervals employed by \texttt{\textbf{CoopLSVI}} are robust to small deviations. We formalize this with the following regret bound.
\begin{theorem}
\label{thm:ind_hetero_small}
    Algorithm~\ref{alg:ind_homo} when run on $M$ agents with parameter $S$ in the small deviation setting (Assumption~\ref{assumption:bellman_small_deviation}), with $\beta_t = \cO(H\sqrt{d\log (tMH)} + \xi\sqrt{dMT})$ and $\lambda = 1$ obtains the following cumulative regret after $T$ episodes, with probability at least $1-\alpha$,
\begin{equation*} 
\fR(T) = \ctO\left(dH^2\left(dM\sqrt{S} + \sqrt{dMT} \right)\left(\sqrt{\log\left(\frac{1}{\alpha}\right)} + 2\xi\sqrt{dMT}\right)\right).
\end{equation*}
\end{theorem}

\begin{remark}[Comparison with Misspecification]
\normalfont
While this demonstrates that \texttt{\textbf{CoopLSVI}} is robust to small deviations in the different MDPs, the analysis can be extended to the case when the MDPs are ``approximately'' linear, as done in~\citet{jin2020provably} (Theorem 3.2). A key distinction in the above result and the standard bound in the misspecification setting is that in the general misspecified linear MDP there are two aspects to the anlaysis - the first being the (adversarial) error introduced from the linear approximation, and the second being the error introduced by executing a policy following the misspecified linear approximation. In our case, the second term does not exist as the policy is valid within each agents' own MDP, but the misspecification error still remains.
\end{remark}

\subsubsection{Large Heterogeneity}
For the large heterogeneity case, in order to transfer knowledge from a different agents' MDP, we assume that each agent $m \in \cM$ possesses an additional contextual description $\bkappa(m) \in \bbR^k$ that describes the heterogeneity linearly. We state this formally below.

\begin{definition}[Heterogenous Linear MDP]
\label{def:bellman_hetero}
A heterogeneous parallel $\text{MDP}(\cS, \cA, H, \bbP, R)$ is a set of linear MDPs with two feature maps $\bphi : \cS \times \cA \rightarrow \bbR^d$ and $\bkappa : \cM \rightarrow \bbR^k$, if for any $h \in [H]$, there exist $d$ unknown (signed) measures $\bmu_h = (\mu_h^1, ..., \mu_h^d)$ over $\cS$, an unknown vector $\btheta_h \in \bbR^d$, $k$ unknown (signed) measures $\bnu_{h} = (\nu_{ h}^1, ..., \nu_{h}^k)$ over $\cS$ and an unknown vector $\balpha_{h} \in \bbR^k$ such that for any $(x, a) \in \cS \times \cA$ and $m \in \cM$,
\begin{equation*}
\bbP_{m, h}(\cdot | x, a) = \begin{bmatrix}
    \bphi(x, a) \\
    \bkappa(m)
    \end{bmatrix}^\top
    \begin{bmatrix}
    \bmu_h(\cdot) \\
    \bnu_h(\cdot)
    \end{bmatrix},
    r_{m, h}(x,a) = \begin{bmatrix}
    \bphi(x, a) \\
    \bkappa(m)
    \end{bmatrix}^\top
    \begin{bmatrix}
    \btheta_h \\
    \balpha_h
    \end{bmatrix}.
\end{equation*}
We denote the combined features via the shorthand:
\begin{align*}
    \btphi(m, x, a) = \left[\bphi(x, a)^\top, \bkappa(m)^\top\right]^\top, \btmu_h(\cdot) = \left[\bmu_h(\cdot)^\top, \bnu_h(\cdot)^\top\right]^\top, \bttheta_h = \left[\btheta_h^\top, \balpha_h^\top\right]^\top.
\end{align*}
 We assume, WLOG, that $\lVert \btphi(m, x, a) \rVert \leq 1 \ \forall\ (m, x, a) \in \cM \times \cS \times \cA$, $\max\left\{\lVert\bmu_h(\cS)\rVert, \lVert \btheta_h \rVert\right\} \leq \sqrt{d}$, and  $\max\left\{\lVert\bnu_h(\cS)\rVert, \lVert \balpha_h \rVert\right\} \leq \sqrt{k}$.
\end{definition}

The above assumption is identical to the linear MDP assumption (Definition~\ref{def:linear_mdp}) except that the transition and reward functions are also a function of contextual information possessed by the agent about its environment. Such information is usually present in many applications, e.g., as demonstrated in~\citet{krause2011contextual} for the contextual bandit. Concretely, we assume that for each MDP, the discrepancies between both the transition and reward functions can be explained as a linear function of the underlying agent-specific contexts, i.e., $\bkappa$, which is independent of the state and action pair $(x, a)$\footnote{Intricate models can be assumed that exploit interdependence in a sophisticated manner (see, e.g., \citet{dubey2020kernel, deshmukh2017multi}), however, we leave that for future work.}. Then, each agent predicts an \textit{agent-specific} $Q$ and value function.

Specifically, for each $t \in [T]$, each agent $m \in \cM$ obtains a sequence of value functions $\{Q^t_{m, h}\}_{h \in [H]}$ by iteratively performing linear least-squares ridge regression from the \textit{multi-agent} history available from the previous $t-1$ episodes, but in contrast to the homogenous case, it now learns a $Q-$function over $\cM \times \cS \times \cA$ and value function over $\cM \times \cA$. Each agent $m$ first sets $Q_{m, H+1}^t$ to be a zero function, and for any $h \in [H]$, solves the regression problem in $\bbR^{d+k}$ to obtain $Q-$values.
\begin{equation}\label{eqn:qhat_ind_hetero}
\widehat Q^t_{m, h} \leftarrow \underset{\w \in \bbR^{d+k}}{\argmin}\left\{\underset{(n, x, a, x') \in \cU^m_{h}(t)}{\sum} \left[r_{n, h}(x, a) + V^t_{m, h+1}(n, x')- \w^\top\btphi(n, x, a)\right]^2 + \lambda\lVert \w \rVert^2_2\right\}.
\end{equation}
Here, $\lambda > 0$ is a regularizer, and $n \in \cM$ denotes the agent whose $Q$ function the agent $m$ is estimating, and $V_{m, h+1}^t(n, x) = \max_{a \in \cA} Q(n, x, a)$, where the $Q$ values are given by, for any $n, x, a$,
\begin{equation*}
    Q(n, x, a) = \widehat Q^t_{m, h}(n, x, a) + \beta^t_{m, h}\lVert \btphi(n, x, a) \rVert_{\btLambda^t_{m, h}}.
\end{equation*}
Here as well, we have the least-squares solution described as follows.  Let $\psi^m_h(t)$ be an ordering of $\cU^m_{h}(t)$, and $U^m_h(t) = | \cU^m_h(t) |$. Then, we denote the targets $y_\tau = y_{m, h}(n_\tau, x_\tau, a_\tau, a'_\tau)$, and the features $\btphi_\tau = \btphi(n_\tau, x_\tau, a_\tau), \forall\ \tau \in \psi^m_h(t)$. Next, we denote the covariance $\btLambda^t_{m, h} \in \bbR^{d+k \times d+k}$ and bias $\btu^t_{m, h} \in \bbR^{d+k}$ as,
\begin{equation}\label{eqn:gram_ind_hetero}
\btLambda^t_{m, h} =\sum_{\tau \in \psi^m_h(t)} \left[\btphi_\tau\btphi_\tau^\top\right] +\lambda\bI_{d+k}, \btu^t_{m,h} = \sum_{\tau \in \psi^m_h(t)} \left[\btphi_\tau y_\tau\right].
\end{equation}
Let $\ctZ = \cM \times \cS \times \cA$ and $\tilde z = (n, x, a)$. We have $\forall\ \tilde z \in \ctZ$,
\begin{equation}\label{eqn:q_ind_hetero}
\widehat Q^t_{m, h}(\tilde z) = \btphi(\tilde z)^\top\left(\btLambda^t_{m, h}\right)^{-1}\btu^t_{m, h}.
\end{equation}
At any episode $t$ and step $h$, each agent $m$ then follows the greedy policy with respect to $Q(m, x^t_{m, h})$. The remainder of the algorithm is identical to Algorithm~\ref{alg:ind_homo}, and is presented in Algorithm~\ref{alg:ind_hetero} in the appendix. To present the regret bound, we first define coefficient of heterogeneity $\chi$, and then present the regret bound in terms of this coefficient.
\begin{definition}
\label{def:coff_heterogeneity}
For the parallel MDP defined in Definition~\ref{def:bellman_hetero}, let $\bK^\kappa_h = \left[\bnu_h(m)^\top\bnu_h(m')\right]_{m, m' \in \cM}$ be the Gram matrix of agent-specific contexts. The coefficient of heterogeneity is defined as $\chi = \max_{h \in [H]}\text{rank}(\bK^\kappa_h) \leq k$.
\end{definition}
\begin{theorem}
\label{thm:ind_hetero_large}
Algorithm~\ref{alg:ind_hetero} when run on $M$ agents with parameter $S$ in the heterogeneous setting (Definition~\ref{def:bellman_hetero}), with $\beta_t = \cO(H\sqrt{(d+k)\log (tMH)})$ and $\lambda = 1$ obtains the following cumulative regret after $T$ episodes, with probability at least $1-\alpha$,
\begin{align*}
\fR(T)=\ctO\left((d+k)H^2\left(M(d +\chi)\sqrt{S} + \sqrt{(d +\chi)MT} \right)\sqrt{\log\left(\frac{1}{\alpha}\right)} \right).
\end{align*}
\end{theorem}
\begin{remark}[Optimality Discussion]
\normalfont
Heterogeneous \texttt{\textbf{CoopLSVI}} regret is bounded by the similarity in the agents' MDPs. In the case when the agents' have identical MDPs, $\chi = 1$, which implies that the heterogenous variant has a worse regret by a factor of $(1+\frac{k}{d})\sqrt{1+\frac{1}{d}}$, which arises from the fact that we use a model that lies in $\bbR^{d+k}$. Nevertheless, this suboptimality is indeed an artifact of our regret analysis, particularly introduced by the covering number of linear functions in $\bbR^{d+k}$, and future work can address modifications to ensure tightness. Alternatively, in the worst case, $\chi = k$, which matches the linear parallel MDP in $d+k$ dimensions, which ensures that no suboptimality has been introduced by the heterogeneous analysis. Under this model however, one can only observe improvements when $k = o(\sqrt{M})$ suffices for modeling the heterogeneity within the parallel MDP.
\end{remark}

\section{\texttt{CoopLSVI} for Multiagent MDPs}
The next environment we consider is the simultaneous-move multi-agent MDP~\citep{boutilier1996planning}, which is an extension of an MDP to multiple agents~\citep{xie2020learning}.

\subsection{Heterogeneous Multi-agent MDPs} 
A multi-agent MDP (MMDP) can be formally described as $\text{MMDP}(\cS, \cM, \cA, H, \bbP, \bR)$, where the set of agents $\cM$ is finite and countable with cardinality $M$, the state and action spaces are factorized as $\cS = \cS_1 \times \cS_2 \times \dots \cS_M$ and $\cA = \cA_1 \times \cA_2 \times \dots \cA_M$, where $\cS_i$ and $\cA_i$ denote the individual state and action space for agent $i$ respectively. Furthermore, the transition matrix $\bbP = \{\bbP_h\}_{h \in [H]}$ depends on the joint action-state configuration for all agents, i.e., $\bbP_h : \cS \times \cA \rightarrow \cS$, and so does the (vector-valued) reward function $\bR = \{\br_h\}_{h \in [H]}, \br_h : \cS \times \cA \rightarrow \bbR^M$.

The MMDP proceeds as follows. In each episode $t \in [T]$ each agent fixes a policy $\pi^t_m = \{\pi^t_{m, h}\}_{h \in [H]}$ in a common initial state $\x^t_{1} = \{x^t_{m, 1}\}_{h \in [H]}$ picked arbitrarily by the environment. For each step $h \in [H]$ of the episode, each agent observes the overall state $\x^t_{h}$, selects an individual action $a^t_{m, h} \sim \pi^t_{m, h}(\cdot | \x^t_{h})$ (denoted collectively as the joint action $\ba^t_h = \{a^t_{m, h}\}_{m \in \cM}$), and obtains a reward $r_{m, h}(\x^t_{h}, \ba^t_{h})$ (denoted collectively as the joint reward $\br_h(\x^t_{h}, \ba^t_{h}) = \{r_{m, h}(\x^t_{h}, \ba^t_{h})\}_{m \in \cM}$). All agents transition subsequently to a new joint state $\x^t_{h+1} = \{x^t_{m, h+1}\}_{m \in \cM}$ sampled according to $\bbP_{h}(\cdot | \x^t_{h}, \ba^t_{h})$. The episode terminates at step $H+1$ where all agents receive no reward. After termination, the agents communicate by exchanging messages with a server, prior to the next episode. Let $\bpi = \{\bpi_h\}_{h \in [H]}, \bpi_h = \{\pi_{m, h}\}_{m \in \cM}$ denote the joint policy for all $M$ agents. We can define the vector-valued value function over all states $\x \in \cS$ for a policy $\bpi$ as,
\begin{align*}
    \bV^\bpi_h(\x) \triangleq \bbE_\bpi\left[ \sum_{i=h}^H \br_i(\x_i, \ba_i) \ \Big| \ \x_h = \x \right].
\end{align*}
One can define the analogous vector-valued $Q$-function for a policy $\bpi$ and any $\x \in \cS, \ba \in \cA, h \in [H]$,
\begin{equation*}
\bQ^\bpi_h(\x, \ba) \triangleq \br_h(\x, \ba) + \bbE_\bpi\left[ \sum_{i=h+1}^H \br_i(\x_i, \ba_i) \ \Big| \ \x_h = \x, \ba_h = \ba \right].
\end{equation*}

Without assumptions, a general MMDP represents a wide variety of environments, including competitive stochastic games as well as cooperative games~\citep{shapley1953stochastic}. In this paper, our focus is to consider cooperative learning in the \textit{fully-observable} setting, where the complete state and actions (but \textit{not} rewards) are visible to all agents. 

\begin{remark}[Multi-agent Environments]
\normalfont
MMDPs have been used to model a variety of decision processes, and they are an instance of \textit{stochastic games}~\citep{shapley1953stochastic}, and are most closely related to the general framework for \textit{repeated games}~\citep{myerson1982optimal}. Repeated games are in turn generalizations of partially observable MDPs (POMDPs,~\cite{aastrom1965optimal}), and involve a variety of distinct challenges beyond communication and non-stationarity. For these multi-agent environments it is not possible to characterize optimal behavior without global objectives, as, unlike the single-agent setting, achieving a large cumulative reward is typically at odds with individually optimal behavior.
\end{remark}
We focus on recovering the set of \textit{cooperative} Pareto-optimal policies, i.e., policies that cannot improve any individual agent's reward without decreasing the reward of the other agents. Formally,
\begin{definition}[Pareto domination,~\citet{paria2020flexible}]
A (multi-agent) policy $\bpi$ Pareto-dominates another policy $\bpi'$ if and only if $\bV^\bpi_1(\x) \succeq \bV^{\bpi'}_1(\x) \forall \ \x \in \cS$. Consequently, a policy is Pareto-optimal if it is not Pareto-dominated by any policy. We denote the set of possible policies by $\bPi$, and the set of Pareto-optimal policies by $\bPi^\star$. As an example, it is evident that the policies that maximize any agent's individual reward as well as the average reward are all elements of $\bPi^\star$. 
\end{definition}
Our objective is to design an algorithm that recovers $\bPi^\star$ with high probability. 
\subsection{LSVI with Random Scalarizations}
Our approach for multiagent MDPs, inspired by multi-objective optimization, is to utilize the method of \textit{random scalarizations}~\citep{knowles2006parego, paria2020flexible}, where we define a probability distribution over the Pareto-optimal policies by converting the vector-valued function $\bV$ to a parameterized scalar measure. Consider a \textit{scalarization function} $\fs_\bupsilon(\x) = \bupsilon^\top\x : \bbR^M \rightarrow \bbR$ parameterized by $\bupsilon$ belonging to  $\bUpsilon \subseteq \Delta^M$ (unit simplex in $M$ dimensions). We then have the \textit{scalarized} value function $V^\bpi_{\bupsilon, h}(x) : \cS \rightarrow \bbR$ and $Q-$function $Q^\bpi_{\bupsilon, h} : \cS \times \cA \rightarrow \bbR$ for some joint policy $\bpi$ as
\begin{equation}\label{eqn:v_scalarized}
    V^\bpi_{\bupsilon, h}(\x) \triangleq \fs_\bupsilon(\bV^\bpi_h(\x)) = \bupsilon^\top\bV^\bpi_h(\x)\text{ , and } Q^\bpi_{\bupsilon, h}(\x, \ba) \triangleq \fs_\bupsilon(\bQ^\bpi_h(\x, \ba)) = \bupsilon^\top\bQ^\bpi_h(\x, \ba).
\end{equation}
Since both $\cA = \prod_i \cA_i$ and $H$ are finite, there exists an optimal multi-agent policy for any fixed scalarization $\bupsilon$, which gives the optimal value $V^\star_{\bupsilon, h} = \sup_{\bpi \in \bPi} V^\bpi_{\bupsilon, h}(\x)$ for all $\x \in \cS$ and $h \in [H]$. The primary advantage of considering a scalarized formulation is that for any fixed scalarization parameter, an optimal \textit{joint} policy exists, which coincides with the optimal policy for an MDP defined over the joint action space $\cS \times \cA$, given as follows.
\begin{proposition} 
\label{prop:mg_bellman_optimal}
For the scalarized value function given in Equation~\ref{eqn:v_scalarized}, the Bellman optimality conditions are given as, for all $h \in [H], \x \in \cS, \ba \in \cA, \bupsilon \in \bUpsilon$,
\begin{align*}
    &Q^\star_{\bupsilon, h}(\x, \ba) = \fs_\bupsilon\br_h(\x, \ba) + \bbP_h V^\star_{\bupsilon, h}(\x, \ba), V^\star_{\bupsilon, h}(\x) = \max_{\ba \in \cA} Q^\star_{\bupsilon, h}(\x, \ba)\text{, and } V^\star_{\bupsilon, H+1}(\x) = 0.
\end{align*}
\end{proposition}
The optimal policy for any fixed $\bupsilon$ is given by the greedy policy with respect to the Bellman-optimal scalarized Q values. We denote this (unique) optimal policy by $\pi^\star_{\bupsilon}$. We now demonstrate that each $\bpi^\star_\bupsilon$ lies in the Pareto frontier. 
\begin{proposition}
\label{prop:mg_pareto_frontier}
For any parameter $\bupsilon \in \bUpsilon$, the optimal greedy policy $\bpi^\star_\bupsilon$ with respect to the scalarized $Q$-value that satisfies Proposition~\ref{prop:mg_bellman_optimal} lies in the Pareto frontier $\bPi^\star$.
\end{proposition}
The above result claims that by ``projecting'' an MMDP to an MDP via scalarization, one can recover a policy on the Pareto frontier. Indeed, if the set of policies $\bPi^\star_\bUpsilon = \{\bpi^\star_\bupsilon | \bupsilon \in \bUpsilon\}$ spans $\bPi^\star$, one can recover $\bPi^\star$ by simply learning $\bPi^\star_\bUpsilon$. The success of this approach, is however, limited by the scalarization technique as well as $\bPi^\star$ may not be convex.
\begin{remark}[Limits of Scalarization]\label{rem:convexity_pistar}
\normalfont
Using scalarizations to recover $\bPi^\star$ suffers from the drawback that convexity assumptions on the scalarization function limit algorithms to only recover policies within the convex regions of $\bPi^\star$~\citep{vamplew2008limitations}. Subsequently, our algorithm is limited in this sense as it relies on convex scalarizations, however, we leave the extension to non-convex regions as future work, and assume $\bPi^\star$ to be convex for simplicity.
\end{remark}
\subsection{Bayes Regret}

Generally, algorithms for cooperative multi-agent RL consider maximizing the cumulative reward of all agents~\citep{littman1994markov}. In the fully-observable scenario (i.e., all agent observe the complete $(\x, \ba)$), the problem reduces to that of an MDP with fixed value and reward functions (given as the sum of individual value and rewards). Indeed by considering the scalarization $\bupsilon' = \frac{1}{M}\cdot{\bm 1}_M$ we can observe that by Proposition~\ref{prop:mg_bellman_optimal}, the optimal policy $\bpi^\star_{\bupsilon'}$ corresponds to the optimal policy for the MPD defined over $\cS \times \cA$ with rewards given by the average rewards obtained by each agent. It is therefore straightforward to recover a no-regret policy using a single-agent algorithm (by making a linear MDP assumption over the joint state and action space $\cS \times \cA$, as in~\cite{jin2020provably}), as it is similar to the parallel setting. Moreover, in distributed applications, additional constraints in the environment (e.g., load balancing,~\citet{schaerf1994adaptive}), may require learning policies that prioritize an agent over others, and hence we consider a general notion of \textit{Bayes regret}. Our objective is to approximate $\bPi^\star$ by learning a set of $T$ policies $\widehat \bPi_T$ that minimize the Bayes regret, given by,
\begin{equation}
\label{def:bayes_regret}
\fR_B(T) \triangleq \underset{\bupsilon \sim p_\bUpsilon}{\bbE}\left[\underset{\x \in \cS}{\max}\left[V_{\bupsilon, 1}^\star(\x) - \underset{\bpi \in \widehat\bPi_T}{\max}V^{\bpi}_{\bupsilon,1}(\x)\right]\right].
\end{equation}
Here $p_\bUpsilon$ is a distribution over $\bUpsilon$ that characterizes the nature of policies we wish to recover. For example, if we set $p_\bUpsilon$ as the uniform distribution over $\Delta^M$ then we can expect the policies recovered to prioritize all agents equally\footnote{One may consider minimizing the regret for a fixed scalarization $\bupsilon' = \bbE_{p_\bUpsilon}\bupsilon$, however, that will also recover only one policy in $\bPi^\star$, whereas we desire to capture regions of $\bPi^\star$.}. The advantage of minimizing Bayes regret can be understood as follows. For any $\bupsilon \in \bUpsilon$, if $\bpi^\star_{\bupsilon} \in \widehat\bPi_T$, then the regret incurred is 0. Hence, by collecting policies that minimize Bayes regret, we are effectively searching for policies that span dense regions of $\bPi^\star$ (assuming convexity, see Remark~\ref{rem:convexity_pistar}). Consider now the \textit{cumulative regret}:
\begin{equation}
\label{def:cumulative_regret}\fR_C(T) \triangleq \underset{t \in [T]}{\sum} \underset{\bupsilon_t \sim p_\bUpsilon}{\bbE}\left[\underset{\x \in \cS}{\max}\left[V_{\bupsilon_t, 1}^\star(\x) - V^{\bpi_{t}}_{\bupsilon_t,1}(\x)\right]\right].
\end{equation}
Where $\bupsilon_1, ..., \bupsilon_T \sim p_\bUpsilon$ are sampled i.i.d. from $p_\bUpsilon$, and $\bpi_{t}$ refers to the joint policy at episode $t$. Under suitable conditions on $\fs$ and $\bUpsilon$, we can bound the two quantities.
\begin{proposition}\label{prop:bayes_regret_bound}
For $\fs$ that is Lipschitz and bounded $\bUpsilon$, we have that $\fR_B(T) \leq \frac{1}{T}\fR_C(T) + o(1).$
\end{proposition}

We focus on minimizing the regret $\fR_C(T)$, as any no-regret algorithm for $\fR_C$ bounds $\fR_B$. 

\subsection{\texttt{Coop-LSVI} for Linear Multiagent MDPs}

The key observation for algorithm design in this setting is that for any MMDP$(\cS, \cM, \cA, H, \bbP, \bR)$, the optimal policy with respect to a fixed scalarization parameter $\bupsilon$ coincides with the optimal policy for the MDP$(\cS, \cA, H, \bbP, R')$ where $R'(\x, \ba) = \bupsilon^\top\br(\x, \ba) \ \forall\ \cS \times \cA$. Based on this observation, we extend the notion of a linear MDP to the multi-agent setting under random scalarizations. We first provide a natural regularity assumption of linearity.

\begin{definition}[Linear Multiagent MDP]
\label{def:linear_mg}
A Multiagent MDP $\text{MMDP}(\cS, \cM, \cA, H, \bbP, \bR)$ is a linear multi-agent MDP with $M+1$ features $\left\{\bphi_i\right\}_{i=1}^M, \bphi_i : \cS \times \cA \rightarrow \bbR^{d_1}$ and $\bphi_c: \cS \times \cA \rightarrow \bbR^{d_2}$ where $d_1 + d_2 = d$, if for any $h \in [H]$ there exist an unknown vector $\btheta_h \in \bbR^{d_1}$ and there exist $d_2$ unknown (signed) measures $\bmu_h = (\mu_h^1, ..., \mu_h^{d_2})$ over $\cS$ such that for any $m \in \cM, \x \in \cS, \ba \in \cA, h \in [H]$,
\begin{align*}
    r_{m, h} = \left\langle \bphi_m(\x, \ba), \btheta_h \right\rangle\text{, and } \bbP_h(\cdot | \x, \ba) = \langle \bphi_c(\x, \ba), \bmu_h(\cdot)\rangle.
\end{align*}
We denote the overall feature vector as $\bPhi(\cdot) \in \bbR^{d \times M}$, where, for any $\x \in \cS, \ba \in \cA$,
\begin{align*}
    \bPhi(\x, \ba) = \begin{bmatrix}
    \bphi_1(\x, \ba)^\top,& \bphi_c(\x, \ba)^\top \\
    \vdots & \vdots \\
    \bphi_M(\x, \ba)^\top,& \bphi_c(\x, \ba)^\top \\
    \end{bmatrix}^\top.
\end{align*}
Under this representation, we have that for any $\x \in \cS, \ba \in \cA, h \in [H]$,
\begin{align*}
    \br_h(\x, \ba) = \bPhi(\x, \ba)^\top \begin{bmatrix}
    \btheta_h \\
    {\bm 0}_{d_2}
    \end{bmatrix},\text{ and, }
    {\bm 1}_M\cdot\bbP_h(\cdot | \x, \ba) = \bPhi(\x, \ba)^\top \begin{bmatrix}
    {\bm 0}_{d_1} \\
    \bmu_h(\cdot)
    \end{bmatrix}.
\end{align*}
We assume without loss of generality that $\lVert \bPhi(\x, \ba) \rVert \leq 1 \ \forall\ (\x, \ba) \in \cS \times \cA$, $\lVert \btheta_h \rVert \leq \sqrt{d_1} \leq \sqrt{d}$ and $\lVert \bmu_h(\cS) \rVert \leq \sqrt{d_2} \leq \sqrt{d}$.
\end{definition}
Observe that this assumption implies that there exist a set of weights such that the scalarized $Q-$values for any sclarization parameter $\bupsilon$ are linear projections of the combined features $\bPhi(\cdot)$. 
\begin{lemma}[Linearity of weights in MMDP]
\label{lem:bound_mg_policy_weight_main_paper}
Under the linear MMDP Assumption (Definition~\ref{def:linear_mg}), for any fixed policy $\pi$ and $\bupsilon \in \bUpsilon$, there exist weights $\{\w^\pi_{\bupsilon, h}\}_{h \in [H]}$ such that $Q^\pi_{\bupsilon, h}(\x, \ba) = \bupsilon^\top \bPhi(\x, \ba)^\top \w^\pi_{\bupsilon, h}$ for all $(\x, \ba, h) \in \cS \times \cA \times [H]$, where $\lVert \w^\pi_{\bupsilon, h} \rVert_2 \leq 2H\sqrt{d}$.
\end{lemma}
\begin{remark}[Multi-agent modeling assumptions]
\normalfont
We now discuss how this assumption differs from the typical modeling assumption in the single-agent function approximation setting~\citep{jin2020provably, wang2020provably}. In contrast to the typical assumption (see Assumption A in \citet{jin2020provably}, also in~\citet{bradtke1996linear, melo2007q}), here we model, for any agent, the reward $r_{m, h}$ and dynamics $\bbP_h$ for all agents separately, each with $d_1$ and $d_2$ dimensions respectively such that $d_1 + d_2 = d$. In the single-agent setting, identical assumptions on the reward and transition kernels will lead to a model with complexity $\max\{d_1, d_2\}$, whereas in our formulation we have a complexity of $d_1 + d_2$. Given that both $d_1, d_2 \leq d$ this implies that our fomulation incurs a maximum overhead of $2\sqrt{2}$ in the regret if applied to the LSVI algorithm presented in~\citet{jin2020provably}. Furthermore,  observe that in the \textit{fully-cooperative} setting, where $r_{1, h} = ... = r_{M, h} \forall h \in [H]$, we have that $\bphi_1 = \bphi_2 = ... = \bphi_M$ satisfies the modeling requirement. We also have that the features $\bPhi$ are functions of the~\textit{joint} action, which differentiates the MMDP from a parallel MDP setting.
\end{remark}

\begin{algorithm}[t]
\caption{\texttt{\textbf{Coop-LSVI}} for Multiagent MDPs}
\label{alg:mg}
\begin{algorithmic}[1] %[1] enables line numbers
\STATE \textbf{Input}: $T, \bPhi, H, S$, sequence $\beta_h = \{(\beta^t_{h})_{t}\}$.
\STATE \textbf{Initialize}: $\bLambda^t_{h} = \lambda\bI_d, \delta\bLambda^t_{h} = {\bf 0}, \cU^m_h, \cW^m_h = \emptyset$.
\FOR{episode $t=1, 2, ..., T$}
\FOR{agent $m \in \cM$}
\STATE Receive initial state $\x^t_{1}, \bupsilon_t \sim p_\bUpsilon$.
\STATE Set $V^t_{\bupsilon_t, H+1}(\cdot) \leftarrow 0$.
\FOR{step $h = H, ..., 1$}
% \STATE Compute $\bLambda^t_{h} \leftarrow \bS^t_{m, h} + \delta\bS^t_{m, h}$.
% \STATE Compute $\widehat Q^t_{m, h}$ and $\sigma^t_{m, h}$ (Eqn.~\ref{eqn:q_mg}).
\STATE Compute $Q^t_{\bupsilon_t, h}(\cdot, \cdot)$ (Eqn.~\ref{eqn:q_mg})
\STATE Set $V^t_{\bupsilon_t, h}(\cdot) \leftarrow \max_{\ba \in \cA}Q^t_{\bupsilon_t, h}(\cdot, \ba)$.
\ENDFOR
\FOR{step $h=1, ..., H$}
\STATE Take action $a^t_{m, h} \leftarrow [\argmax_{a \in \cA} Q^t_{\bupsilon_t, h}(\x^t_{h}, \ba)]_m$.
\STATE Observe $r^t_{m, h}, \x^t_{h+1}$.
\STATE Update $\delta\bLambda^t_{h} \leftarrow \delta\bLambda^t_{h} + \bPhi(\bz^t_{h})\bPhi(\bz^t_{h})^\top$.
\STATE Update $\cW^m_h \leftarrow \cW^m_h \cup (m, r^t_{m, h})$.
\IF{$\log\frac{\det\left(\bLambda^t_{h} + \delta\bLambda^t_{h}+ \lambda\bI\right)}{ \det\left(\bLambda^t_{h}+ \lambda\bI\right)} > S$}
\STATE \textsc{Synchronize}$ \leftarrow $ \textsc{True}. \label{step:alg_ind_homo}
\ENDIF
\ENDFOR
\ENDFOR
\IF{\textsc{Synchronize}}
\FOR{step $h = H, ..., 1$}
\STATE [$\forall$ \textsc{Agents}] Send $\cW^h_m \rightarrow$\textsc{Server}.
\STATE [\textsc{Server}] Aggregate $\cW^h \rightarrow \cup_{m \in \cM} \cW^m_h$.
\STATE [\textsc{Server}] Communicate $\cW^h$ to each agent.
\STATE [$\forall$ \textsc{Agents}] Set $\delta\bLambda^t_{h} \leftarrow 0, \cW^m_h \leftarrow \emptyset$.
\STATE [$\forall$ \textsc{Agents}] Set $\bLambda^t_{h} \leftarrow \bLambda^t_{h} + \sum_{(\x, \ba) \in \cW^h} \bPhi(\x, \ba)\bPhi(\x, \ba)^\top$.
\STATE [$\forall$ \textsc{Agents}] Set $\cU^m_h \leftarrow \cU^m_h \cup \cW^m_h$
\ENDFOR
\ENDIF
\ENDFOR
\end{algorithmic}
\end{algorithm}
\subsection{Algorithm Design} 

The motivation behind our design is to learn a function that simultaneously can recover a policy close to $\bpi^\star_\bupsilon$ for each MDP corresponding to $\bupsilon \in \bUpsilon$, recovering $\bPi^\star$. The algorithm is once again a distributed variant of least-squares value iteration with UCB exploration. Following Proposition~\ref{prop:mg_bellman_optimal}, the central idea is to scalarize the MMDP with a randomly sampled parameter $\bupsilon$, and each of the $M$ agents will execute the \textit{joint} policy that coincides with (an approximation of) $\bpi^\star_\bupsilon$. Now, the key idea is to make sure that each agent acts according to the joint policy that is aiming to mimic $\bpi^\star_\bupsilon$. Therefore, we must ensure that the local estimate for the \textit{joint} policy obtained by any agent must be identical, such that the \textit{joint} action is in accordance with $\bpi^\star_\bupsilon$. To achieve this we will utilize a rare-switching technique similar to Algorithm~\ref{alg:ind_homo}. In any episode $t \in [T]$, the objective would be to obtain the optimal (scalar) $Q-$values $Q^\star_{\bupsilon_t, h}$ by recursively applying the Bellman equation and solving the resulting equations via a \textit{vector-valued} regression. Since the policy variables are designed to be identical across agents at all times, we drop the $m$ subscript.

Specifically, let us assume the last synchronization between agents occured in episode $k_t$. Each agent obtains an \textit{identical} sequence of value functions $\{Q^t_{\bupsilon, h}\}_{h \in [H]}$ by iteratively performing linear least-squares ridge regression from the history available from the previous $k_t$ episodes, but in contrast to the parallel setting, it now first learns a vector $Q-$function $\bQ_{t, h}$ over $\bbR^M$, which is scalarized to obtain the $Q-$value as $ Q_{\bupsilon_t, h} = \bupsilon_t^\top\bQ_{t, h}$. Each agent $m$ first sets $\bQ_{t, H+1}$ to be a zero vector, and for any $h \in [H]$, solves the following sequence of regressions to obtain $Q-$values. For each $h  = H, ..., 1$, for each agent computes,
\begin{align*}\label{eqn:v_mg}
V^t_{\bupsilon_t, h+1}(\x) \leftarrow \argmax_{\ba \in \cA} \left\langle \bupsilon_t, \bQ_{t, h+1}(\x, \ba) \right\rangle, \ \widehat\bQ_{t, h} \leftarrow \underset{\w \in \bbR^d}{\argmin}\left[\underset{\tau \in [k_t]}{\sum} \left\lVert \y_{\tau, h} - \bPhi(\x_\tau, \ba_\tau)^\top\w\right\rVert_2^2 + \lambda\lVert \w \rVert_2^2\right].
\end{align*}
\begin{equation}\label{eqn:q_mg}
Q_{t, h}(\x, \ba) \leftarrow \bupsilon_t^\top\widehat\bQ_{t, h} + \beta_t\cdot\lVert \bPhi(\x, \ba)^\top(\bLambda^h_{t} )^{-1}\bPhi(\x, \ba)\rVert_{2}.
\end{equation}
Here the targets $\y_{\tau, h} = \br_h(\x, \ba) + {\bm 1}_M\cdot V^t_{\bupsilon_t, h+1}$, ${\bm 1}_M$ denotes the all-ones vector in $\bbR^M$, $\beta_t$ is an appropriately chosen sequence and $\bLambda^h_t$ is described subsequently. Once all of these quantities are computed, each agent $m=1, .., M$ selects the action $a^t_{m, h} = \left[\argmax_{\ba \in \cA} Q_{t, h}(\x^t_h, \ba)\right]_m$ for each $h \in [H]$.
Hence, the joint action $\ba^t_h = \{a^t_{m, h}\}_{m=1}^M = \argmax_{\ba \in \cA}Q_{t, h}(\x^t_h, \ba)$. Observe that while the computations of the policy is decentralized, the policies coincide at all instances by the modeling assumption and the periodic synchronizations between the agents. We now present the closed form of $\widehat \bQ_{t, h}$. Consider the contraction $\bz^h_\tau = (\x^h_\tau, \ba^h_\tau)$ and the map $\bPhi^h_t: \bbR^d \rightarrow \bbR^{Mt}$ given by,
\begin{align*}
    \bPhi^h_t\btheta \triangleq \left[(\bPhi(\bz^h_1)^\top\btheta)^\top, ..., (\bPhi(\bz_{t}^\top\btheta)^\top \right]^\top\ \forall\ \btheta \in \bbR^d.
\end{align*}
Now, consider $\bLambda^h_t = (\bPhi^h_{k_t})^\top\bPhi^h_{k_t} + \lambda\bI_d \in \bbR^{d\times d}$, and the matrix $\bU^h_t = \sum_{\tau=1}^{k_t} \bPhi(\bz^h_\tau)\y_{\tau, h}$. Then, we have by a multi-task concentration (see Appendix B of~\citet{chowdhury2020no}),
\begin{align}
    \widehat \bQ_{t, h}(\x, \ba) = \bPhi(\x, \ba)^\top\left(\bLambda^h_{t}\right)^{-1}\bU^h_{t}.
\end{align}
Now, to limit communication, the synchronization protocol is similar to that in Section 5.1 of~\citet{abbasi2011improved}. Whenever $\det(\bLambda^h_t) \geq S\cdot\det( \bLambda^h_{k_t})$, for some constant $S$, the agents synchronize their rewards. The pseudocode for the algorithm is presented in Algorithm~\ref{alg:mg}.

\subsection{Regret Analysis}

The primary challenges in bounding the regret arise from the fact that the agents are simultaneously trying to recover a class of policies (the Pareto frontier, $\bPi^\star$) instead of any single policy within $\bPi^\star$. This modeling choice requires us to estimate the individual rewards for each of the $M$ agents simultaneously, and leverage vector-valued concentration such that we can bound the estimation error from the optimal policy for any $\bupsilon \in \bUpsilon$. For this, we first derive a concentration result on the $\ell_2$-error on estimating vector-valued value functions $(\bV)$ which is independent of the sampled $\bupsilon_t$, by leveraging the (vector-valued) sef-normalized martingale analysis from~\citet{chowdhury2020no}. Next, we note that the scalarization function $\fs$ is Lipschitz with constant $1$, and therefore, for any $\bupsilon \in \bUpsilon$ and vectors $\bv_1 , \bv_2 \in \bbR^M$, $|\fs_\bupsilon(\bv_1) - \fs_\bupsilon(\bv_2)| \leq \lVert \bv_1 - \bv_2 \rVert_2$, and hence our previous concentration result allows us to bound the estimation error \textit{independently} of the scalarization $\bupsilon$. The structure of the rest of the bound is similar to that of the parallel MDP setting. Here we provide the primary regret bound, and defer the proof to the Appendix.

\begin{theorem}
\label{thm:ind_mg}
\textbf{\texttt{CoopLSVI}} when run on an MMDP with $M$ agents and communication threshold $S$, $\beta_t=\cO(H\sqrt{d\log (tMH)})$ and $\lambda=1$ obtains, with probability at least $1-\alpha$, cumulative regret:
\begin{equation*}
\fR_C(T)=\widetilde\cO\left(d^{\frac{3}{2}}H^2\sqrt{ST\log\left(\frac{1}{\alpha}\right)}\right).
\end{equation*}
\end{theorem}
\begin{remark}[Multiagent Regret Bound]
\normalfont
Theorem~\ref{thm:ind_mg} claims in conjunction with Proposition~\ref{prop:bayes_regret_bound} that \textbf{\texttt{CoopLSVI}} obtains Bayes regret of $\ctO(\sqrt{T})$ even when communication is limited. Note that the dependence on $T$ matches that of MDP algorithms (e.g.,~\citet{jin2020provably}), and moreover, we recover the same rate (up to logarithmic factors) when $M=1$, ensuring that the analysis is tight. Additionally, we see that this algorithm can easily be applied to an MDP by simply selecting $p_\bUpsilon$ to be a point mass at the appropriate $\bupsilon$, with no increase in regret. Finally, we see that \texttt{\textbf{CoopLSVI}} can also be emulated on a single agent with $M$ objectives, where $S = 1$, which provides, to the best of our knowledge, the first no-regret algorithm for multi-objective reinforcement learning.
\end{remark}
Note that as the common state $\x$ is visible to all agents, the agents only require communication of rewards. While the protocol in this setting essentially operates on a similar idea (communicating only when the variance of the history for any step crosses a threshold), the exact form of the threshold differs slightly in this case: instead of computing individual policy parameters using local observations, here the agents employ a \textit{rarely-switching} strategy that delays updating global parameters until the threshold condition is met. Despite the slight difference in the communication strategy, the overall communication complexity is similar to the parallel MDP variant.
\begin{lemma}[Communication Complexity]
\label{lem:communication_complexity_mg}
If Algorithm~\ref{alg:mg} is run with threshold $S > 1$, then the total number of episodes with communication $n \leq dH\log_S\left(1+MT/d)\right) + H$. When $S\leq 1$, $n = T$.
\end{lemma}
\begin{remark}[Communication complexity]
\normalfont
As is with the parallel MDP setting, we can control the communication budget by adjusting the threshold parameter $S$. Note that when $S = 1$, we have that communication will occur each round round, as the threshold will be satisfied trivially by the rank-1 update to the covariance matrix. If the horizon $T$ is known in advance, one can set $S=(1+MT/d)^{1/C}$ for some independent constant $C > 1$, to ensure that the total rounds of communication is a fixed constant $(dC+1)H$, which provides us a group regret of $\ctO(M^{\frac{1}{2C}}\cdot T^{\frac{1}{2}+\frac{1}{2C}})$. This dependence of the regret on the threshold $S$ is indeed worse than the protocol for the parallel MDP, as in this case, the agents \textit{do not} utilize local observations until they are synchronized, and merely readjust policy parameters for different scalarizations $\bupsilon$, necessitating frequent communication. A balance between communication and regret can be obtained by setting $S=C'$ for some absolute constant $C'$, leading to a total $\cO(\log MT)$ rounds of communication with $\ctO(\sqrt{T})$ regret.
\end{remark}
\section{Conclusion}

We presented \textbf{\texttt{CoopLSVI}}, a cooperative multi-agent reinforcement learning algorithm that attains sublinear regret while maintaining sublinear communication under linear function approximation. While most research in multi-agent RL focuses either on \textit{fully-cooperative} settings (i.e., a multi-agent MDP with identical reward functions), or stochastic games~\citep{shapley1953stochastic}, the heterogeneous setting considered here allows for the agents to both observe their rewards privately, while generalizing to Bayes regret guarantees, extending several lines of prior work~\citep{kar2013cal, zhang2018fully}. Similarly, \textbf{\texttt{CoopLSVI}} is the first algorithm to provide provably sublinear regret in heterogeneous parallel MDPs. Given the rapid advancements in federated learning, we believe this to be a valuable line of inquiry, extending beyond the work that has been done in the related problem of multi-armed bandits. 

There are several open questions that our work posits. A few areas include extending \textbf{\texttt{CoopLSVI}} to a fully-decentralized network topology; tight lower bounds on communication-regret tradeoffs; and robust estimation to avoid side information in heterogeneous settings. We believe our work will serve as a valuable stepping stone to further developments in this area.

\appendix
\onecolumn

\section{Omitted Algorithms}

\begin{algorithm}[H]
\caption{\texttt{\textbf{Coop-LSVI}} for Heterogeneous Rewards}
\label{alg:ind_hetero}
\begin{algorithmic} %[1] enables line numbers
\STATE \textbf{Input}: $T, \btphi, H, S$, sequence $\beta_h = \{(\beta^t_{m, h})_{m, t}\}$.
\STATE \textbf{Initialize}: $\bS^t_{m, h}, \delta\bS^t_{m, h} = {\bf 0}, \cU^m_h, \cW^m_h = \emptyset$.
\FOR{episode $t=1, 2, ..., T$}
\FOR{agent $m \in \cM$}
\STATE Receive initial state $x^t_{m, 1}$.
\STATE Set $V^t_{m, H+1}(\cdot) \leftarrow 0$.
\FOR{step $h = H, ..., 1$}
\STATE Compute $\btLambda^t_{m, h} \leftarrow \bS^t_{m, h} + \delta\bS^t_{m, h}$.
\STATE Compute $\widehat Q^t_{m, h}$ and $\sigma^t_{m, h}$ (Eqn.~\ref{eqn:qhat_ind_hetero}).
\STATE Compute $Q^t_{m, h}(\cdot, \cdot, \cdot)$ (Eqn.~\ref{eqn:q_ind_hetero})
\STATE Set $V^t_{m, h}(\cdot) \leftarrow \max_{a \in \cA}Q^t_{m, h}(\cdot, a)$.
\ENDFOR
\FOR{step $h=1, ..., H$}
\STATE Take action $a^t_{m, h} \leftarrow \argmax_{a \in \cA} Q^t_{m, h}(m, x^t_{m, h}, a)$.
\STATE Observe $r^t_{m, h}, x^t_{m, h+1}$.
\STATE Update $\delta\bS^t_{m, h} \leftarrow \delta\bS^t_{m, h} + \btphi(m, z^t_{m, h})\btphi(m, z^t_{m, h})^\top$.
\STATE Update $\cW^m_h \leftarrow \cW^m_h \cup (m, x, a, x')$.
\IF{$\log\frac{\det\left(\bS^t_{m, h} + \delta\bS^t_{m, h}+ \lambda\bI\right)}{ \det\left(\bS^t_{m, h}+ \lambda\bI\right)} > \frac{S}{\Delta t_{m, h}}$}
\STATE \textsc{Synchronize}$ \leftarrow $ \textsc{True}. \label{step:alg_ind_homo}
\ENDIF
\ENDFOR
\ENDFOR
\IF{\textsc{Synchronize}}
\FOR{step $h = H, ..., 1$}
\STATE [$\forall$ \textsc{Agents}] Send $\cW^h_m \rightarrow$\textsc{Server}.
\STATE [\textsc{Server}] Aggregate $\cW^h \rightarrow \cup_{m \in \cM} \cW^m_h$.
\STATE [\textsc{Server}] Communicate $\cW^h$ to each agent.
\STATE [$\forall$ \textsc{Agents}] Set $\delta\bLambda^t_{h} \leftarrow 0, \cW^m_h \leftarrow \emptyset$.
\STATE [$\forall$ \textsc{Agents}] Set $\bLambda^t_{h} \leftarrow \bLambda^t_{h} + \sum_{(n, x, a) \in \cW^h} \btphi(n, x, a)\bphi(n, x, a)^\top$.
\STATE [$\forall$ \textsc{Agents}] Set $\cU^m_h \leftarrow \cU^m_h \cup \cW^m_h$
\ENDFOR
\ENDIF
\ENDFOR
\end{algorithmic}
\end{algorithm}

\newpage
\section{Parallel MDP Proofs}
\subsection{Proof of Lemma~\ref{prop:communication_complexity}}
\begin{proof}
Denote an epoch as the number of episodes between two rounds of communication. Let $q = \sqrt{\frac{ST}{d\log(1 + T/d)}+1}$. There can be at most $\lceil T/q \rceil$ rounds of communication such that they occur after an epoch of length $q$. On the other hand, if there is any round of communication succeeding an epoch (that begins, say at time $t$) of length $< n'$, then for that epoch, $\log\frac{\det\left(\bS^t_{m, h} + \delta\bS^t_{m, h} + \lambda\bI_d\right)}{\det\left(\bS^{t}_{m, h} + \lambda\bI_d\right)} \geq \frac{S}{q}$. Let the communication occur at a set of episodes $t'_1, ..., t'_n$. Now, since:
\begin{align}
    \sum_{i=1}^{n-1}\log\frac{\det\left(\bS^{t_{i+1}}_{m, h}\right)}{\det\left(\bS^{t_{i}}_{m, h}\right)} = \log\frac{\det\left(\bLambda^T_h\right)}{\det\left(\bLambda^0_h\right)} \leq d\log(1 + T/(d)),
\end{align}
We have that the total number of communication rounds succeeding epochs of length less than $n'$ is upper bounded by $ \log\frac{\det\left(\bLambda^T_h\right)}{\det\left(\bLambda^0_h\right)} \leq d\log(1 + T/(d))\cdot(q/S)$. Combining both the results together, we have the total rounds of communication as:
\begin{align}
    n &\leq \lceil T/q \rceil + \lceil d\log(1 + T/(d))\cdot(q/S) \rceil \\
    &\leq T/q + d\log(1 + T/(d))\cdot(q/S) + 2
\end{align}
Replacing $q$ from earlier and summing over $h \in [H]$ (as communication may be triggered by any of the steps satisfying the condition) gives us the final result.
\end{proof}
\subsection{Proof of Theorem~\ref{thm:ind_homo} (Homogenous Setting)}
We first present our primary concentration result to bound the error in the least-squares value iteration.
\begin{lemma}
Under the setting of Therorem~\ref{thm:ind_homo}, let $c_\beta$ be the constant defining $\beta$, and $\bS^t_{m, h}$ and $\bLambda^k_t$ be defined as follows.
\begin{multline*}
    \bS^t_{m, h} = \sum_{n=1}^M \sum_{\tau=1}^{k_t}\bphi(x^\tau_{n, h}, a^\tau_{n, h}) \left[V^t_{m, h+1}(x^\tau_{n, h+1})- (\bbP_hV^t_{m, h+1})(x^\tau_{n, h}, a^\tau_{n, h})\right] \\+ \sum_{\tau = k_t + 1}^{t-1} \bphi(x^\tau_{m, h}, a^\tau_{m, h})\left[V^t_{m, h+1}(x^\tau_{m, h+1})- (\bbP_hV^t_{m, h+1})(x^\tau_{m, h}, a^\tau_{m, h})\right],
\end{multline*} 
\begin{align*}
    \bLambda^t_{m, h} =  \sum_{n=1}^M \sum_{\tau=1}^{k_t}\bphi(x^\tau_{n, h}, a^\tau_{n, h})\bphi(x^\tau_{n, h}, a^\tau_{n, h})^\top + \sum_{\tau = k_t + 1}^{t-1} \bphi(x^\tau_{m, h}, a^\tau_{m, h})\bphi(x^\tau_{m, h}, a^\tau_{m, h})^\top + \lambda\bI_d.
\end{align*}
Where $V \in \cV$ and $\cN_\epsilon$ denotes the $\epsilon-$covering of the value function space $\cV$. Then, there exists an absolute constant $c_\beta$ independent of $M, T, H, d$, such that, with probability at least $1-\delta'/2$ for all $m \in \cM, t \in [T], h \in [H]$ simultaneously,
\begin{align*}
    \left\lVert \bS^t_{m, h} \right\rVert_{(\bLambda^t_{m, h})^{-1}} &\leq c_\beta\cdot dH\sqrt{2\log\left(\frac{dMTH}{\delta'}\right)}.
\end{align*}
\label{lem:parallel_homo_lsvi}
\end{lemma}
\begin{proof}
The proof is done in two steps. The first step is to bound the deviations in $\bS$ for any fixed function $V$ by a martingale concentration. The second step is to bound the resulting concentration over all functions $V$ by a covering argument. Finally, we select appropriate constants to provide the form of the result required.

\noindent\textbf{\underline{Step 1}}.
Note that for any agent $m$, the function $V^t_{m, h+1}$ depends on the historical data from all $M$ agents from the first $k_t$ episodes, and the personal historical data for the first $(t-1)$ episodes, and depends on
\begin{align}
    \cU^m_h(t) = \left(\cup_{n \in [M], \tau \in [k_t]}\{(x^\tau_{n, h}, a^\tau_{n, h}, x^\tau_{n, h+1})\}\right) \bigcup \left(\cup_{\tau \in [k_t+1, t-1} \{(x^\tau_{m, h}, a^\tau_{m, h}, x^\tau_{m, h+1})\}\right).
\end{align}

To bound the term we will construct an appropriate filtration to use a self-normalized concentration defined on elements of $\cU^m_h(t)$. We highlight that in the multi-agent case with stochastic communication, it is not straightforward to provide a uniform martingale concentration that holds for all $t \in [T]$ simultaneously (as is done in the single-agent case), as the stochasticity in the environment dictates when communication will take place, and subsequently the quantity considered within self-normalization will depend on this communication itself. To circumvent this issue, we will first fix $k_t \leq t$ and obtain a filtration for a fixed $k_t$. Then, we will take a union bound over all $k_t \in [t]$ to provide the final self-normalized bound. We first fix $k_t$ and define the following mappings where $i \in \left[M(t-1)\right], l \in [t-1], $ and $n \in [M]$.
\begin{align*}
    \mu(i) = \left\lceil \frac{i}{M} \right\rceil,
    \nu(i) = i (\text{mod } M), \text{ and, }
    \eta(l, n) = l\cdot(M+1) + n - 1.
\end{align*}  
Now, for a fixed $k_t$, consider the stochastic processes $\{\tilde x_\tau\}_{\tau = 1}^\infty$ and $\{\tilde\bphi_\tau\}_{\tau = 1}^\infty$, where,
\begin{align*}
   \tilde \bphi_i = \bphi(x^{\nu(i)}_{\mu(i), h+1}) \otimes \bone_d\left\{ \left(\mu(i) = m\right) \lor \left(\nu(i) \leq k_t\right)\right\}
\end{align*} 
Here $\otimes$ denotes the Hadamard product, and $\bone_{d}$ is the indicator function in $\bbR^d$. Consider now the filtration $\{\cF_\tau\}_{\tau=0}^\infty$, where $\cF_0$ is empty, and $\cF_\tau = \sigma\left(\left\{ \bigcup (\tilde x_i, \tilde \bphi_i)\right\}_{i \leq \tau}\right)$, where $\sigma(\cdot)$ denotes the corresponding $\sigma-$algebra formed by the set.

At any instant $t$ for any agent $m$, the function $V^t_{m, h+1}$ and features $\bphi(x^t_{m, h}, a^t_{m, h})$ depend only on historical data from all other agents $[M] \setminus \{m\}$ up to the last episode of synchronization $k_t \leq t - 1$ and depend on the personal data up to episode $t-1$. Hence, both are $V^t_{m, h+1}$ and $\bphi(x^t_{m, h}, a^t_{m, h})$ are measurable with respect to 
\begin{align*}
    \sigma\left(\left\{ \bigcup_{l=1}^{k_t}\bigcup_{n=1}^{M} (\tilde x_{\eta(l, n)}, \tilde \bphi_{\eta(l, n)})\right\} \bigcup \left\{\bigcup_{l=k_t+1}^{t-1} (\tilde x_{\eta(l, m)}, \tilde\bphi_{\eta(l, m)})\right\} \right).
\end{align*}
This is a subset of $\cF_{\eta(t, m)}$. Therefore $V^t_{m, h+1}$ is $\cF_{\eta(t, m)}-$measurable for fixed $k_t$. Now, consider $\cU^m_h(\tau)$, the set of features available to agent $m$ at episode $\tau \leq t$. We therefore have that, for any value function $V$,
\begin{align*}
    &\sum_{\tau=1}^{M(t-1)} \btphi_{m, h}(\tau)\left\{ V(\tilde x_\tau) - \bbE[V(\tilde x_\tau) | \cF_{\tau-1}] \right\} \\
    &= \sum_{\tau=1}^{M(t-1)} \left[ \bphi(x^{\nu(i)}_{\mu(i), h+1}) \otimes \bone_d\left\{ \left(\mu(i) = m\right) \lor \left(\nu(i) \leq k_t\right)\right\}\right]\left\{ V(\tilde x_\tau) - \bbE\left[V(\tilde x_\tau) | \cF_{\tau-1}\right] \right\} \\
    &= \sum_{(x_\tau, a_\tau, x'_\tau) \in \cU^m_h(t)} \bphi(x_\tau, a_\tau)\left\{ V(x'_\tau) - \bbE[V(x'_\tau) | \cF_{\tau-1}] \right\}. 
\end{align*}
Now, when $V = V^t_{m, h+1}$, we have from the above,
\begin{align*}
    &\sum_{\tau=1}^{M(t-1)} \btphi_\tau\left\{ V^t_{h, m+1}(\tilde x_\tau) - \bbE[V^t_{h, m+1}(\tilde x_\tau) | \cF_{\tau-1}]\right\} \\&= \sum_{(n_\tau, x_\tau, a_\tau, x'_\tau) \in \cU^m_h(t)} \bphi(x_\tau, a_\tau)\left\{ V^t_{m, h+1}(x'_\tau) - \bbE[V^t_{m, h+1}(x'_\tau) | \cF_{\tau-1}] \right\} = \bS^t_{m, h}. 
\end{align*}
Furthermore, consider $\widetilde\bLambda^t_{m, h} = \lambda \bI_d + \sum_{\tau=1}^{M(t-1)} \btphi_\tau\btphi_\tau^\top$. For the second term, we have,
\begin{align*}
     &\widetilde\bLambda^t_{m, h}=\lambda \bI_d +\sum_{\tau=1}^{M(t-1)} \btphi_\tau\btphi_\tau^\top \\
     =&\lambda \bI_d\\& + \sum_{\tau=1}^{M(t-1)} \left[\bphi(x^{\nu(i)}_{\mu(i), h+1}) \otimes \bone_d\left\{ \mu(i) = m \lor \nu(i) \leq k_t\right\}\right]\left[\bphi(x^{\nu(i)}_{\mu(i), h+1}) \otimes \bone_d\left\{ \mu(i) = m \lor \nu(i) \leq k_t\right\}\right]^\top \\
     =&\lambda \bI_d +\sum_{(n_\tau, x_\tau, a_\tau, x'_\tau) \in\cU^m_h(t)} \bphi(x_\tau, a_\tau)\bphi(x_\tau, a_\tau)^\top = \bLambda^t_{m, h}.
\end{align*}
To complete the proof, we bound $\left\lVert \sum_{\tau=1}^{M(t-1)} \widetilde\bphi_{m, h}(\tau)\left\{V^t_{h, m+1}(\tilde x_\tau) - \bbE[V^t_{h, m+1}(\tilde x_\tau) | \cF_{\tau-1}]\right\} \right\rVert_{(\widetilde\bLambda^t_{m, h})^{-1}}$ over all $k_t \in [t]$. We proceed following a self-normalized martingale bound and a covering argument, as done in~\cite{yang2020provably}.

Applying Lemma~\ref{lem:self_normalized_single_task} to  $\left\lVert \sum_{\tau=1}^{M(t-1)} \widetilde\bphi_{m, h}(\tau)\left\{V^t_{h, m+1}(\tilde x_\tau) - \bbE[V^t_{h, m+1}(\tilde x_\tau) | \cF_{\tau-1}]\right\} \right\rVert_{(\widetilde\bLambda^t_{m, h})^{-1}}$ under the filtration $\{\cF_\tau\}_{\tau=0}^\infty$ described earlier, we have that with probability at least $1-\delta'$,
\begin{align*}
    \left\lVert \bS^t_{m, h} \right\rVert^2_{(\bLambda^t_{m, h})^{-1}} &= \left\lVert \sum_{\tau=1}^{M(t-1)} \btphi_\tau\left\{V^t_{h, m+1}(\tilde x_\tau) - \bbE[V^t_{h, m+1}(\tilde x_\tau) | \cF_{\tau-1}]\right\} \right\rVert^2_{(\widetilde\bLambda^t_{m, h})^{-1}} \\
    &\leq \sup_{V \in \cV} \left\lVert \sum_{\tau=1}^{M(t-1)} \btphi_\tau\left\{V(\tilde x_\tau) - \bbE[V(\tilde x_\tau) | \cF_{\tau-1}]\right\} \right\rVert^2_{(\widetilde\bLambda^t_{m, h})^{-1}} \\
    &\leq 4H^2\cdot\log\frac{\det\left(\widetilde\bLambda^t_{m,h}\right)}{\det\left(\lambda\bI_d\right)} + 8H^2\log(|\cN_\epsilon|/\delta') + 8M^2t^2\epsilon^2/\lambda.
\end{align*}
Where $\cN_\epsilon$ is an $\epsilon-$covering of $\cV$. Therefore, we have that, with probability at least $1-\delta'$, for any fixed $k_t \leq t$,
\begin{align*}
    \left\lVert \bS^t_{m, h} \right\rVert_{(\bLambda^t_{m, h})^{-1}} &\leq 2H\sqrt{ \log\left(\frac{\det\left(\widetilde\bLambda^t_{m,h}\right)}{\det\left(\lambda\bI_d\right)}\right) + 2\log\left(\frac{|\cN_\epsilon|}{\delta}\right) + \frac{2M^2t^2\epsilon^2}{H^2\lambda}}.
\end{align*}
Taking a union bound over all $k_t \in [t]$, $m \in \cM, t \in [T], h \in [H]$ and replacing $\delta' = \delta/(MHT^2)$ gives us that with probability at least $1-\delta'/2$ for all $m \in \cM, t \in [T], h \in [H]$ simultaneously,
\begin{align}
    \left\lVert \bS^t_{m, h} \right\rVert_{(\bLambda^t_{m, h})^{-1}} &\leq 2H\sqrt{ \log\left(\frac{\det\left(\widetilde\bLambda^t_{m,h}\right)}{\det\left(\lambda\bI_d\right)}\right) + \log\left(MHT^2\cdot\frac{|\cN_\epsilon|}{\delta'}\right) + \frac{2M^2t^2\epsilon^2}{H^2\lambda}}.\\
    &\leq 2H\sqrt{\log\left(\frac{\det\left( \bLambda^t_h\right)}{\det\left(\lambda\bI_d\right)}\right) + 2\log\left(\frac{MHT^2|\cN_\epsilon|}{\delta'}\right) + \frac{2t^2\epsilon^2}{H^2\lambda}}\\
    &\leq2H\sqrt{d\log\frac{t+\lambda}{\lambda} + 4\log(MHT)+  2\log\left(\frac{|\cN_\epsilon|}{\delta'}\right) + \frac{2t^2\epsilon^2}{H^2\lambda}}\tag{AM $\geqslant$ GM; determinant-trace inequality}
\end{align}

\noindent\textbf{\underline{Step 2}}.
Here $\cN_\epsilon$ is an $\epsilon-$covering of the function class $\cV_{\text{UCB}}$ for any $h \in [H], m \in [M]$ or $t \in [T]$ under the distance function $\text{dist}(V, V') = \sup_{x \in \cS} |V(x) - V'(x)|$. To bound this quantity by the appropriate covering number, we first observe that for any $V \in \cV_{\text{UCB}}$, we have that the policy weights are bounded as $2H\sqrt{dMT/\lambda}$ (Lemma~\ref{lem:bound_homo_algo_weight}). Therefore, by Lemma~\ref{lem:covering_ind_homo} we have for any constant $B$ such that $\beta^t_{m, h} \leq B$,
\begin{align}
    \log |\cN_\varepsilon| \leq d \log\left(1+8H\sqrt{\frac{dMT}{\lambda\epsilon^2}}\right) + d^2 \log\left(1 + \frac{8d^{1/2}B^2}{\lambda\epsilon^2}\right).
\end{align}
Recall that we select the hyperparameters $\lambda = 1$ and $\beta = \cO(dH\sqrt{\log(TMH)})$, and to balance the terms in $\bar\beta^t_h$ we select $\epsilon = \epsilon^\star = dH/T$. Finally, we obtain that for some absolute constant $c_\beta$, by replacing the above values,
\begin{align}
    \log |\cN_\varepsilon| \leq d \log\left(1+8\sqrt{\frac{MT^3}{d}}\right) + d^2 \log\left(1 + 8c_\beta d^{1/2}T^2\log(TMH)\right).
\end{align}
Therefore, for some absolute constant $C'$ independent of $M, T, H, d$ and $c_\beta$, we have,
\begin{align}
    \log |\cN_\varepsilon| \leq C'd^2 \log\left(CdT\log(TMH)\right).
\end{align}
Replacing this result in the result from Step 1, we have that with probability at least $1-\delta'/2$ for all $m \in \cM, t \in [T], h \in [H]$ simultaneously,
\begin{align*}
    &\left\lVert \bS^t_{m, h} \right\rVert_{(\bLambda^t_{m, h})^{-1}} \\ &\leq 2H\sqrt{(d+2)\log\frac{t+\lambda}{\lambda} + 2\log\left(\frac{1}{\alpha}\right) +  C'd^2 \log\left(c_\beta dT\log(TMH)\right) + 2 + 4\log(TMH)}.
\end{align*}
This implies that there exists an absolute constant $C$ independent of $M, T, H, d$ and $c_\beta$, such that, with probability at least $1-\delta'/2$ for all $m \in \cM, t \in [T], h \in [H]$ simultaneously,
\begin{align}
    \left\lVert \bS^t_{m, h} \right\rVert_{(\bLambda^t_{m, h})^{-1}} &\leq C\cdot dH\sqrt{2\log\left(\frac{(c_\beta + 2)dMTH}{\delta'}\right)}.
\end{align}
Now, following the procedure in Lemma B.4 of~\citet{jin2020provably}, we can select $c_\beta$ such that we have,
\begin{align}
    \left\lVert \bS^t_{m, h} \right\rVert_{(\bLambda^t_{m, h})^{-1}} &\leq c_\beta\cdot dH\sqrt{2\log\left(\frac{dMTH}{\delta'}\right)}.
\end{align}
This finishes the proof.
\end{proof}
Next, we present the key result for cooperative value iteration, which demonstrates that for any agent the estimated $Q-$values have bounded error for any policy $\pi$. This result is an extension of Lemma B.4 of~\cite{jin2020provably} on to the homogenous setting. 
\begin{lemma}\label{lem:parallel_homo_weight_delta}
There exists an absolute constant $c_\beta$ such that for $\beta^t_{m, h} = c_\beta\cdot dH \sqrt{\log(2dMHT/\delta')}$ for any policy $\pi$, such that for each $x \in \cS, a \in \cA$ we have for all $ m \in \cM, t \in[T], h \in [H]$ simultaneously, with probability at least $1-\delta'/2$,
\begin{align*}
     \left| \langle \bphi(x, a), \w^t_{m, h} - \w^\pi_h \rangle \right| \leq  \bbP_h(V^t_{m, h+1} - V^\pi_{m, h+1})(x, a) + c_\beta\cdot dH\cdot\lVert \bphi(z) \rVert_{(\bLambda^t_{m, h})^{-1}}\cdot \sqrt{2\log\left(\frac{dMTH}{\delta'}\right)}.
\end{align*}
\end{lemma}
\begin{proof}
By the Bellman equation and the assumption of the linear MDP (Definition~\ref{def:linear_mdp}), we have that for any policy $\bpi$, there exist weights $\w^\pi_{h}$ such that, for all $z \in \cZ$,
\begin{align}
    \langle \bphi(z), \w^\pi_h \rangle = r_h(z) + \bbP_h V^\pi_{h+1}(z).
\end{align}
The set of all observations available to any agent at instant $t$ is given by $\cU^m_h(t)$ for step $h$, with the cardinality of this set being $U^h_m(t)$. For convenience, let us assume an ordering $\tau = 1, ..., U^h_m(t)$ over this set and use the shorthand $U_m = U^h_m(t)$. Therefore, we have, for any $m \in \cM$,
\begin{align}
    \w^t_{m, h} - \w^\pi_h &= (\bLambda_{m, h}^t)^{-1}\sum_{\tau=1}^{U_m}\left[\bphi_\tau[(r_h+V^t_{m, h+1})(x_\tau)]\right] - \w^\pi_h \\
    &= (\bLambda_{m, h}^t)^{-1}\left\{ -\lambda\w^\pi_h + \sum_{\tau=1}^{U_m}\left[\bphi_\tau[V^t_{m, h+1}(x'_\tau) -\bbP_h V^\pi_{m, h+1}(x_\tau, a_\tau)]\right]\right\}.
\end{align}
\begin{multline}
\implies \w^t_{m, h} - \w^\pi_h = \underbrace{-\lambda(\bLambda_{m, h}^t)^{-1}\w^\pi_h}_{\bv_1} + \underbrace{(\bLambda_{m, h}^t)^{-1}\left\{\sum_{\tau=1}^{U_m}\left[\bphi_\tau[V^t_{m, h+1}(x'_\tau) -\bbP_h V^t_{m, h+1}(z_\tau)]\right] \right\}}_{\bv_2} \\+ \underbrace{(\bLambda_{m, h}^t)^{-1}\left\{\sum_{\tau=1}^{U_m}\left[\bphi_\tau[\bbP_h V^t_{m, h+1} - \bbP_h V^\pi_{m, h+1})(z_\tau)]\right] \right\}}_{\bv_3}.
\end{multline}
Now, we know that for any $z \in \cZ$ for any policy $\pi$,
\begin{align}
    \left| \langle \bphi(z), \bv_1 \rangle \right| \leq \lambda\left| \langle \bphi(z), \bLambda_{m, h}^t)^{-1}\w^\pi_h \rangle \right| \leq \lambda\cdot \lVert \w^\pi_h \rVert \lVert \bphi(z) \rVert_{(\bLambda^t_{m, h})^{-1}} \leq 2H\lambda\sqrt{d}\lVert \bphi(z) \rVert_{(\bLambda^t_{m, h})^{-1}}
\end{align}
Here the last inequality follows from Lemma~\ref{lem:bound_homo_policy_weight}. For the second term, we have by Lemma~\ref{lem:parallel_homo_lsvi} that there exists an absolute constant $C$ independent of $M, T, H, d$ and $c_\beta$, such that, with probability at least $1-\delta'/2$ for all $m \in \cM, t \in [T], h \in [H]$ simultaneously,
\begin{align}
    \left| \langle \bphi(z), \bv_2 \rangle \right| \leq \left\lVert \bphi(z)\right\rVert_{(\bLambda^t_{m, h})^{-1}}\cdot c_\beta\cdot dH\cdot\sqrt{2\log\left(\frac{dMTH}{\delta'}\right)}.
\end{align}
For the last term, note that,
\begin{align}
    &\left| \langle \bphi(z), \bv_3 \rangle \right| \\
    &= \left\langle \bphi(z), (\bLambda_{m, h}^t)^{-1}\left\{\sum_{\tau=1}^{U_m}\left[\bphi_\tau[\bbP_h V^t_{m, h+1} - \bbP_h V^\pi_{m, h+1})(z_\tau)]\right] \right\} \right\rangle \\
    &= \left\langle \bphi(z), (\bLambda_{m, h}^t)^{-1}\sum_{\tau=1}^{U_m}\left[\bphi_\tau\bphi^\top_\tau\int(V^t_{m, h+1} -  V^\pi_{m, h+1})(x')d\bmu_h(x')\right]\right\rangle \\
    &= \left\langle \bphi(z), \int(V^t_{m, h+1} -  V^\pi_{m, h+1})(x')d\bmu_h(x')\right\rangle -\lambda\left\langle \bphi(z), (\bLambda_{m, h}^t)^{-1}\int(V^t_{m, h+1} -  V^\pi_{m, h+1})(x')d\bmu_h(x')\right\rangle\\
    &= \bbP_h(V^t_{m, h+1} - V^\pi_{m, h+1})(x, a) -\lambda\left\langle \bphi(z), (\bLambda_{m, h}^t)^{-1}\int(V^t_{m, h+1} -  V^\pi_{m, h+1})(x')d\bmu_h(x')\right\rangle\\
    &= \bbP_h(V^t_{m, h+1} - V^\pi_{m, h+1})(x, a) + 2H\sqrt{d\lambda}\lVert \bphi(z) \rVert_{(\bLambda^t_{m, h})^{-1}}.
\end{align}
Putting it all together, we have that since $\langle \bphi(z), \w^t_{m, h} - \w^\pi_h \rangle = \langle \bphi(z), \bv_1 + \bv_2 + \bv_3 \rangle$, there exists an absolute constant $C$ independent of $M, T, H, d$ and $c_\beta$, such that, with probability at least $1-\delta'/2$ for all $m \in \cM, t \in [T], h \in [H]$ simultaneously,
\begin{multline*}
    \left| \langle \bphi(x, a), \w^t_{m, h} - \w^\pi_h \rangle \right|\leq  \bbP_h(V^t_{m, h+1} - V^\pi_{m, h+1})(x, a) \\ + \lVert \bphi(z) \rVert_{(\bLambda^t_{m, h})^{-1}}\left( C\cdot dH\cdot\sqrt{2\log\left((c_\beta+2)\frac{dMTH}{\delta'}\right)} + 2H\sqrt{d\lambda} + 2H\lambda\sqrt{d}\right)
\end{multline*}
Since $\lambda \leq 1$ and since $C$ is independent of $c_\beta$, we can select $c_\beta$ such that we have the following for any $(x, a) \in \cS \times \cA$ with probability $1-\delta'/2$ simultaneously for all $h \in [H], m \in \cM, t \in [T]$,
\begin{align}
    \left| \langle \bphi(x, a), \w^t_{m, h} - \w^\pi_h \rangle \right| \leq  \bbP_h(V^t_{m, h+1} - V^\pi_{m, h+1})(x, a) + c_\beta\cdot dH\cdot\lVert \bphi(z) \rVert_{(\bLambda^t_{m, h})^{-1}}\cdot \sqrt{2\log\left(\frac{dMTH}{\delta'}\right)}.
\end{align}
\end{proof}

\begin{lemma}[UCB in the Homogenous Setting]
With probability at least $1-\delta'/2$, we have that for all $(x, a, h, t, m) \in \cS \times \cA \times [H] \times [T] \times \cM$, $Q^t_{m, h}(x, a) \geq Q^\star_{m, h}(x, a)$.
\label{lem:parallel_homo_ucb}
\end{lemma}
\begin{proof}
The proof is done by induction, identical to the proof in Lemma B.5 of~\citet{jin2020provably}, and we urge the reader to refer to the aforementioned source.
\end{proof}
\begin{lemma}[Recursive Relation in Homogenous Settings]
Let $\delta^t_{m, h} = V^t_{m, h}(x^t_{m, h}) - V^{\pi_t}_{m, h}(x^t_{m, h})$, and $\xi^t_{m, h+1} = \bbE\left[\delta^t_{m, h} | x^t_{m, h}, a^t_{m, h}\right] - \delta^t_{m, h}$. Then, with probability at least $1-\alpha$, for all $(t, m, h) \in [T] \times \cM \times [H]$ simultaneously,
\begin{align}
    \delta^t_{m, h} \leq \delta^t_{m, h+1} + \xi^t_{m, h+1} + 2\left\lVert \bphi(x^t_{m, h}, a^t_{m, h})\right\rVert_{(\bLambda^t_{m, h})^{-1}}\cdot c_\beta\cdot dH\cdot\sqrt{2\log\left(\frac{dMTH}{\alpha}\right)}.
\end{align}
\label{lem:parallel_homo_recursion}
\end{lemma}
\begin{proof}
By Lemma~\ref{lem:parallel_homo_weight_delta}, we have that for any $(x, a, h, m, t) \in \cS \times \cA \times [H] \times \cM \times [T]$ with probability at least $1-\alpha/2$,
\begin{multline*}
    Q^t_{m, h}(x, a) - Q^{\pi_t}_{m, h}(x, a) \leq \bbP_h(V^t_{m, h+1} - V^{\pi_t}_{m, h})(x, a) \\+ 2\left\lVert \bphi(x, a)\right\rVert_{(\bLambda^t_{m, h})^{-1}}\cdot c_\beta\cdot dH\cdot\sqrt{2\log\left(\frac{dMTH}{\alpha}\right)}.
\end{multline*}
Replacing the definition of $\delta^t_{m, h}$ and $V^{\pi_t}_{m, h}$ finishes the proof.
\end{proof}
\begin{lemma}
For $ \xi^t_{m, h}$ as defined earlier and any $\delta \in (0, 1)$, we have with probability at least $1-\delta/2$,
\begin{align}
\sum_{t=1}^T \sum_{m=1}^M\sum_{h=1}^H \xi^t_{m, h}  \leq \sqrt{2H^3MT\log\left(\frac{2}{\alpha}\right)}.
\end{align}
\label{lem:parallel_homo_martingale}
\end{lemma}
\begin{proof}
We generalize the procedure from~\cite{jin2018q}, by demonstrating that the overall sums can be written as bounded martingale difference sequences with respect to an appropriately chosen filtration. For any $(t, m, h) \in [T] \times [M] \times [H]$, we define the $\sigma$-algebra $\cF_{t, m, h}$ as,
\begin{align}
    \cF_{t, m, h}  &= \sigma\left(\left\{ \left(x^\tau_{l, i}, a^\tau_{l, i}\right)\right\}_{(\tau, l, i) \in [t-1]\times [M] \times [H]} \cup \left\{ \left(x^t_{l, i}, a^t_{l, i}\right)\right\}_{(i, l) \in [h] \times [m-1]} \cup \left\{ \left(x^t_{m, i}, a^t_{m, i}\right)\right\}_{i \in [h]} \right)
\end{align}
Where we denote the $\sigma-$algebra generated by a finite set by $\sigma(\cdot)$. For any $t \in [T], m \in [M], h \in [H]$, we can define the timestamp index $\tau(t, m, h)$ as  $\tau(t, m, h) = (t-1)\cdot HM + h(m-1) + (h-1)$. We see that this ordering ensures that the $\sigma-$algebras from earlier form a filtration. We can see that for any agent $m \in [M]$, $Q^t_{m, h}$ and $V^t_{m, h}$ are both obtained based on the trajectories of the first $(t-1)$ episodes, and are both measurable with respect to $\cF_{t, 1, 1}$ (which is a subset of $\cF_{t, m, h}$ for all $h \in [H]$ and $m \in [M]$). Moreover, note that $a^t_{m, h} \sim \pi_{m, t}(\cdot | x^t_{m, h})$ and $x^t_{m, h+1} \sim \bbP_{m, h}(\cdot | x^t_{m, h}, a^t_{m, h})$. Therefore,
\begin{align}
 \bbE_{\bbP_{m, h}}[\xi^t_{m, h} | \cF_{t, m, h}] =  0.
\end{align}
where we set $\cF_{1, 0, 0}$ with the empty set. We define the martingale $\{U_{t, m, h}\}_{(t, h, m) \in [T]\times[M]\times[H]}$ indexed by $\tau(t, m, h)$ defined earlier, as follows. For any $(t, m, h) \in [T] \times [M] \times [H]$, we define
\begin{align}
    U_{t, m, h} =\left\{ \sum_{(a,b,c)} \xi^a_{b, c} : \tau(a, b, c) \leq \tau(t, m, h)\right\},
\end{align}
Additionally, we have that
\begin{align}
    U_{T, M, H} &= \sum_{t=1}^T \sum_{m=1}^M \sum_{h=1}^H  \xi^t_{m, h}.
\end{align}
Now, we have that for each $m \in \cM$, $V^t_{m, h}, Q^t_{m, h}, V^{\pi_{m, t}}_{m, h}, Q^{\pi_{m, t}}_{m, h}$ take values in $[0, H]$. Therefore, wh have that $ \xi^t_{m, h} \leq 2H$ for all $(t, m, h) \in [T] \times [M] \times [H]$. This allows us to apply the Azuma-Hoeffding inequality~\citep{azuma1967weighted} to $U_{T, M, H}$. We therefore obtain that for all $\tau > 0$,
\begin{align}
    \bbP\left( \sum_{t=1}^T \sum_{m=1}^M \sum_{h=1}^H\xi_{m, h}^t > \tau \right) \leq \exp\left(\frac{-\tau^2}{2H^3MT}\right).
\end{align}
Setting the RHS as $\alpha/2$, we obtain that with probability at least $1-\alpha/2$,
\begin{align}
    \sum_{t=1}^T \sum_{m=1}^M \sum_{h=1}^H \xi_{m, h}^t \leq \sqrt{2H^3MT\log\left(\frac{2}{\alpha}\right)}.
\end{align}
\end{proof}

\begin{lemma}[Variance control via communication in homogenous factored environments]
Let Algorithm~\ref{alg:ind_homo} be run for any $T > 0$ and $M \geq 1$, with $S$ as the communication control factor. Then, the following holds for the cumulative variance.
\begin{align}
    \sum_{m=1}^M\sum_{t=1}^T \left\lVert \bphi(z^t_{m, h}) \right\rVert_{(\bLambda^t_{m, h})^{-1}} \leq  2\log\left(\frac{\det\left( \bLambda^T_h\right)}{\det\left(\lambda\bI_d\right)}\right)\left(\frac{M}{\log 2}\right)\sqrt{S} + 2\sqrt{2MT\log\left(\frac{\det\left( \bLambda^T_h\right)}{\det\left(\lambda\bI_d\right)}\right)}.
\end{align}
\label{lem:parallel_homo_variance_sum}
\end{lemma}
\begin{proof}
Consider the following mappings $\nu_M, \nu_T : [MT] \rightarrow [M] \times [T]$.
\begin{align}
    \nu_M(\tau) = \tau (\text{mod } M), \text{and } \nu_T = \left\lceil \frac{\tau}{M}\right\rceil.
\end{align}
Now, consider $\bar\bLambda^\tau_h = \lambda\bI_d + \sum_{u=1}^\tau \bphi\left(z^{\nu_T(u)}_{\nu_M(u), h}\right)\bphi\left(z^{\nu_T(u)}_{\nu_M(u), h}\right)^\top$ for $\tau > 0$ and $\bar\bLambda^0_h = \lambda\bI_d$. Furthermore, assume that global synchronizations occur at round $\bsigma = (\sigma_{1}, ..., \sigma_{n})$ where there are a total of $n-1$ rounds of synchronization and $\sigma_{i} \in [T] \forall\ i \in [N-1]$ and $\sigma_{n} = T$, i.e., the final round.
Let $R_h = \left\lceil  \log\left(\frac{\det\left(\bar\bLambda^T_{h}\right)}{\det\left(\lambda\bI_d\right)}\right) \right\rceil$. It follows that there exist at most $R_h$ periods between synchronization (i.e., intervals $\sigma_{k-1}$ to $\sigma_k$ for $k \in [N]$) in which the following does not hold true:
\begin{align}
    \label{eqn:interval_sync_single_agent}
    1 \leq \frac{\det(\bar\bLambda^{\sigma_{k}}_h)}{\det(\bar\bLambda^{k-1}_h)}\leq 2.
\end{align}
Let us denote the event when Equation~(\ref{eqn:interval_sync_single_agent}) does holds for an interval $\sigma_{k-1}$ to $\sigma_{k}$ as $E$. Now, for any $t \in [\sigma_{k-1}, \sigma_{k}]$, we have, for any $m \in [M]$,
\begin{align}
    \left\lVert \bphi(z^t_{m, h}) \right\rVert_{(\bLambda^t_{m, h})^{-1}} &\leq \left\lVert \bphi(z^t_{m, h}) \right\rVert_{(\bar\bLambda^{t}_{h})^{-1}}\sqrt{\frac{\det\left(\bar\bLambda^t_{h}\right)}{\det\left(\bLambda^t_{m,h}\right)}} \leq \left\lVert \bphi(z^t_{m, h}) \right\rVert_{(\bar\bLambda^t_{h})^{-1}}\sqrt{\frac{\det\left(\bar\bLambda^{\sigma_{k}}_{ h}\right)}{\det\left(\bar\bLambda^{\sigma_{k-1}}_{h}\right)}} \\
    &\leq 2\left\lVert \bphi(z^t_{m, h}) \right\rVert_{(\bar\bLambda^{t}_{h})^{-1}}.
\end{align}
Here, the first inequality follows from the fact that $\bLambda^t_{m, h} \preccurlyeq \bar\bLambda^t_{h}$, the second inequality follows from the fact that $\bLambda^t_{m, h} \preccurlyeq \bar\bLambda^{\sigma_{k}}_{h} \implies \det(\bLambda^t_{m, h}) \leqslant \det(\bar\bLambda^{\sigma_{k}}_{h})$, and $\bLambda^t_{m, h} \succcurlyeq \bar\bLambda^{\sigma_{k-1}}_{h} \implies \det(\bLambda^t_{m, h}) \geqslant \det(\bar\bLambda^{\sigma_{k-1}}_{h})$; and the final inequality follows from the fact that event $E$ holds. Now, we can consider the partial sums only in the intervals for which event $E$ holds. For any $t \in [T]$, consider $\sigma(t) = \max_{i \in [N]} \{\sigma_{i} | \sigma_{i} \leq t\}$ denote the last round of synchronization prior to episode $t$. Then,
\begin{align}
    \sum_{t: E\text{ is true}}^T\sum_{m=1}^M \left\lVert \bphi(z^t_{m, h}) \right\rVert_{(\bLambda^t_{m, h})^{-1}} &\leq \sqrt{MT\sum_{m=1}^M\sum_{t: E\text{ is true}}^T \left\lVert \bphi(z^t_{m, h}) \right\rVert^2_{(\bLambda^t_{m, h})^{-1}}} \\
    \leq 2\sqrt{MT\sum_{m=1}^M\sum_{t: E\text{ is true}}^T \left\lVert \bphi(z^t_{m, h}) \right\rVert^2_{(\bar\bLambda^t_{h})^{-1}}}
    &\leq 2\sqrt{MT\sum_{m=1}^M\sum_{t=1}^T \left\lVert \bphi(z^t_{m, h}) \right\rVert^2_{(\bar\bLambda^t_{h})^{-1}}} \\
    =2\sqrt{MT\sum_{m=1}^M\sum_{\tau=1}^T \left\lVert \bphi(z^{\nu_T(\tau)}_{\nu_M(\tau), h}) \right\rVert^2_{(\bar\bLambda^t_{h})^{-1}}} 
    &= 2\sqrt{MT\log\left(\frac{\det\left( \bLambda^T_h\right)}{\det\left(\lambda\bI_d\right)}\right)}.
\end{align}
Here, the first inequality follows from Cauchy-Schwarz, the second inequality follows from the fact that event $E$ holds, and the final equality follows from Lemma~\ref{lem:variance_sum}. Now, we sum up the cumulative sum for episodes when $E$ does not hold. Consider an interval $\sigma_{k-1}$ to $\sigma_{k}$ for $k \in [N]$ of length $\Delta_{k} = \sigma_{k} - \sigma_{k-1}$ in which $E$ does not hold. We have that,
\begin{align}
    \sum_{m=1}^M\sum_{t=\sigma_{k-1}}^{\sigma_k} \left\lVert \bphi(z^t_{m, h}) \right\rVert_{(\bLambda^t_{m, h})^{-1}} &\leq \sum_{m=1}^M\sqrt{\Delta_{k, h}\sum_{t=\sigma_{k-1}}^{\sigma_k} \left\lVert \bphi(z^t_{m, h}) \right\rVert^2_{(\bLambda^t_{m, h})^{-1}}} \\
    \leq \sum_{m=1}^M\sqrt{\Delta_{k, h}\cdot \log_\omega\left(\frac{\det(\bLambda^{\sigma_{k}}_{m, h})}{\det(\bLambda^{\sigma_{k-1}}_{m, h})}\right)} 
    &\leq \sum_{m=1}^M\sqrt{\Delta_{k, h}\cdot \log_\omega\left(\frac{\det(\bar\bLambda^{\sigma_k}_{ h})}{\det(\bar\bLambda^{\sigma_{k-1}}_{ h})}\right)} \leq M\sqrt{S}.
\end{align}
The last inequality follows from the synchronization criterion. Now, note that there are at most $R_h$ periods in which event $E$ does not hold, and hence the total sum in this period can be bound as,
\begin{align}
    \sum_{(t : E \text{ is not true})}^T \sum_{m=1}^M\left\lVert \bphi(z^t_{m, h}) \right\rVert_{(\bLambda^t_{m, h})^{-1}} &\leq R_hM\sqrt{S} \leq \left(  \log\left(\frac{\det\left( \bLambda^T_h\right)}{\det\left(\lambda\bI_d\right)}\right)+ 1\right)M\sqrt{S}.
\end{align}
Therefore, we can bound the total variance as,
\begin{align*}
    \sum_{m=1}^M\sum_{t=1}^T \left\lVert \bphi(z^t_{m, h}) \right\rVert_{(\bLambda^t_{m, h})^{-1}} &\leq  \left(  \log\left(\frac{\det\left( \bLambda^T_h\right)}{\det\left(\lambda\bI_d\right)}\right) + 1\right)M\sqrt{S} + 2\sqrt{MT\log\left(\frac{\det\left( \bLambda^T_h\right)}{\det\left(\lambda\bI_d\right)}\right)} \\
    &\leq  2\log\left(\frac{\det\left( \bLambda^T_h\right)}{\det\left(\lambda\bI_d\right)}\right)\left(\frac{M}{\log 2}\right)\sqrt{S} + 2\sqrt{2MT\log\left(\frac{\det\left( \bLambda^T_h\right)}{\det\left(\lambda\bI_d\right)}\right)}.
\end{align*}
\end{proof}

We are now ready to prove Theorem~\ref{thm:ind_homo}. We first restate the Theorem for completeness.\\

\noindent\textbf{Theorem~\ref{thm:ind_homo}} (Homogenous Regret).
Algorithm~\ref{alg:ind_homo} when run on $M$ agents with communication threshold $S$, $\beta_t = \cO(H\sqrt{d\log (tMH)})$ and $\lambda = 1$ obtains the following cumulative regret after $T$ episodes, with probability at least $1-\alpha$,
\begin{equation*} 
\fR(T)=\widetilde\cO\left(d^{\frac{3}{2}}H^2\left(M\sqrt{S} + \sqrt{MT}\right)\sqrt{\log\left(\frac{1}{\alpha}\right)}\right).
\end{equation*}
\begin{proof}
    We have by the definition of group regret:
    \begin{align}
        \fR(T) &= \sum_{m=1}^M\sum_{t=1}^T V^\star_{m, 1}(x^t_{m, 1}) - V^{\pi_t}_{m, 1}(x^t_{m, 1}) \leq \sum_{m=1}^M\sum_{t=1}^T \delta^t_{m, 1} \\
        &\leq \sum_{m=1}^M\sum_{t=1}^T\sum_{h=1}^H \xi^t_{m, h} + 2c_\beta\cdot dH\cdot\sqrt{2\log\left(\frac{dMTH}{\alpha}\right)} \left(\sum_{t=1}^T\sum_{m=1}^M\sum_{h=1}^H \left\lVert \bphi(x, a)\right\rVert_{(\bLambda^t_{m, h})^{-1}} \right).
    \end{align}
Where the last inequality holds with probability at least $1-\alpha/2$, via Lemma~\ref{lem:parallel_homo_recursion} and Lemma~\ref{lem:parallel_homo_ucb}. Next, we can bound the first term via Lemma~\ref{lem:parallel_homo_martingale}. We have with probability at least $1-\alpha$, for some absolute constant $c_\beta$,
\begin{align}
     \fR(T) &\leq \sqrt{2H^3MT\log\left(\frac{2}{\alpha}\right)} + 2c_\beta\cdot dH\cdot\sqrt{2\log\left(\frac{dMTH}{\alpha}\right)} \left(\sum_{t=1}^T\sum_{m=1}^M\sum_{h=1}^H \left\lVert \bphi(x, a)\right\rVert_{(\bLambda^t_{m, h})^{-1}} \right).
\end{align}
Finally, to bound the summation, we use Lemma~\ref{lem:parallel_homo_variance_sum}. We have that,
\begin{align}
    \sum_{t=1}^T\sum_{m=1}^M\sum_{h=1}^H \left\lVert \bphi(x, a)\right\rVert_{(\bLambda^t_{m, h})^{-1}} &\leq 2\sum_{h=1}^H\left(\log\left(\frac{\det\left( \bLambda^T_h\right)}{\det\left(\lambda\bI_d\right)}\right)\left(\frac{M}{\log 2}\right)\sqrt{S} + 2\sqrt{2MT\log\left(\frac{\det\left( \bLambda^T_h\right)}{\det\left(\lambda\bI_d\right)}\right)}\right) \\
    &\leq 2H\log(dMT)M\sqrt{S} + 2\sqrt{2dMT\log(MT)}.
\end{align}
Where the last inequality is an application of the determinant-trace inequality and using the fact that $\lVert \bphi(\cdot) \rVert_2 \leq 1$. Replacing this result, we have that with probability at least $1-\alpha$,
\begin{align}
    \fR(T) &\leq \sqrt{2H^3MT\log\left(\frac{2}{\alpha}\right)} + 2c_\beta dH^2\sqrt{2\log\left(\frac{dMTH}{\alpha}\right)} \left(2\log(dMT)M\sqrt{S} + 2\sqrt{2dMT\log(MT)}\right) \\
    &= \ctO\left(d^{3/2}H^2\left(M\sqrt{S} + \sqrt{MT} \right)\sqrt{\log\left(\frac{1}{\alpha}\right)} \right).
\end{align}
\end{proof}

\newpage
\subsection{Proof of Theorem~\ref{thm:ind_hetero_small} (Small Heterogeneity)}

The proof for this section is very similar to that of Theorem~\ref{thm:ind_homo} with some modifications to handle the differences between MDPs as a case of model misspecification. First, we must bound the difference in the projected $Q-$values for any pair of MDPs under the small heterogeneity condition.

\begin{lemma}
Under the small heterogeneity condition (Assumption~\ref{assumption:bellman_small_deviation}), for any policy $\bpi$ over $\cS \times \cA$, let the corresponding weights at step $h$ for two MDPs $m, m' \in \cM$ be given by $\w^\pi_{m, h}, \w^{\pi}_{m', h}$ respectively, i.e., $Q^\pi_{m, h}(x, a) = \langle \bphi(x, a), \w^\pi_{m, h}\rangle$ and  $Q^\pi_{m', h}(x, a) = \langle \bphi(x, a), \w^\pi_{m', h}\rangle$. Then, we have for any $x, a \in \cS \times \cA$,
\begin{align*}
    \left| \langle \bphi(x, a), \w^\pi_{m, h}  - \w^\pi_{m', h} \rangle \right| \leq 2H\xi.
\end{align*}
\end{lemma}
\begin{proof}
    The proof follows from the fact that for any $h \in [H]$, 
    \begin{align}
         &\left| \langle \bphi(x, a), \w^\pi_{m, h}  - \w^\pi_{m', h} \rangle \right| \\
         &=  \left| Q^\pi_{m, h}(x, a)  - Q^\pi_{m', h}(x, a) \right| \\
         &\leq |r_{m, h}(x, a) - r_{m', h}(x, a)| + \left| \bbP_{m, h} V^\pi_{m, h+1} (x, a) -  \bbP_{m', h} V^\pi_{m', h+1} (x, a) \right| \\
         &\leq |r_{m, h}(x, a) - r_{m', h}(x, a)| + \sup_{x' \in \cS}|V^\pi_{m, h+1}(x') - V^\pi_{m', h+1}(x')|\cdot \lVert  (\bbP_{m, h}-  \bbP_{m', h})(x, a) \rVert_{\text{TV}}\\
         &\leq  |r_{m, h}(x, a) - r_{m', h}(x, a)| + H\cdot \lVert  (\bbP_{m, h}-  \bbP_{m', h})(x, a) \rVert_{\text{TV}}\\
         &\leq 2H\xi.
    \end{align}
    Here the last inequality follows from Assumption~\ref{assumption:bellman_small_deviation}.
\end{proof}

Now, we reproduce a general result bounding bias introduced by the potentially adversarial noise due to misspecification.
\begin{lemma}\label{lem:parallel_small_hetero_bias}
Let $\{\varepsilon_\tau\}_{\tau=1}^t$ be a sequence such that $| \varepsilon_\tau | \leq B$. We have, for any $(h, t, m) \in [H] \times [T] \times \cM$, and $\bphi \in \bbR^d$,
\begin{align*}
    |\bphi^\top(\bLambda^t_{m, h})^{-1}\sum_{\tau=1}^{U^m_h(t)}\bphi_\tau\varepsilon_\tau| &\leq B\sqrt{dMt}\lVert \bphi \rVert_{(\bLambda^t_{m, h})^{-1}}.
\end{align*}
\end{lemma}
\begin{proof}
    Recall that at any instant the collective set of observations possessed by an agent is given by $\cU^m_h(t)$ with size $U^m_h(t) \leq Mt$. We have that,
    \begin{align}
        |\bphi^\top(\bLambda^t_{m, h})^{-1}\sum_{\tau=1}^{U^m_h(t)}\bphi_\tau\varepsilon_\tau| &\leq B\cdot |\bphi^\top(\bLambda^t_{m, h})^{-1}\sum_{\tau=1}^{U^m_h(t)}\bphi_\tau| \\
        &\leq B\cdot \sqrt{\left[\sum_{\tau=1}^{U^m_h(t)}\bphi^\top(\bLambda^t_{m, h})^{-1}\bphi\right]\cdot\left[\sum_{\tau=1}^{U^m_h(t)}\bphi_\tau^\top(\bLambda^t_{m, h})^{-1}\bphi_\tau\right]}\\
        &\leq B\sqrt{dMt}\lVert \bphi \rVert_{(\bLambda^t_{m, h})^{-1}}.
    \end{align}
\end{proof}
Now we present the primary concentration result for the small heterogeneity setting.
\begin{lemma}
There exists an absolute constant $c_\beta$ such that for $\beta^t_{m, h} = c_\beta\cdot dH (\sqrt{\log(2dMHT/\delta')} + \xi\sqrt{dMT})$ for any policy $\pi$, there exists a constant $c_\beta$ such that for each $x \in \cS, a \in \cA$ we have for all $ m \in \cM, t \in[T], h \in [H]$ simultaneously, with probability at least $1-\delta'/2$,
\begin{multline*}
    \left| \langle \bphi(x, a), \w^t_{m, h} - \w^\pi_{m, h} \rangle \right| \leq \\ \bbP_{m, h}(V^t_{m, h+1} - V^\pi_{m, h+1})(x, a) + c_\beta\cdot dH\cdot\lVert \bphi(z) \rVert_{(\bLambda^t_{m, h})^{-1}} \left(\sqrt{2\log\left(\frac{dMTH}{\delta'}\right)} + 2\xi\sqrt{dMT}\right).
\end{multline*}
\label{lem:parallel_small_hetero_delta}
\end{lemma}
\begin{proof}
    By the Bellman equation and the assumption of the linear MDP (Definition~\ref{def:linear_mdp}), we have that for any policy $\bpi$, there exist weights $\w^\pi_{m, h}$ such that, for all $z \in \cZ$,
\begin{align}
    \langle \bphi(z), \w^\pi_{m, h} \rangle = r_{m, h}(z) + \bbP_{m, h} V^\pi_{h+1}(z).
\end{align}
Recall that the set of all observations available to any agent at instant $t$ is given by $\cU^m_h(t)$ for step $h$, with the cardinality of this set being $U^h_m(t)$. For convenience, let us assume an ordering $\tau = 1, ..., U^h_m(t)$ over this set and use the shorthand $U_m = U^h_m(t)$. Therefore, we have, for any $m \in \cM$,
\begin{align}
    &\w^t_{m, h} - \w^\pi_{m, h} \\
    &= (\bLambda_{m, h}^t)^{-1}\sum_{\tau=1}^{U_m}\left[\bphi_\tau[(r_h+V^t_{m, h+1})(x_\tau)]\right] - \w^\pi_{m, h} \\
    &= (\bLambda_{m, h}^t)^{-1}\left\{ -\lambda\w^\pi_h + \sum_{\tau=1}^{U_m}\left[\bphi_\tau[V^t_{m, h+1}(x'_\tau) -\bbP_{m_\tau, h} V^\pi_{m, h+1}(x_\tau, a_\tau)]\right]\right\}.
\end{align}
\begin{multline}
   \implies \w^t_{m, h} - \w^\pi_{m, h} = \underbrace{-\lambda(\bLambda_{m, h}^t)^{-1}\w^\pi_{m, h}}_{\bv_1} + \underbrace{(\bLambda_{m, h}^t)^{-1}\left\{\sum_{\tau=1}^{U_m}\left[\bphi_\tau[V^t_{m, h+1}(x'_\tau) -\bbP_{m, h} V^t_{m, h+1}(z_\tau)]\right] \right\}}_{\bv_2} \\ + \underbrace{(\bLambda_{m, h}^t)^{-1}\left\{\sum_{\tau=1}^{U_m}\left[\bphi_\tau[\bbP_{m, h} V^t_{m, h+1} - \bbP_{m, h} V^\pi_{m, h+1})(z_\tau)]\right] \right\}}_{\bv_3} \\ + \underbrace{(\bLambda_{m, h}^t)^{-1}\left\{\sum_{\tau=1}^{U_m}\left[\bphi_\tau[\bbP_{m, h} V^t_{m, h+1} - \bbP_{m_\tau, h} V^t_{m, h+1})(z_\tau)]\right] \right\}}_{\bv_4} \\ + \underbrace{(\bLambda_{m, h}^t)^{-1}\left\{\sum_{\tau=1}^{U_m}\left[\bphi_\tau[\bbP_{m, h} V^\pi_{m, h+1} - \bbP_{m_\tau, h} V^\pi_{m, h+1})(z_\tau)]\right] \right\}}_{\bv_5}.
\end{multline}
Now, we know that for any $z \in \cZ$ for any policy $\pi$,
\begin{align}
    \left| \langle \bphi(z), \bv_1 \rangle \right| \leq \lambda\left| \langle \bphi(z), \bLambda_{m, h}^t)^{-1}\w^\pi_h \rangle \right| \leq \lambda\cdot \lVert \w^\pi_h \rVert \lVert \bphi(z) \rVert_{(\bLambda^t_{m, h})^{-1}} \leq 2H\lambda\sqrt{d}\lVert \bphi(z) \rVert_{(\bLambda^t_{m, h})^{-1}}
\end{align}
Here the last inequality follows from Lemma~\ref{lem:bound_homo_policy_weight}. For the second term, we have by Lemma~\ref{lem:parallel_homo_lsvi} that there exists an absolute constant $C$ independent of $M, T, H, d$ and $c_\beta$, such that, with probability at least $1-\delta'/2$ for all $m \in \cM, t \in [T], h \in [H]$ simultaneously,
\begin{align}
    \left| \langle \bphi(z), \bv_2 \rangle \right| \leq \left\lVert \bphi(z)\right\rVert_{(\bLambda^t_{m, h})^{-1}}\cdot c_\beta\cdot dH\cdot\sqrt{2\log\left(\frac{dMTH}{\delta'}\right)}.
\end{align}
For the third term, note that,
\begin{align}
    &\left| \langle \bphi(z), \bv_3 \rangle \right| \\
    &= \left\langle \bphi(z), (\bLambda_{m, h}^t)^{-1}\left\{\sum_{\tau=1}^{U_m}\left[\bphi_\tau[\bbP_h V^t_{m, h+1} - \bbP_h V^\pi_{m, h+1})(z_\tau)]\right] \right\} \right\rangle \\
    &= \left\langle \bphi(z), (\bLambda_{m, h}^t)^{-1}\sum_{\tau=1}^{U_m}\left[\bphi_\tau\bphi^\top_\tau\int(V^t_{m, h+1} -  V^\pi_{m, h+1})(x')d\bmu_h(x')\right]\right\rangle \\
    &= \left\langle \bphi(z), \int(V^t_{m, h+1} -  V^\pi_{m, h+1})(x')d\bmu_h(x')\right\rangle -\lambda\left\langle \bphi(z), (\bLambda_{m, h}^t)^{-1}\int(V^t_{m, h+1} -  V^\pi_{m, h+1})(x')d\bmu_h(x')\right\rangle\\
    &= \bbP_h(V^t_{m, h+1} - V^\pi_{m, h+1})(x, a) -\lambda\left\langle \bphi(z), (\bLambda_{m, h}^t)^{-1}\int(V^t_{m, h+1} -  V^\pi_{m, h+1})(x')d\bmu_h(x')\right\rangle\\
    &= \bbP_h(V^t_{m, h+1} - V^\pi_{m, h+1})(x, a) + 2H\sqrt{d\lambda}\lVert \bphi(z) \rVert_{(\bLambda^t_{m, h})^{-1}}\\
\end{align}
For the fourth and fifth terms, we have that both $[\bbP_{m, h} V^\pi_{m, h+1} - \bbP_{m_\tau, h} V^\pi_{m, h+1})(z_\tau)]$ and $[\bbP_{m, h} V^t_{m, h+1} - \bbP_{m_\tau, h} V^t_{m, h+1})(z_\tau)]$ are bounded by $H\xi$ (from Assumption~\ref{assumption:bellman_small_deviation} and the fact that the value functions are always smaller than $H$). This gives us, by Lemma~\ref{lem:parallel_small_hetero_bias},
\begin{align}
    \left| \langle \bphi(z), \bv_4 + \bv_5 \rangle \right| \leq 2H\xi\sqrt{dMt}\lVert \bphi(z) \rVert_{(\bLambda^t_{m, h})^{-1}}
\end{align}
Putting it all together, we have that since $\langle \bphi(z), \w^t_{m, h} - \w^\pi_{m, h} \rangle = \langle \bphi(z), \bv_1 + \bv_2 + \bv_3 + \bv_4 + \bv_5\rangle$, there exists an absolute constant $C$ independent of $M, T, H, d$ and $c_\beta$, such that, with probability at least $1-\delta'/2$ for all $m \in \cM, t \in [T], h \in [H]$ simultaneously,
\begin{multline}
    \left| \langle \bphi(x, a), \w^t_{m, h} - \w^\pi_{m, h} \rangle \right| \leq  \bbP_{m, h}(V^t_{m, h+1} - V^\pi_{m, h+1})(x, a) + \\  \lVert \bphi(z) \rVert_{(\bLambda^t_{m, h})^{-1}}\left( C\cdot dH\cdot\sqrt{2\log\left((c_\beta+2)\frac{dMTH}{\delta'}\right)} + 2H\sqrt{d\lambda} + 2H\lambda\sqrt{d} +  2H\xi\sqrt{dMT}\right)
\end{multline}
Since $\lambda \leq 1$ and since $C$ is independent of $c_\beta$, we can select $c_\beta$ such that we have the following for any $(x, a) \in \cS \times \cA$ with probability $1-\delta'/2$ simultaneously for all $h \in [H], m \in \cM, t \in [T]$,
\begin{multline}
    \left| \langle \bphi(x, a), \w^t_{m, h} - \w^\pi_{m, h} \rangle \right| \leq \\ \bbP_{m, h}(V^t_{m, h+1} - V^\pi_{m, h+1})(x, a) + c_\beta\cdot dH\cdot\lVert \bphi(z) \rVert_{(\bLambda^t_{m, h})^{-1}} \left(\sqrt{2\log\left(\frac{dMTH}{\delta'}\right)} + 2\xi\sqrt{dMT}\right).
\end{multline}
\end{proof}
We now present an analogous recursive relationship in the small heterogeneity setting.

\begin{lemma}[Recursive Relation in Small Heterogeneous Settings]
Let $\delta^t_{m, h} = V^t_{m, h}(x^t_{m, h}) - V^{\pi_t}_{m, h}(x^t_{m, h})$, and $\xi^t_{m, h+1} = \bbE\left[\delta^t_{m, h} | x^t_{m, h}, a^t_{m, h}\right] - \delta^t_{m, h}$. Then, with probability at least $1-\alpha$, for all $(t, m, h) \in [T] \times \cM \times [H]$ simultaneously,
\begin{align*}
    \delta^t_{m, h} \leq \delta^t_{m, h+1} + \xi^t_{m, h+1} +  c_\beta\cdot dH\cdot\lVert \bphi(x, a) \rVert_{(\bLambda^t_{m, h})^{-1}} \left(\sqrt{2\log\left(\frac{dMTH}{\delta'}\right)} + 2\xi\sqrt{dMT}\right).
\end{align*}
\label{lem:parallel_small_hetero_recursion}
\end{lemma}
\begin{proof}
By Lemma~\ref{lem:parallel_small_hetero_delta}, we have that for any $(x, a, h, m, t) \in \cS \times \cA \times [H] \times \cM \times [T]$ with probability at least $1-\alpha/2$,
\begin{multline}
    Q^t_{m, h}(x, a) - Q^{\pi_t}_{m, h}(x, a) = \left| \langle \bphi(x, a), \w^t_{m, h} - \w^\pi_{m, h} \rangle \right| \leq \\ \bbP_{m, h}(V^t_{m, h+1} - V^\pi_{m, h+1})(x, a) + c_\beta\cdot dH\cdot\lVert \bphi(x, a) \rVert_{(\bLambda^t_{m, h})^{-1}} \left(\sqrt{2\log\left(\frac{dMTH}{\delta'}\right)} + 2\xi\sqrt{dMT}\right).
\end{multline}
Replacing the definition of $\delta^t_{m, h}$ and $V^{\pi_t}_{m, h}$ finishes the proof.
\end{proof}

We are now ready to prove Theorem~\ref{thm:ind_hetero_small}. We first restate the Theorem for completeness.\\

\noindent\textbf{Theorem~\ref{thm:ind_hetero_small}}.
Algorithm~\ref{alg:ind_homo} when run on $M$ agents with parameter $S$ in the small deviation setting (Assumption~\ref{assumption:bellman_small_deviation}), with $\beta_t = \cO(H\sqrt{d\log (tMH)} + \xi\sqrt{dMT})$ and $\lambda = 1$ obtains the following cumulative regret after $T$ episodes, with probability at least $1-\alpha$,
\begin{equation*} 
\fR(T) = \ctO\left(d^{3/2}H^2\left(M\sqrt{S} + \sqrt{MT} \right)\left(\sqrt{\log\left(\frac{1}{\alpha}\right)} + 2\xi\sqrt{dMT}\right)\right).
\end{equation*}
\begin{proof}
    We have by the definition of group regret:
    \begin{align}
        &\fR(T)\\
        &= \sum_{m=1}^M\sum_{t=1}^T V^\star_{m, 1}(x^t_{m, 1}) - V^{\pi_t}_{m, 1}(x^t_{m, 1}) \leq \sum_{m=1}^M\sum_{t=1}^T \delta^t_{m, 1} \\
        &\leq \sum_{m=1}^M\sum_{t=1}^T\sum_{h=1}^H \xi^t_{m, h} + 4c_\beta\cdot dH\cdot\left(\sqrt{\log\left(\frac{dMTH}{\alpha}\right)} +  \xi\sqrt{dMT}\right) \left(\sum_{t=1}^T\sum_{m=1}^M\sum_{h=1}^H \left\lVert \bphi(x, a)\right\rVert_{(\bLambda^t_{m, h})^{-1}} \right).
    \end{align}
Where the last inequality holds with probability at least $1-\alpha/2$, via Lemma~\ref{lem:parallel_small_hetero_recursion} and Lemma~\ref{lem:parallel_homo_ucb}. Next, we can bound the first term via Lemma~\ref{lem:parallel_homo_martingale}. We have with probability at least $1-\alpha$, for some absolute constant $c_\beta$,
\begin{multline*}
     \fR(T) \leq \sqrt{2H^3MT\log\left(\frac{2}{\alpha}\right)}\\ + 4c_\beta\cdot dH\cdot\left(\sqrt{\log\left(\frac{dMTH}{\alpha}\right)}+ \xi\sqrt{dMT}\right) \left(\sum_{t=1}^T\sum_{m=1}^M\sum_{h=1}^H \left\lVert \bphi(x, a)\right\rVert_{(\bLambda^t_{m, h})^{-1}} \right).
\end{multline*}
Finally, to bound the summation, we use Lemma~\ref{lem:parallel_homo_variance_sum}. We have that,
\begin{align}
    \sum_{t=1}^T\sum_{m=1}^M\sum_{h=1}^H \left\lVert \bphi(x, a)\right\rVert_{(\bLambda^t_{m, h})^{-1}} &\leq 2\sum_{h=1}^H\left(\log\left(\frac{\det\left( \bLambda^T_h\right)}{\det\left(\lambda\bI_d\right)}\right)\left(\frac{M}{\log 2}\right)\sqrt{S} + 2\sqrt{2MT\log\left(\frac{\det\left( \bLambda^T_h\right)}{\det\left(\lambda\bI_d\right)}\right)}\right) \\
    &\leq 2H\log(dMT)M\sqrt{S} + 2\sqrt{2dMT\log(MT)}.
\end{align}
Where the last inequality is an application of the determinant-trace inequality and using the fact that $\lVert \bphi(\cdot) \rVert_2 \leq 1$. Replacing this result, we have that with probability at least $1-\alpha$,
\begin{multline*}
    \fR(T) \leq \sqrt{2H^3MT\log\left(\frac{2}{\alpha}\right)} \\
    + 4c_\beta\cdot dH^2\cdot\left(\sqrt{\log\left(\frac{dMTH}{\alpha}\right)}+ \xi\sqrt{dMT}\right) \left(2\log(dMT)M\sqrt{S} + 2\sqrt{2dMT\log(MT)}\right)
\end{multline*}
\begin{align*}
    \implies \fR(T)  = \ctO\left(d^{3/2}H^2\left(M\sqrt{S} + \sqrt{MT} \right)\left(\sqrt{\log\left(\frac{1}{\alpha}\right)} + 2\xi\sqrt{dMT}\right)\right).
\end{align*}
\end{proof}

\newpage
\subsection{Proof for Theorem~\ref{thm:ind_hetero_large} (Large Heterogeneity)}
The proof for this section is largely similar to that of Theorem~\ref{thm:ind_homo}, however since we use the modified feature, the analysis differs in several key places. First we introduce the basic result which relates the variance with the coefficient of heterogeneity.

\begin{lemma}[Variance Decomposition] 
\label{lem:parallel_large_hetero_variance_decomp}
Under the heterogeneous parallel MDP assumption (Definition~\ref{def:bellman_hetero}) and coefficient of heterogeneity defined in Definition~\ref{def:coff_heterogeneity}, we have that,
\begin{align*}
    \max_{h \in [H]} \log\det\left(\btLambda^T_h\right) &\leq \left(d+\lambda+ \chi\right)\log(MT).
\end{align*}
\end{lemma}
\begin{proof}
   We know, from the form of $\btLambda^T_h$ that, 
   \begin{align}
       \log\det\left(\btLambda^T_h\right) &= \log\det\left((\btPhi^T_h)^\top(\btPhi^T_h\right) + \lambda\bI_{d+k}) = \log\det\left((\btPhi^T_h)(\btPhi^T_h)^\top + \lambda\bI_{MT}\right).
   \end{align}
   Here, $\btPhi^T_h \in \bbR^{MT\times (d+k)}$ is the matrix of all features $\btphi(x, a, m)$ for step $h$ until episode $T$. Now, observe that the matrix $(\btPhi^T_h)(\btPhi^T_h)^\top$ can be rewritten as the sum of two matrices $(\btPhi^T_h)(\btPhi^T_h)^\top = (\bPhi^T_h)(\bPhi^T_h)^\top + \btK_h^T$, where $[\btK_h^T]_{i, j} = \bnu(m_i)^\top\bnu(m_j), \btK_h^T \in \bbR^{MT \times MT}$, i.e., the corresponding dot-product contribution from the agent-specific features between any pair of transitions, and $(\bPhi^T_h)(\bPhi^T_h)^\top$ refers to the regular (agent-agnostic) features, i.e., $[(\bPhi^T_h)(\bPhi^T_h)^\top]_{i, j} = \bphi_i^\top\bphi_j$. Now, from Theorem IV of~\citet{madiman2008entropy}, we have that,
   \begin{align}
        \log\det\left((\btPhi^T_h)(\btPhi^T_h)^\top + \lambda\bI_{MT}\right) &\leq \log\det\left((\bPhi^T_h)(\bPhi^T_h)^\top + \lambda\bI_{MT}\right) + \log\det\left(\btK_h^T + \lambda\bI_{MT}\right) \\
        &= \log\det\left((\bPhi^T_h)^\top(\bPhi^T_h) + \lambda\bI_{d}\right) + \log\det\left(\btK_h^T + \lambda\bI_{MT}\right) \\
        &\leq d\log(MT) +  \log\det\left(\btK_h^T + \lambda\bI_{MT}\right) \\
        &\leq d\log(MT) +  \lambda\log(MT) + \text{rank}(\btK^T_h)\cdot\log(MT) \\
        &= (d+\lambda)\log(MT) + \text{rank}(\bK^\kappa_h)\log(MT).
   \end{align}
The second inequality follows from $\lVert \bphi(\cdot) \rVert \leq 1$ and then applying an AM-GM inequality followed by the determinant-trace inequality (as is common in bandit analyses). The final equality follows by the fact that since $\btK_h^T$ is $T \times T$ tiles of $\bK^\kappa_h$ followed by permutations, which implies that $\text{rank}(\btK_h^T) = \text{rank}(\bK^\kappa_h)$. Taking the maximum over all $h \in [H]$ and gives us the result.
\end{proof}

We first present a variant of the previous concentration result to bound the least-squares value iteration error (analog of Lemma~\ref{lem:parallel_homo_lsvi}).
\begin{lemma}
Under the setting of Theorem~\ref{thm:ind_hetero_large}, let $c'_\beta$ be the constant defining $\beta$, and $\btS^t_{m, h}$ and $\btLambda^k_t$ be defined as follows.
\begin{multline*}
    \btS^t_{m, h} = \sum_{n=1}^M \sum_{\tau=1}^{k_t}\btphi(n, x^\tau_{n, h}, a^\tau_{n, h}) \left[V^t_{m, h+1}(n, x^\tau_{n, h+1})- (\bbP_{m, h}V^t_{m, h+1})(n, x^\tau_{n, h}, a^\tau_{n, h})\right] \\+ \sum_{\tau = k_t + 1}^{t-1} \btphi(n, x^\tau_{m, h}, a^\tau_{m, h})\left[V^t_{m, h+1}(m, x^\tau_{m, h+1})- (\bbP_{m, h}V^t_{m, h+1})(m, x^\tau_{m, h}, a^\tau_{m, h})\right],
\end{multline*} 
\begin{align*}
    \btLambda^t_{m, h} =  \sum_{n=1}^M \sum_{\tau=1}^{k_t}\btphi(n, x^\tau_{n, h}, a^\tau_{n, h})\btphi(n, x^\tau_{n, h}, a^\tau_{n, h})^\top + \sum_{\tau = k_t + 1}^{t-1} \btphi(n, x^\tau_{m, h}, a^\tau_{m, h})\btphi(n, x^\tau_{m, h}, a^\tau_{m, h})^\top + \lambda\bI_{d+k}.
\end{align*}
Where $V \in \cV$ and $\cN_\epsilon$ denotes the $\epsilon-$covering of the value function space $\cV$. Then, there exists an absolute constant $C$ independent of $M, T, H, d$ and $c'_\beta$, such that, with probability at least $1-\delta'/2$ for all $m \in \cM, t \in [T], h \in [H]$ simultaneously,
\begin{align*}
    \left\lVert \btS^t_{m, h} \right\rVert_{(\btLambda^t_{m, h})^{-1}} &\leq C\cdot (d+k)H\sqrt{2\log\left(\frac{(c'_\beta + 2)(d+k)MTH}{\delta'}\right)}.
\end{align*}
\label{lem:parallel_large_hetero_lsvi}
\end{lemma}
\begin{proof}
   The proof is identical to that of Lemma~\ref{lem:parallel_homo_lsvi}, except that we utilize the combined features of dimensionality $(d+k)$, which requires us to select an alternative constant $c'_\beta$ in the bound.
\end{proof}

\begin{lemma}\label{lem:parallel_large_hetero_weight_delta}
There exists an absolute constant $c'_\beta$ such that for $\beta^t_{m, h} = c'_\beta\cdot dH \sqrt{\log(2(d+k)MHT/\delta')}$ for any policy $\pi$, there exists a constant $c_\beta$ such that for each $x \in \cS, a \in \cA$ we have for all $ m \in \cM, t \in[T], h \in [H]$ simultaneously, with probability at least $1-\delta'/2$,
\begin{multline*}
     \left| \langle \btphi(n, x, a), \w^t_{m, h} - \w^\pi_{m, h} \rangle \right| \leq  \bbP_{m, h}(V^t_{m, h+1} - V^\pi_{m, h+1})(n, x, a) \\
     + c_\beta\cdot (d+k)H\cdot\lVert \btphi(n, z) \rVert_{(\btLambda^t_{m, h})^{-1}}\cdot \sqrt{2\log\left(\frac{(d+k)MTH}{\delta'}\right)}.
\end{multline*}
\end{lemma}
\begin{proof}
   The proof for this is identical to Lemma~\ref{lem:parallel_homo_weight_delta}, however we modify the application of Lemma~\ref{lem:parallel_homo_lsvi} with Lemma~\ref{lem:parallel_large_hetero_lsvi} instead.
\end{proof}

\begin{lemma}[UCB in the Heterogeneous Setting]
With probability at least $1-\delta'/2$, we have that for all $(x, a, h, t, m) \in \cS \times \cA \times [H] \times [T] \times \cM$, $Q^t_{m, h}(x, a) \geq Q^\star_{m, h}(x, a)$.
\label{lem:parallel_large_hetero_ucb}
\end{lemma}
\begin{proof}
The proof is done by induction, identical to the proof in Lemma B.5 of~\citet{jin2020provably}, and we urge the reader to refer to the aforementioned source.
\end{proof}
\begin{lemma}[Recursive Relation in Heterogeneous Settings]
Let $\delta^t_{m, h} = V^t_{m, h}(x^t_{m, h}) - V^{\pi_t}_{m, h}(x^t_{m, h})$, and $\xi^t_{m, h+1} = \bbE\left[\delta^t_{m, h} | x^t_{m, h}, a^t_{m, h}\right] - \delta^t_{m, h}$. Then, with probability at least $1-\alpha$, for all $(t, m, h) \in [T] \times \cM \times [H]$ simultaneously,
\begin{align}
    \delta^t_{m, h} \leq \delta^t_{m, h+1} + \xi^t_{m, h+1} + 2\left\lVert \btphi(m, x^t_{m, h}, a^t_{m, h})\right\rVert_{(\btLambda^t_{m, h})^{-1}}\cdot c'_\beta\cdot (d+k)H\cdot\sqrt{2\log\left(\frac{(d+k)MTH}{\alpha}\right)}.
\end{align}
\label{lem:parallel_large_hetero_recursion}
\end{lemma}
\begin{proof}
By Lemma~\ref{lem:parallel_large_hetero_weight_delta}, we have that for any $(x, a, h, m, t) \in \cS \times \cA \times [H] \times \cM \times [T]$ with probability at least $1-\alpha/2$,
\begin{multline}
    Q^t_{m, h}(x, a) - Q^{\pi_t}_{m, h}(x, a) \leq \bbP_{m, h}(V^t_{m, h+1} - V^{\pi_t}_{m, h})(x, a)\\ + 2\left\lVert \btphi(x, a)\right\rVert_{(\btLambda^t_{m, h})^{-1}}\cdot c'_\beta\cdot (d+k)H\cdot\sqrt{2\log\left(\frac{(d+k)MTH}{\alpha}\right)}.
\end{multline}
Replacing the definition of $\delta^t_{m, h}$ and $V^{\pi_t}_{m, h}$ finishes the proof.
\end{proof}
We are now ready to prove Theorem~\ref{thm:ind_hetero_large}. We first restate the Theorem for completeness.\\

\noindent\textbf{Theorem~\ref{thm:ind_hetero_large}}.
Algorithm~\ref{alg:ind_hetero} when run on $M$ agents with parameter $S$ in the heterogeneous setting (Definition~\ref{def:bellman_hetero}), with $\beta_t = \cO(H\sqrt{(d+k)\log (tMH)})$ and $\lambda = 1$ obtains the following cumulative regret after $T$ episodes, with probability at least $1-\alpha$,
\begin{align*}
\fR(T)=\ctO\left((d+k)H^2\left(M(d +\chi)\sqrt{S} + \sqrt{(d +\chi)MT} \right)\sqrt{\log\left(\frac{1}{\alpha}\right)} \right).
\end{align*}
\begin{proof}
    We have by the definition of group regret:
    \begin{align}
        &\fR(T)\\ &= \sum_{m=1}^M\sum_{t=1}^T V^\star_{m, 1}(x^t_{m, 1}) - V^{\pi_t}_{m, 1}(x^t_{m, 1}) \leq \sum_{m=1}^M\sum_{t=1}^T \delta^t_{m, 1} \\
        &\leq \sum_{m=1}^M\sum_{t=1}^T\sum_{h=1}^H \xi^t_{m, h} + 2c'_\beta\cdot (d+k)H\cdot\sqrt{2\log\left(\frac{(d+k)MTH}{\alpha}\right)} \left(\sum_{t=1}^T\sum_{m=1}^M\sum_{h=1}^H \left\lVert \btphi(m, x, a)\right\rVert_{(\btLambda^t_{m, h})^{-1}} \right).
    \end{align}
Where the last inequality holds with probability at least $1-\alpha/2$, via Lemma~\ref{lem:parallel_large_hetero_recursion} and Lemma~\ref{lem:parallel_large_hetero_ucb}. Next, we can bound the first term via Lemma~\ref{lem:parallel_homo_martingale}. We have with probability at least $1-\alpha$, for some absolute constant $c'_\beta$,
\begin{multline}
     \fR(T) \leq \sqrt{2H^3MT\log\left(\frac{2}{\alpha}\right)} \\ + 2c'_\beta\cdot (d+k)H\cdot\sqrt{2\log\left(\frac{(d+k)MTH}{\alpha}\right)} \left(\sum_{t=1}^T\sum_{m=1}^M\sum_{h=1}^H \left\lVert \btphi(m, x, a)\right\rVert_{(\btLambda^t_{m, h})^{-1}} \right).
\end{multline}
Finally, to bound the summation, we use Lemma~\ref{lem:parallel_homo_variance_sum}. We have that,
\begin{align}
    &\sum_{t=1}^T\sum_{m=1}^M\sum_{h=1}^H \left\lVert \btphi(x, a)\right\rVert_{(\btLambda^t_{m, h})^{-1}} \\
    &\leq 2\sum_{h=1}^H\left(\log\left(\frac{\det\left( \btLambda^T_h\right)}{\det\left(\lambda\bI_d\right)}\right)\left(\frac{M}{\log 2}\right)\sqrt{S} + 2\sqrt{2MT\log\left(\frac{\det\left( \btLambda^T_h\right)}{\det\left(\lambda\bI_d\right)}\right)}\right) \\
    &\leq 2H(d+\chi)\log(MT)M\sqrt{S} + 2H\sqrt{2(d+\chi)MT\log(MT)}.
\end{align}
Where the last inequality is an application of the variance decomposition (Lemma~\ref{lem:parallel_large_hetero_variance_decomp}) and using the fact that $\lVert \bphi(\cdot) \rVert_2 \leq 1$. Replacing this result, we have that with probability at least $1-\alpha$,
\begin{multline}
    \fR(T) \leq \sqrt{2H^3MT\log\left(\frac{2}{\alpha}\right)} \\
    + 2c'_\beta\cdot (d+k)H^2\cdot\sqrt{2\log\left(\frac{dMTH}{\alpha}\right)} \left(2\log(MT)M(d +\chi)\sqrt{\cdot S} + 2\sqrt{2(d +\chi)MT\log(MT)}\right).
\end{multline}
\begin{align}
    \implies  \fR(T) = \ctO\left((d+k)H^2\left(M(d +\chi)\sqrt{S} + \sqrt{(d +\chi)MT} \right)\sqrt{\log\left(\frac{1}{\alpha}\right)} \right).
\end{align}
\end{proof}
\newpage
\section{Multiagent MDP Proofs}
\subsection{Proof of Proposition~\ref{prop:mg_bellman_optimal}}
We restate the Proposition for clarity.

\noindent\textbf{Proposition~\ref{prop:mg_bellman_optimal}.} For the scalarized value function given in Equation~\ref{eqn:v_scalarized}, the Bellman optimality conditions are given as, for all $h \in [H], \x \in \cS, \ba \in \cA$ for any fixed $\bupsilon \in \bUpsilon$,
\begin{align*}
    Q^\star_{\bupsilon, h}(\x, \ba) = \fs_\bupsilon\br_h(\x, \ba) + \bbP_h V^\star_{\bupsilon, h}(\x, \ba), V^\star_{\bupsilon, h}(\x) = \max_{\ba \in \cA} Q^\star_{\bupsilon, h}(\x, \ba)\text{, and } V^\star_{\bupsilon, H+1}(\x) = 0.
\end{align*}
\begin{proof}
We prove the above result by reducing the scalarized MMDP to an equivalent MDP. Observe that for any fixed $\bupsilon \in \bUpsilon$, the (vector-valued) rewards can be scalarized to a scalar reward. For any step $h \in [H]$, for any \textit{fixed} $\bupsilon \in \bUpsilon$, consider the MDP with state space $\cS = \cS_1 \times ... \times \cS_M$, action space $\cA = \cA_1 \times ... \times \cA_M$ and reward function $r'_h$ such that for all $(\x, \ba) \in \cS \times \cA, r'_h(\x, \ba) = \bupsilon^\top\br_h(\x, \ba)$. Therefore $r'_h(\x, \ba) \in [0, 1]$ (since $\br_h$ lies on the $M-$dimensional simplex). Therefore, if the group of agents cooperate to optimize the scalarized reward (for any \textit{fixed} scalarization parameter), the optimal (joint) policy coincides with the optimal policy for the aforementioned MDP defined over the \textit{joint} state and action spaces. The optimal policy for the scalarized MDP is given by the greedy policy with respect to the following parameters:
\begin{align}
    Q^\star_{\bupsilon, h}(\x, \ba) = r'_h(\x, \ba) + \bbP_h V^\star_{\bupsilon, h}(\x, \ba), V^\star_{\bupsilon, h}(\x) = \max_{\ba \in \cA} Q^\star_{\bupsilon, h}(\x, \ba)\text{, and } V^\star_{\bupsilon, H+1}(\x) = 0.
\end{align}
Replacing the reward function with the vector-valued reward in terms of $\bupsilon$ provides us the result.
\end{proof}

\subsection{Proof of Proposition~\ref{prop:mg_pareto_frontier}}

We first restate the Proposition for clarity.

\noindent\textbf{Proposition~\ref{prop:mg_pareto_frontier}}. For any parameter $\bupsilon \in \bUpsilon$, the optimal greedy policy $\bpi^\star_\bupsilon$ with respect to the scalarized $Q$-value that satisfies Proposition~\ref{prop:mg_bellman_optimal} lies in the Pareto frontier $\bPi^\star$.

\begin{proof}
The proof proceeds by contradiction. Assume that $\bpi^\star_\bupsilon$ does not lie in the Pareto frontier, then there exists a policy $\bpi' \in \bPi$ such that $\bV^{\bpi'}_1(\x) \succeq \bV^{\bpi^\star_\bupsilon}_1(\x)$ for all $\x \in \cS$ and $\bpi \neq \bpi^\star_\bupsilon$. Consider the final step $H$. Then, for any state $\x \in \cS$, we have that if $\bV^{\bpi'}_H(\x) \succeq \bV^{\bpi^\star_\bupsilon}_H(\x)$, then,
\begin{align}
    \br_H(\x, \bpi'(\x)) \succeq \br_H(\x, \bpi^\star_\bupsilon(\x)) \implies \fs_\bupsilon\br_H(\x, \bpi'(\x)) \geq \fs_\bupsilon\br_H(\x,  \bpi^\star_\bupsilon(\x)).
\end{align}
However, this is only true with equality if $\bpi'(\x) = \bpi^\star_\bupsilon(\x)$ for all $\x \in \cS$, as for any $\x \in \cS$, $\bpi^\star_{\bupsilon, H}(\x) = \argmax[ \fs_\bupsilon\br_H(\x, \ba)]\geq \fs_\bupsilon\br_H(\x, \ba')$ for any other $\ba' \in \cA$. Therefore, we have that $\bpi'_H(\x) = \bpi^\star_{\bupsilon, H}(\x)$ for each $\x \in \cS$, and that $\bV^{\bpi'}_H(\x) = \bV^{\bpi^\star_\bupsilon}_H(\x)$. This implies that $\bbP_{H}\bV^{\bpi'}_H(\x, \ba) = \bbP_{H}\bV^{\bpi^\star_\bupsilon}_H(\x, \ba)$ for all $\x \in \cS$ and $\ba \in \cA$. Now, if $\bV^{\bpi'}_{H-1}(\x) \succeq \bV^{\bpi^\star_\bupsilon}_{H-1}(\x)$, then we have that,
{
\small
\begin{align}
    &\br_{H-1}(\x, \bpi'_{H-1}(\x)) + \bbE_{\x' \sim \bbP_H(\cdot | \x, \bpi'_{H-1}(\x))}\left[ \bV^{\bpi'}_H(\x')\right] \succeq \br_{H-1}(\x, \bpi^\star_\bupsilon(\x)) + \bbP_{H}\bV^{\bpi^\star_\bupsilon}_H(\x, \bpi^\star_\bupsilon(\x)) \\
    &\implies \br_{H-1}(\x, \bpi'_{H-1}(\x)) + \bbE_{\x' \sim \bbP_H(\cdot | \x, \bpi'_{H-1}(\x))}\left[ \bV^{\bpi^\star_\bupsilon}_H(\x')\right] \succeq \br_{H-1}(\x, \bpi^\star_\bupsilon(\x)) + \bbP_{H}\bV^{\bpi^\star_\bupsilon}_H(\x, \bpi^\star_\bupsilon(\x)) \\
    &\implies \fs_\bupsilon\left(\br_{H-1}(\x, \bpi'_{H-1}(\x)) + \bbE_{\x' \sim \bbP_H(\cdot | \x, \bpi'_{H-1}(\x))}\left[ \bV^{\bpi^\star_\bupsilon}_H(\x')\right]\right) \geq \fs_\bupsilon\left(\br_{H-1}(\x, \bpi^\star_\bupsilon(\x)) + \bbP_{H}\bV^{\bpi^\star_\bupsilon}_H(\x, \bpi^\star_\bupsilon(\x))\right) \\
    &\implies \fs_\bupsilon\br_{H-1}(\x, \bpi'_{H-1}(\x)) + \bbE_{\x' \sim \bbP_H(\cdot | \x, \bpi'_{H-1}(\x))}\left[ \bV^{\bpi^\star_\bupsilon}_H(\x')\right] \geq \fs_\bupsilon\br_{H-1}(\x, \bpi^\star_\bupsilon(\x)) + \bbP_{H}\bV^{\bpi^\star_\bupsilon}_H(\x, \bpi^\star_\bupsilon(\x)).
\end{align}
}%
This is true only if $\bpi'_{H-1}(\x) = \bpi^\star_{\bupsilon, H}(\x)$ for each $\x \in \cS$, as $\bpi^\star_{\bupsilon, H}$ is the greedy policy with respect to $\fs_\bupsilon\br_{H-1}(\x, \ba) + \bbP_{H}\bV^{\bpi^\star_{\bupsilon, H}}_H(\x, \ba)$. Continuing this argument inductively for $h = H-2, H-3, ..., 1$ we obtain that $\bV^{\bpi'}_1(\x) \succeq \bV^{\bpi^\star_\bupsilon}_1(\x)$ for each $\x \in \cS$ only if $\bpi' = \bpi^\star_\bupsilon$. This is a contradiction as we assumed that $\bpi' \neq \bpi^\star_\bupsilon$, and hence $\bpi^\star_\bupsilon$ lies in $\bPi^\star$.
\end{proof}

\subsection{Proof of Proposition~\ref{prop:bayes_regret_bound}}
We first restate Proposition~\ref{prop:bayes_regret_bound} for clarity.

\noindent\textbf{Proposition~\ref{prop:bayes_regret_bound}}.
For any scalarization $\fs$ that is Lipschitz and bounded $\bUpsilon$, we have that $\fR_B(T) \leq \frac{1}{T}\fR_C(T) + o(1).$

\begin{proof}
We will follow the approach in~\cite{paria2020flexible} (Appendix B.3). Recall that $\bUpsilon$ is a bounded subset of $\bbR^M$. Now, we have that since $\fs_\bupsilon(\cdot) = \bupsilon^\top(\cdot)$, we have that $\fs_\bupsilon$ is Lipschitz with constant $M$ with respect to the $\ell_1-$norm, i.e., for any $\y \in \bbR^M$,
\begin{align}
    |\fs_\bupsilon(\y) - \fs_{\bupsilon'}(\y)| \leq M\lVert \bupsilon - \bupsilon' \rVert_1.
\end{align}
Now, consider the Wasserstein distance conditioned on the history $\cH$ between the sampling distribution $p_\bUpsilon$ on $\bUpsilon$ and the empirical distribution $\widehat p_\bUpsilon$ corresponding to $\{\bupsilon_t\}_{t=1}^T$,
\begin{align}
    W_1(p_\bUpsilon, \widehat p_\bUpsilon) = \inf_{q}\left\{ \bbE_q\lVert X-Y\rVert_1, q(X) = p_\bUpsilon, q(Y) = \widehat p_\bUpsilon\right\},
\end{align}
where $q$ is a joint distribution on the RVs $X, Y$ with marginal distributions equal to $p_\bUpsilon$ and $\widehat p_\bUpsilon$. We therefore have for some randomly drawn samples $\bupsilon_1, \bupsilon_2, ..., \bupsilon_T$ and for any arbitrary sequence of (joint) policies $\widehat\bPi_T = \{\bpi_1, ..., \bpi_T\}$, for any state $\x \in \cS$,
\begin{align}
    \frac{1}{T}\sum_{t=1}^T \max_{\x \in \cS}\left[V^{\bpi_t}_{\bupsilon_t, 1}(\x) - \bbE_{\bupsilon \in \bUpsilon}\left[\max_{\bpi \in \widehat\bPi_T} V^{\bpi}_{\bupsilon, 1}(\x)\right]\right] &\leq \frac{1}{T}\sum_{t=1}^T \max_{\x \in \cS}\left[V^{\bpi_t}_{\bupsilon_t, 1}(\x) - \bbE_{\bupsilon \in \bUpsilon}\left[\max_{\bpi \in \widehat\bPi_T} V^{\bpi}_{\bupsilon, 1}(\x)\right]\right] \\
    &\leq \bbE_{q(X, Y)}\left[\max_{\x \in \cS}\left[\max_{\bpi \in \widehat\bPi_T}  V^{\bpi}_{X, 1}(\x) - \max_{\bpi \in \widehat\bPi_T} V^{\bpi}_{Y, 1}(\x)\right]\right] \\
    &\leq \bbE_{q(X, Y)}\left[M\lVert X - Y \rVert_1\right].
\end{align}
Taking an expectation with respect to $\cH = \{\bupsilon_1, .., \bupsilon_T\}$, we have,
\begin{align}
   &\fR_B(T) - \frac{1}{T}\fR_C(T)\\  &= \bbE_{\bupsilon \in \bUpsilon} \left[\max_{\x \in \cS}\left[V^{\star}_{\bupsilon, 1}(\x) - \max_{\bpi \in \widehat\bPi_T} V^{\bpi}_{\bupsilon, 1}(\x)\right]\right] - \bbE_{\cH}\left[\frac{1}{T}\sum_{t=1}^T \max_{\x \in \cS}\left[V^{\star}_{\bupsilon_t, 1}(\x) - V^{\bpi_t}_{\bupsilon_t, 1}(\x)\right]\right] \\
   &= \bbE_{\cH}\left[\frac{1}{T}\sum_{t=1}^T \max_{\x \in \cS}\left[V^{\star}_{\bupsilon_t, 1}(\x) - \max_{\bpi \in \widehat\bPi_T} V^{\bpi}_{\bupsilon_t, 1}(\x)\right]\right] - \bbE_{\cH}\left[\frac{1}{T}\sum_{t=1}^T \max_{\x \in \cS}\left[V^{\star}_{\bupsilon_t, 1}(\x) - V^{\bpi_t}_{\bupsilon_t, 1}(\x)\right]\right] \\
   &\leq \bbE_{\cH}\left[\frac{1}{T}\sum_{t=1}^T \max_{\x \in \cS}\left[V^{\bpi_t}_{\bupsilon_t, 1}(\x) - \bbE_{\bupsilon \in \bUpsilon}\left[\max_{\bpi \in \widehat\bPi_T} V^{\bpi}_{\bupsilon, 1}(\x)\right]\right]\right]\\
   &\leq M\bbE_{q(X, Y)}\left[\lVert X - Y \rVert_1\right].
\end{align}
The penultimate inequality follows from $\max$ being a contraction mapping in bounded domains, and the final inequality follows from the previous analysis. To complete the proof, we first take an infimum over $q$ and observe that the subsequent RHS converges at a rate of $\ctO(T^{-1/M})$ under mild regulatory conditions, as shown by~\citet{canas2012learning}.
\end{proof}
% \begin{proposition}
% \label{prop:mg_weight_existence}
% For a scalarized linear MG, for any policy $\bpi$ and scalarization parameter $\bupsilon \in \bUpsilon$, there exist weights $\{\w^\bpi_{\bupsilon, h}\}_{h \in [H]}$ such that for any $(\x, \ba, h) \in \cS \times \cA \times [H]$, we have $Q^\bpi_{\bupsilon, h} = \fs_\bupsilon\left(\left\langle \bPhi(\x, \ba), \w^\bpi_{\bupsilon, h} \right\rangle \right)$.
% \end{proposition}
% \begin{proof}
% From the Bellman equation for the scalarized $Q-$values, we have,
% \begin{align}
%   Q^\bpi_{\bupsilon, h}(\x, \ba) &= \fs_\bupsilon(\br_h(\x, \ba)) + \bbP_h V^\bpi_{\bupsilon, h}(\x, \ba) = \fs_\bupsilon(\bPhi(\x, \ba)^\top \btheta_h) + \int_{\x' \in \cS} V^\bpi_{\bupsilon, h}(\x')\cdot\fs_\bupsilon(\bPhi(\x, \ba)^\top d\bmu_h(\x')) \\
%   &= \bupsilon^\top\left[\bPhi(\x, \ba)^\top \btheta_h + \bPhi(\x, \ba)^\top \left[ \int_{\x' \in \cS}V^\bpi_{\bupsilon, h}(\x')d\bmu_h(\x') \right]\right] \\
%   &= \bupsilon^\top\bPhi(\x, \ba)^\top \left[\btheta_h + \left[ \int_{\x' \in \cS}V^\bpi_{\bupsilon, h}(\x')d\bmu_h(\x') \right]\right] = \fs_\bupsilon\left(\bPhi(\x, \ba)^\top \left[\btheta_h + \left[ \int_{\x' \in \cS}V^\bpi_{\bupsilon, h}(\x')d\bmu_h(\x') \right]\right]\right).
% \end{align}
% Therefore, we have that $Q^\bpi_{\bupsilon, h}(\x, \ba) = \fs_\bupsilon\left(\left\langle \bPhi(\x, \ba), \w^\bpi_{\bupsilon, h} \right\rangle \right)$ where $\w^\bpi_{\bupsilon, h} = \btheta_h + \left[ \int_{\x' \in \cS}V^\bpi_{\bupsilon, h}(\x')d\bmu_h(\x') \right]$.
% \end{proof}

\subsection{Proof of Theorem~\ref{thm:ind_mg}}
The proof in the multiagent MDP setting is similar to that of the parallel MDP setting. There are several differences: first, in each episode, since we sample a scalarization parameter $\bupsilon_t$ from $\bUpsilon_T$, we would like to derive concentration results \textit{independent} of the scalarization parameter. We do this by utilizing \textit{vector-valued} concentration results and utilizing the monotonicity of the scalarized $Q-$value. We first present a vector-valued concentration result.
\begin{lemma}
\label{lem:mg_lsvi}
For any $m \in [M], h \in [H]$ and $t \in [T]$, let $k_t$ denote the episode after which the last global synchronization has taken place, and $\bS^h_t$ and $\bLambda^h_t$ be defined as follows.
\begin{align*}
    \bS^h_{\bupsilon, t} = \sum_{\tau=1}^{k_t}\bPhi(\x^\tau_{h}, \ba^\tau_{h}) \left[\bv^t_{\bupsilon, h+1}(\x^\tau_{h+1})- (\bbP_h\bv^t_{\bupsilon, h+1})(\x^\tau_{h}, \ba^\tau_{h})\right],\ \bLambda^h_t = \lambda\bI_d + (\bPhi^{k_t}_h)^\top(\bPhi^{k_t}_h).
\end{align*} 
Where $\bv^t_{\bupsilon, h+1}(\x) = {\bm 1}_M \cdot V^t_{\bupsilon, h+1}(\x) \ \forall\ \x \in \cS$, ${\bm 1}_M$ denotes the all-ones vector in $\bbR^M$, and $C_\beta$ is the constant such that $\beta^t_h = C_\beta\cdot dH\sqrt{\log(TMH)}$. Then, there exists a constant $C$ such that with probability at least $1-\delta$,
\begin{align*}
    \sup_{\bupsilon \in \bUpsilon}\left\lVert \bS^t_{\bupsilon, h} \right\rVert_{(\bLambda^t_{ h})^{-1}} &\leq C\cdot dH\sqrt{2\log\left(\frac{(C_\beta + 2)dMTH}{\delta'}\right)}.
\end{align*}
\end{lemma}

\begin{proof}The proof is done in two steps. The first step is to bound the deviations in $\bS$ for any fixed function $V$ by a martingale concentration. The second step is to bound the resulting concentration over all functions $V$ by a covering argument. Finally, we select appropriate constants to provide the form of the result required.

\textbf{\underline{Step 1}}.
Recall that $\bS^t_{\bupsilon, h} = \sum_{\tau=1}^{k_t} \bPhi(\bz^\tau_h)[V^t_{\bupsilon, h+1}(\x^\tau_{h+1}) -(\bbP_h V^t_{\bupsilon, h+1})(\bz^\tau_h)]$, where $\bv^t_{\bupsilon, h+1}$ is the vector with each entry being $V^t_{\bupsilon, h+1}$. We have that $V^t_{\bupsilon, h+1}(\x^\tau_{h+1}) -(\bbP_h V^t_{\bupsilon, h+1})(\bz^\tau_h) = \bv^t_{\bupsilon, h+1} - \bbP_h \bv^t_{\bupsilon, h+1}$. Consider the following distance metric $\text{dist}_\bUpsilon$,
\begin{align}
    \text{dist}_\bUpsilon(\bv, \bv') &= \sup_{\x \in \cS, \bupsilon \in \bUpsilon} \left\lVert \bv(\x) - \bv(\x')\right\rVert_1
    .
\end{align}Let $\cV_{\bUpsilon}$ be the family of all vector-valued UCB value functions that can be produced by the algorithm, and now let $\cN_\epsilon$ be an $\epsilon-$covering of $\cV_{\bUpsilon}$ under $\text{dist}_\bUpsilon$, i.e., for every $\bv \in \cV_\bUpsilon$, there exists $\bv' \in \cN_\epsilon$ such that $\text{dist}_\bUpsilon(\bv, \bv') \leq \epsilon$. Now, here again, we adopt a similar strategy as the independent case. To bound the RHS, we decompose $\bS^t_{\bupsilon, h}$ in terms of the covering described earlier. We know that since $\cN_\epsilon$ is an $\epsilon-$covering of $\cV_{\bUpsilon}$, there exists a $\bv' \in \cN_\epsilon$ and $\bDelta = \bv^t_{\bupsilon, h+1} - \bv'$ such that,
\begin{align}
    \bS^t_{\bupsilon, h} &= \sum_{\tau=1}^{k_t} \bPhi(\bz^\tau_h)\left[\bv'(\x^\tau_{h+1}) - \bbP_h \bv'(\bz^\tau_h)\right] + \sum_{\tau=1}^{k_t} \bPhi(\bz^\tau_h)\left[\bDelta(\x^\tau_{h+1}) - \bbP_h \bDelta(\bz^\tau_h)\right].
\end{align}
Now, observe that by the definition of the covering, we have that $\lVert \bDelta \rVert_1 \leq \epsilon$. Therefore, we have that $\left\lVert \bDelta(\x) \right\rVert_{(\bLambda^t_h)^{-1}} \leq \epsilon/\sqrt{\lambda}$, and $\left\lVert \bbP_h\bDelta(\bz) \right\rVert_{(\bLambda^t_h)^{-1}} \leq \epsilon/\sqrt{\lambda}$ for all $\bz \in \cZ, \x \in \cS, h \in [H]$. Therefore, since $\left\lVert \bPhi(\bz) \right\rVert_2 \leq \sqrt{M}$,
\begin{align}
    \left\lVert  \bS^t_{\bupsilon, h} \right\rVert^2_{(\bLambda^t_h)^{-1}} &\leq 2\left\lVert \sum_{\tau=1}^{k_t} \bPhi(\bz^\tau_h)\left[\bv'(\x^\tau_{h+1}) - \bbP_h \bv'(\bz^\tau_h)\right] \right\rVert_{(\bLambda^t_h)^{-1}} + \frac{8Mt^2\epsilon^2}{\lambda}.
\end{align}
Consider the substitution $\bepsilon^t_{\tau, h} = \bv'(\x^\tau_{h+1}) - \bbP_h \bv'(\bz^\tau_{h})$. To bound the first term on the RHS, we consider the filtration $\left\{\cF_\tau\right\}_{\tau=0}^\infty$ where $\cF_0$ is empty, and $\cF_\tau = \sigma\left(\left\{\bigcup \left(\x^i_{h+1}, \bphi(\bz^i_h)\right)\right\}_{i \leq \tau}\right)$, and $\sigma$ denotes the $\sigma-$algebra generated by a finite set. Then, we have that,
\begin{align*}
    &\left\lVert \sum_{\tau=1}^{k_t} \bPhi(\bz^\tau_h)\left[\bv'(\x^\tau_{h+1}) - \bbP_h \bv'(\bz^\tau_h)\right] \right\rVert_{(\bLambda^t_h)^{-1}} \\
    &= \left\lVert \sum_{\tau=1}^{k_t} \bPhi(\bz^\tau_h)\left[\bv'(\x^\tau_{h+1}) - \bbE\left[ \bv'(\x^\tau_{h+1}) | \cF_{\tau-1}\right]\right] \right\rVert_{(\bLambda^t_h)^{-1}} = \left\lVert \sum_{\tau=1}^{k_t} \bPhi(\bz^\tau_h) \bepsilon^t_{\tau, h} \right\rVert_{(\bLambda^t_h)^{-1}}.
\end{align*}
Note that for each $\bepsilon^t_{\tau, h}$, each entry is bounded by $2H$, and therefore we have that the vector $\bepsilon^t_{\tau, h}$ is $H-$sub-Gaussian. Then, applying Lemma~\ref{lem:self_normalized_martingale_multi_task}, we have that,
\begin{align}
   \left\lVert \sum_{\tau=1}^{k_t} \bPhi(\bz^\tau_h) \bepsilon^t_{\tau, h} \right\rVert_{(\bLambda^t_h)^{-1}} &\leq H^2\log\left(\frac{\det\left(\bLambda^{t}_h\right)}{\det\left(\lambda\bI_d\right)\delta^2}\right)\leq H^2\log\left(\frac{\det\left(\bar\bLambda^{t}_h\right)}{\det\left(\lambda\bI_d\right)\delta^2}\right).
\end{align}
Replacing this result for each $\bv \in \cN_\epsilon$, we have by a union bound over each $t \in [T], h \in [H]$, we have with probability at least $1-\delta$, simultaneously for each $t \in [T], h \in [H]$,
\begin{align}
\sup_{\bupsilon_t \in \bUpsilon, \bv \in \cV_\bUpsilon} \left\lVert \bS^t_{\bupsilon, h} \right\rVert_{(\bLambda^t_h)^{-1}} &\leq  2H\sqrt{\log\left(\frac{\det\left(\bar\bLambda^{t}_h\right)}{\det\left(\lambda\bI_d\right)}\right) + \log\left(\frac{HT|\cN_\epsilon|}{\delta}\right) + \frac{2Mt^2\epsilon^2}{H^2\lambda}} \\
&\leq  2H\sqrt{d\log\left(\frac{Mt+\lambda}{\lambda}\right) + \log\left(\frac{|\cN_\epsilon|}{\delta}\right) + \log(HT)+ \frac{2Mt^2\epsilon^2}{H^2\lambda}}.
\end{align}
The last step follows once again by first noticing that $\lVert \bPhi(\cdot) \rVert \leq \sqrt{M}$ and then applying an AM-GM inequality, and then using the determinant-trace inequality.

\textbf{\underline{Step 2}}.
Here $\cN_\epsilon$ is an $\epsilon-$covering of the function class $\cV_{\bUpsilon}$ for any $h \in [H], m \in [M]$ or $t \in [T]$ under the distance function $\text{dist}_\bUpsilon(\bv, \bv') = \sup_{\x \in \cS, \bupsilon \in \bUpsilon} \left\lVert \bv(\x) - \bv(\x')\right\rVert_1$. To bound this quantity by the appropriate covering number, we first observe that for any $V \in \cV_{\bUpsilon}$, we have that the policy weights are bounded as $2HM\sqrt{dT/\lambda}$ (Lemma~\ref{lem:bound_mg_algo_weights}). Therefore, by Lemma~\ref{lem:covering_mg} we have for any constant $B$ such that $\beta^t_{h} \leq B$,
\begin{align}
    \log \left(\cN_\varepsilon\right) \leq d\cdot\log\left(1 + \frac{8HM^3}{\varepsilon}\sqrt{\frac{dT}{\lambda}}\right) + d^2\log\left(1 + \frac{8Md^{1/2}B^2}{\lambda\varepsilon^2}\right).
\end{align}
Recall that we select the hyperparameters $\lambda = 1$ and $\beta = \cO(dH\sqrt{\log(TMH)}$, and to balance the terms in $\bar\beta^t_h$ we select $\epsilon = \epsilon^\star = dH/\sqrt{MT^2}$. Finally, we obtain that for some absolute constant $C_\beta$, by replacing the above values,
\begin{align}
    \log \left(\cN_\varepsilon\right) \leq d\cdot\log\left(1 + \frac{8M^{7/2}T^{3/2}}{d^{1/2}}\right) + d^2\log\left(1 + 8C_\beta d^{1/2}MT^2\log(TMH)\right).
\end{align}
Therefore, for some absolute constant $C'$ independent of $M, T, H, d$ and $C_\beta$, we have,
\begin{align}
    \log |\cN_\varepsilon| \leq C'd^2 \log\left(C_\beta \cdot dMT\log(TMH)\right).
\end{align}
Replacing this result in the result from Step 1, we have that with probability at least $1-\delta'/2$ for all $t \in [T], h \in [H]$ simultaneously,
\begin{align*}
    &\left\lVert \bS^t_{\bupsilon, h} \right\rVert_{(\bLambda^t_{h})^{-1}} \\
    &\leq 2H\sqrt{(d+2)\log\frac{MT+\lambda}{\lambda} + 2\log\left(\frac{1}{\delta'}\right) +  C'd^2 \log\left(C_\beta \cdot dMT\log(TMH)\right) + 2 + 4\log(TH)}.
\end{align*}
This implies that there exists an absolute constant $C$ independent of $M, T, H, d$ and $C_\beta$, such that, with probability at least $1-\delta'/2$ for all $t \in [T], h \in [H], \bupsilon \in \bUpsilon$ simultaneously,
\begin{align}
    \left\lVert \bS^t_{\bupsilon, h} \right\rVert_{(\bLambda^t_{ h})^{-1}} &\leq C\cdot dH\sqrt{2\log\left(\frac{(C_\beta + 2)dMTH}{\delta'}\right)}.
\end{align}
\end{proof}

Next, we present the key result for cooperative value iteration, which demonstrates that for any agent the estimated $Q-$values have bounded error for any policy $\pi$. This result is an extension of Lemma B.4 of~\cite{jin2020provably} on to the multiagent MDP setting. 
\begin{lemma}\label{lem:mg_homo_weight_delta}
There exists an absolute constant $c_\beta$ such that for $\beta^t_{m, h} = c_\beta\cdot dH \sqrt{\log(2dMHT/\delta')}$ for any policy $\pi$, there exists a constant $C'_\beta$ such that for each $x \in \cS, a \in \cA$ we have for all $ m \in \cM, t \in[T], h \in [H]$ simultaneously, with probability at least $1-\delta'/2$,
\begin{align*}
     \left| \langle \bphi(x, a), \w^t_{m, h} - \w^\pi_h \rangle \right| \leq  \bbP_h(V^t_{m, h+1} - V^\pi_{m, h+1})(x, a) + C'_\beta\cdot dH\cdot\lVert \bphi(z) \rVert_{(\bLambda^t_{m, h})^{-1}}\cdot \sqrt{2\log\left(\frac{dMTH}{\delta'}\right)}.
\end{align*}
\end{lemma}
\begin{proof}
By the Bellman equation and the assumption of the linear MMDP (Definition~\ref{def:linear_mg}), we have that for any policy $\bpi$, and $\bupsilon \in \bUpsilon$, there exist weights $\w^\pi_{\bupsilon, h}$ such that, for all $\bz \in \cZ = \cS \times \cA$, 
\begin{align}
    \bupsilon^\top\bPhi(\bz)^\top\w^\pi_{\bupsilon, h} = \bupsilon^\top \br_{h}(\bz) + \bbP_h V^\pi_{\bupsilon, h+1}(\bz) = \bupsilon^\top \left(\br_{h}(\bz) + {\bm 1}_M\cdot\bbP_h V^\pi_{\bupsilon, h+1}(\bz)\right).
\end{align}
We have,
\begin{align}
    \w^t_{\bupsilon, h} - \w^\pi_{\bupsilon, h} &= (\bLambda_{h}^t)^{-1}\sum_{\tau=1}^{k_t}\left[\bPhi_\tau(\x_\tau, \ba_\tau)[\br_h(\x_\tau, \ba_\tau)+{\bm 1}_M\cdot V^t_{\bupsilon, h+1}(\x_\tau)]\right] - \w^\pi_{\bupsilon, h} \\
    &= (\bLambda_{h}^t)^{-1}\left\{ -\lambda\w^\pi_{\bupsilon, h} + \sum_{\tau=1}^{k_t}\left[\bPhi_\tau(\x_\tau, \ba_\tau)[{\bm 1}_M\cdot(V^t_{\bupsilon, h+1}(\x'_\tau) -\bbP_h V^\pi_{\bupsilon, h+1}(\x_\tau, \ba_\tau))]\right]\right\}. 
\end{align}
\begin{multline}
    \w^t_{\bupsilon, h} - \w^\pi_{\bupsilon, h} =  \underbrace{-\lambda(\bLambda_{h}^t)^{-1}\w^\pi_{\bupsilon, h}}_{\bv_1} \\ + \underbrace{(\bLambda_{h}^t)^{-1}\left\{\sum_{\tau=1}^{k_t}\left[\bPhi_\tau(\x_\tau, \ba_\tau)[{\bm 1}_M\cdot(V^t_{\bupsilon, h+1}(\x'_\tau) -\bbP_h V^t_{\bupsilon, h+1}(\x_\tau, \ba_\tau))]\right] \right\}}_{\bv_2} \\ + \underbrace{(\bLambda_{m, h}^t)^{-1}\left\{\sum_{\tau=1}^{k_t}\left[\bPhi_\tau(\x_\tau, \ba_\tau)[{\bm 1}_M\cdot(\bbP_h V^t_{\bupsilon, h+1} -\bbP_h V^\pi_{\bupsilon, h+1})(\x_\tau, \ba_\tau)]\right] \right\}}_{\bv_3}
\end{multline}
Now, we know that for any $\bz \in \cZ$ for any policy $\pi$,
\begin{align*}
    &\left\lVert \langle \bPhi(\bz), \bv_1 \rangle \right\rVert_2 \\
    &\leq \lambda\lVert \langle \bPhi(\bz), (\bLambda_{h}^t)^{-1}\w^\pi_{\bupsilon, h} \rangle \rVert_2 \leq \lambda\cdot \lVert \w^\pi_{\bupsilon, h} \rVert \lVert \bPhi(\bz) \rVert_{(\bLambda^t_{h})^{-1}} \leq 2HM\lambda\sqrt{d}\cdot\lVert \bPhi(\bz) \rVert_{(\bLambda^t_{h})^{-1}}
\end{align*}
Here the last inequality follows from Lemma~\ref{lem:bound_mg_policy_weight}. For the second term, we have by Lemma~\ref{lem:mg_lsvi} that there exists an absolute constant $C_\beta$, independent of $M, T, H, d$  such that, with probability at least $1-\delta'/2$ for all $t \in [T], h \in [H], \bupsilon \in \bUpsilon$ simultaneously,
\begin{align}
    \left\lVert \langle \bPhi(\bz), \bv_2 \rangle \right\rVert_2 \leq \left\lVert \bPhi(\bz)\right\rVert_{(\bLambda^t_{h})^{-1}}\cdot C_\beta\cdot dH\cdot\sqrt{2\log\left(\frac{dMTH}{\delta'}\right)}.
\end{align}
For the last term, note that,
\begin{align}
    &\langle \bPhi(\x, \ba), \bv_3 \rangle  \\
    &= \left\langle \bPhi(\bz), (\bLambda_{ h}^t)^{-1}\left\{\sum_{\tau=1}^{k_t}\bPhi(\x_\tau, \ba_\tau)[{\bm 1}_M\cdot(\bbP_h V^t_{\bupsilon, h+1} -\bbP_h V^\pi_{\bupsilon, h+1})(\x_\tau, \ba_\tau)] \right\} \right\rangle \\
    &= \left\langle \bPhi(\bz), (\bLambda_{ h}^t)^{-1}\left\{\sum_{\tau=1}^{k_t}\bPhi(\x_\tau, \ba_\tau)\bPhi(\x_\tau, \ba_\tau)^\top\int (V^t_{\bupsilon, h+1} - V^\pi_{\bupsilon, h+1})(\x')d\bmu_h(\x')\right\} \right\rangle \\
    &= \left\langle \bPhi(\bz), (\bLambda_{ h}^t)^{-1}\left\{\sum_{\tau=1}^{k_t}\bPhi(\x_\tau, \ba_\tau)\bPhi(\x_\tau, \ba_\tau)^\top\int (V^t_{\bupsilon, h+1} - V^\pi_{\bupsilon, h+1})(\x')d\bmu_h(\x') \right\} \right\rangle \\
    &= \left\langle \bPhi(\bz),\int (V^t_{\bupsilon, h+1} - V^\pi_{\bupsilon, h+1})(\x')d\bmu_h(\x')  \right\rangle -\lambda\left\langle \bPhi(\bz), (\bLambda_{ h}^t)^{-1}\int (V^t_{\bupsilon, h+1} - V^\pi_{\bupsilon, h+1})(\x')d\bmu_h(\x')\right\rangle\\
    &= \int (V^t_{\bupsilon, h+1} - V^\pi_{\bupsilon, h+1})(\x')\left\langle \bPhi(\bz),\bmu_h(\x')  \right\rangle -\lambda\left\langle \bPhi(\bz), (\bLambda_{ h}^t)^{-1}\int (V^t_{\bupsilon, h+1} - V^\pi_{\bupsilon, h+1})(\x')d\bmu_h(\x')\right\rangle\\
    &= {\bm 1}_M\cdot\left(\bbP_h(V^t_{\bupsilon, h+1} - V^\pi_{\bupsilon, h+1})(\x, \ba)\right) -\lambda \left\langle \bPhi(\bz), (\bLambda_{ h}^t)^{-1}\int (V^t_{\bupsilon, h+1} - V^\pi_{\bupsilon, h+1})(\x')d\bmu_h(\x')\right\rangle\\
    &\leq {\bm 1}_M\cdot\left(\bbP_h(V^t_{\bupsilon, h+1} - V^\pi_{\bupsilon, h+1})(\x, \ba) + 2H\sqrt{d\lambda}\lVert \bPhi(\x, \ba) \rVert_{(\bLambda^t_{h})^{-1}}\right).
\end{align}
Putting it all together, we have that since $\langle \bPhi(\x, \ba), \w^t_{\bupsilon, h} - \w^\pi_{\bupsilon, h} \rangle = \langle \bPhi(\x, \ba), \bv_1 + \bv_2 + \bv_3 \rangle$, there exists an absolute constant $C_\beta$ independent of $M, T, H, d$, such that, with probability at least $1-\delta'/2$ for all $t \in [T], h \in [H], \bupsilon \in \bUpsilon$ simultaneously,
\begin{multline}
    \left| \langle \bupsilon^\top\bPhi(\x, \ba), \w^t_{\bupsilon, h} - \w^\pi_{\bupsilon, h} \rangle \right| \leq  \bupsilon^\top{\bm 1}_M\cdot\left(\bbP_h(V^t_{\bupsilon, h+1} - V^\pi_{\bupsilon, h+1})(\x, \ba)\right) \\+ \lVert \bPhi(\x, \ba) \rVert_{(\bLambda^t_{h})^{-1}}\lVert \bupsilon \rVert_2\left(2H\sqrt{d\lambda} + C_\beta\cdot dH\cdot\sqrt{2\log\left(\frac{dMTH}{\delta'}\right)} + 2HM\lambda\sqrt{d}\right)
\end{multline}
Since $\lambda \leq 1$ and $\lVert \bupsilon \rVert_2 \leq 1$, there exists a constant $C'_\beta$ that we have the following for any $(x, a) \in \cS \times \cA$ with probability $1-\delta'/2$ simultaneously for all $h \in [H], \bupsilon \in \bUpsilon, t \in [T]$,
\begin{align*}
    &\left| \langle \bupsilon^\top\bPhi(\x, \ba), \w^t_{\bupsilon, h} - \w^\pi_{\bupsilon, h} \rangle \right| \\
    &\leq  \bbP_h(V^t_{\bupsilon, h+1} - V^\pi_{\bupsilon, h+1})(\x, \ba) + C'_\beta\cdot dMH\cdot\lVert \Phi(\bz)\rVert_{(\bLambda^t_{h})^{-1}}\cdot \sqrt{2\log\left(\frac{dMTH}{\delta'}\right)}.
\end{align*}
\end{proof}

\begin{lemma}[UCB in the Multiagent Setting]
With probability at least $1-\delta'/2$, we have that for all $(x, a, h, t, \bupsilon) \in \cS \times \cA \times [H] \times [T] \times \bUpsilon$, $Q^t_{\upsilon, h}(x, a) \geq Q^\star_{\upsilon, h}(x, a)$.
\label{lem:mg_ucb}
\end{lemma}
\begin{proof}
We prove this result by induction. First, for the last step $H$, note that the statement holds as $Q^t_{\bupsilon, H}(\x, \ba) \geq Q^\star_{\bupsilon, H}(\x, \ba)$ for all $\bupsilon$. Recall that the value function at step $H+1$ is zero. Therefore, by Lemma~\ref{lem:mg_homo_weight_delta}, we have that, for any $\bupsilon \in \bUpsilon$,
\begin{align}
    \left| \langle \bupsilon^\top\bPhi(\x, \ba), \w^t_{\bupsilon, H}\rangle - Q^\star_{\bupsilon, H}(\x, \ba) \right| \leq C'_\beta\cdot dH\cdot\lVert \Phi(\bz)\rVert_{(\bLambda^t_{h})^{-1}}\cdot \sqrt{2\log\left(\frac{dMTH}{\delta'}\right)}.
\end{align}
We have $Q^\star_{\bupsilon, H}(\x, \ba) \leq \langle \bupsilon^\top\bPhi(\x, \ba), \w^t_{\bupsilon, H}\rangle + C'_\beta\cdot dH\cdot\lVert \Phi(\bz)\rVert_{(\bLambda^t_{h})^{-1}}\cdot \sqrt{2\log\left(\frac{dMTH}{\delta'}\right)} = Q^t_{\bupsilon, H}$. Now, for the inductive case, we have by Lemma~\ref{lem:mg_homo_weight_delta} for any $h \in [H], \bupsilon \in \bUpsilon$,
\begin{align*}
    &\left| \langle \bupsilon^\top\bPhi(\x, \ba), \w^t_{\bupsilon, h} - \w^\star_{\bupsilon, h} \rangle - \left(\bbP_h V^\star_{\bupsilon, h+1}(\x, \ba) - \bbP_h V^t_{\bupsilon, h+1}(\x, \ba)\right)\right|\\ &\leq   C'_\beta\cdot dH\cdot\lVert \Phi(\bz)\rVert_{(\bLambda^t_{h})^{-1}}\cdot \sqrt{2\log\left(\frac{dMTH}{\delta'}\right)}.
\end{align*}
By the inductive assumption we have $Q^t_{\bupsilon, h+1}(\x, \ba) \geq Q^\star_{\bupsilon, h+1}(\x, \ba)$ implying $\bbP_h V^\star_{\bupsilon, h+1}(\x, \ba) - \bbP_h V^t_{\bupsilon, h+1}(\x, \ba) \geq 0$. Substituting the appropriate Q value formulations we have,
\begin{align}
   Q^\star_{\bupsilon, h} \leq  \langle \bupsilon^\top\bPhi(\x, \ba), \w^t_{\bupsilon, h}\rangle + C'_\beta\cdot dH\cdot\lVert \Phi(\bz)\rVert_{(\bLambda^t_{h})^{-1}}\cdot \sqrt{2\log\left(\frac{dMTH}{\delta'}\right)} = Q^t_{\bupsilon, h}(\x, \ba).
\end{align}
\end{proof}
\begin{lemma}[Recursive Relation in Multiagent MDP Settings]
For any $\bupsilon \in \bUpsilon$, let $\delta^t_{\bupsilon, h} = V^t_{\bupsilon, h}(\x^t_h) - V^{\pi_t}_{\bupsilon, h}(\x^t_{h})$, and $\xi^t_{\bupsilon, h+1} = \bbE\left[\delta^t_{\bupsilon, h} | \x^t_h, \ba^t_h\right] - \delta^t_{\bupsilon, h}$. Then, with probability at least $1-\alpha$, for all $(t, h) \in [T] \times [H]$ simultaneously,
\begin{align}
    \delta^t_{\bupsilon, h} \leq \delta^t_{\bupsilon, h+1} + \xi^t_{\bupsilon, h+1} + 2\left\lVert \bPhi(\x^t_h, \ba^t_h)\right\rVert_{(\bLambda^t_{h})^{-1}}\cdot C'_\beta\cdot dH\cdot\sqrt{2\log\left(\frac{dMTH}{\alpha}\right)}.
\end{align}
\label{lem:mg_recursion}
\end{lemma}
\begin{proof}
By Lemma~\ref{lem:mg_homo_weight_delta}, we have that for any $(x, a, h, \bupsilon, t) \in \cS \times \cA \times [H] \times \bUpsilon \times [T]$ with probability at least $1-\alpha/2$,
\begin{align*}
    &Q^t_{\bupsilon, h}(\x, \ba) - Q^{\pi_t}_{\bupsilon, h}(\x, \ba) \\ &\leq \bbP_h(V^t_{\bupsilon, h+1} - V^{\pi_t}_{\bupsilon, h})(\x, \ba) + 2\left\lVert \bphi(\x, \ba)\right\rVert_{(\bLambda^t_{h})^{-1}}\cdot C_\beta\cdot dH\cdot\sqrt{2\log\left(\frac{dMTH}{\alpha}\right)}.
\end{align*}
Replacing the definition of $\delta^t_{\bupsilon, h}$ and $V^{\pi_t}_{\bupsilon, h}$ finishes the proof.
\end{proof}
\begin{lemma}
Let $\bupsilon_1, \bupsilon_2, ..., \bupsilon_T$ be $T$ i.i.d. samples from $\bUpsilon$. For $ \xi^t_{\bupsilon_t, h}$ as defined earlier and any $\delta \in (0, 1)$, we have with probability at least $1-\delta/2$,
\begin{align}
\sum_{t=1}^T \sum_{h=1}^H \xi^t_{\bupsilon_t, h}  \leq \sqrt{2H^3T\log\left(\frac{2}{\alpha}\right)}.
\end{align}
\label{lem:mg_martingale}
\end{lemma}
\begin{proof}
The proof is identical to Lemma~\ref{lem:parallel_homo_martingale}, since $|\xi^t_{\bupsilon_t, h}| \leq H$ regardless of $\bupsilon_t$, which allows us to apply the Martingale concentration with the same analysis.
\end{proof}

We are now ready to prove Theorem~\ref{thm:ind_mg}. We first restate the Theorem for completeness.

\noindent\textbf{Theorem~\ref{thm:ind_mg}}.
\textbf{\texttt{CoopLSVI}} when run on a multiagent MDP with $M$ agents and communication threshold $S$, $\beta_t = \cO(dH\sqrt{\log (tMH)})$ and $\lambda = 1 - \frac{1}{MTH}$ obtains the following regret after $T$ episodes, with probability at least $1-\alpha$,
\begin{equation*}
\fR_C(T)=\widetilde\cO\left(d^{\frac{3}{2}}H^2\sqrt{ST\log\left(\frac{1}{\alpha}\right)}\right).
\end{equation*}
\begin{proof}
    We have by the definition of cumulative regret:
    \begin{align}
        \fR_C(T) &= \sum_{t=1}^T \bbE_{\bupsilon_t \sim \bUpsilon}\left[ \max_{\x^t_1 \in \cS}\left[V^\star_{\bupsilon_t, 1}(\x^t_{1}) - V^{\pi_t}_{\bupsilon_t, 1}(\x^t_{1})\right]\right] =  \bbE_{\bupsilon_t \sim \bUpsilon}\left[\sum_{t=1}^T \max_{\x^t_1 \in \cS}\left[V^\star_{\bupsilon_t, 1}(\x^t_{1}) - V^{\pi_t}_{\bupsilon_t, 1}(\x^t_{1})\right]\right].
    \end{align}
    Our analysis focuses only on the term inside the expectation, which we will bound via terms that are independent of $\bupsilon_1, ..., \bupsilon_T$, bounding $\fR_C$. We have,
    \begin{align*}
    &\sum_{t=1}^T \max_{\x^t_1 \in \cS}\left[V^\star_{\bupsilon_t, 1}(\x^t_{1}) - V^{\pi_t}_{\bupsilon_t, 1}(\x^t_{1})\right] \\
    &\leq \sum_{t=1}^T \max_{\x^t_1 \in \cS} \delta^t_{\bupsilon_t, 1} \\
        &\leq \sum_{t=1}^T\left[\max_{\x^t_1 \in \cS} \sum_{h=1}^H \xi^t_{\bupsilon_t, h} + 2C'_\beta\cdot dH\cdot\sqrt{2\log\left(\frac{dMTH}{\alpha}\right)} \left(\sum_{t=1}^T\sum_{h=1}^H \left\lVert \bPhi(\x^t_h, \ba^t_h)\right\rVert_{(\bLambda^t_{h})^{-1}} \right)\right].
    \end{align*}
Where the last inequality holds with probability at least $1-\alpha/2$, via Lemma~\ref{lem:mg_recursion} and Lemma~\ref{lem:mg_ucb}. Next, we can bound the first term via Lemma~\ref{lem:mg_martingale}. We have with probability at least $1-\alpha$, for some absolute constant $C'_\beta$,
\begin{align*}
    &\sum_{t=1}^T \max_{\x^t_1 \in \cS}\left[V^\star_{\bupsilon_t, 1}(\x^t_{1}) - V^{\pi_t}_{\bupsilon_t, 1}(\x^t_{1})\right]  \\ &\leq \sqrt{2H^3T\log\left(\frac{2}{\alpha}\right)} + 2C'_\beta\cdot dH\cdot\sqrt{2\log\left(\frac{dMTH}{\alpha}\right)} \left(\sum_{t=1}^T\sum_{h=1}^H \left\lVert \bPhi(\x^t_h, \ba^t_h)\right\rVert_{(\bLambda^t_{h})^{-1}} \right).
\end{align*}
Finally, to bound the summation, we can use the technique in Theorem 4 of~\citet{abbasi2011improved}. Assume that the last time synchronization of rewards occured was at instant $k_T$. We therefore have, by Lemma 12 of~\citet{abbasi2011improved}, for any $h \in [H]$
\begin{align}
    \sum_{t=1}^T\left\lVert \bPhi(\x^t_h, \ba^t_h)\right\rVert_{(\bLambda^t_{h})^{-1}} \leq \frac{\det(\bar\bLambda^T_{h})}{\det(\bLambda^T_{h})}\sum_{t=1}^T\left\lVert \bPhi(\x^t_h, \ba^t_h)\right\rVert_{(\bar\bLambda^t_{h})^{-1}} \leq \sqrt{S}\sum_{t=1}^T\left\lVert \bPhi(\x^t_h, \ba^t_h)\right\rVert_{(\bar\bLambda^T_{h})^{-1}} 
\end{align}
Here $\bar\bLambda^T_h = \sum_{t=1}^T \bPhi(\x^t_h, \ba^t_h)\bPhi(\x^t_h, \ba^t_h)^\top$ and the last inequality follows from the algorithms' synchronization condition. Replacing this result, we have that,
\begin{align}
    \sum_{t=1}^T\sum_{h=1}^H \left\lVert \bPhi(\x^t_h, \ba^t_h)\right\rVert_{(\bLambda^t_{h})^{-1}} &\leq 2\sum_{h=1}^H\left(\sqrt{S}\sum_{t=1}^T\left\lVert \bPhi(\x^t_h, \ba^t_h)\right\rVert_{(\bar\bLambda^T_{h})^{-1}} \right)
    &\leq 2H\sqrt{ST\cdot d\log\frac{MT+\lambda}{\lambda}}.
\end{align}
Where the last inequality is an application of Lemma~\ref{lem:variance_sum} and using the fact that $\lVert \bPhi(\cdot) \rVert_2 \leq \sqrt{M}$. Replacing this result, we have that with probability at least $1-\alpha$,
\begin{align*}
    &\sum_{t=1}^T \max_{\x^t_1 \in \cS}\left[V^\star_{\bupsilon_t, 1}(\x^t_{1}) - V^{\pi_t}_{\bupsilon_t, 1}(\x^t_{1})\right] \\ &\leq \sqrt{2H^3T\log\left(\frac{2}{\alpha}\right)} + 2C'_\beta\cdot dH^2\cdot\sqrt{2ST\log\left(\frac{dMTH}{\alpha}\right)\cdot d\log\frac{MT+\lambda}{\lambda}}.
\end{align*}
Taking expectation of the RHS over $\bupsilon_1, ..., \bupsilon_T$ gives us the final result (the $\ctO$ notation hides polylogarithmic factors).
\end{proof}
\subsection{Proof of Lemma~\ref{lem:communication_complexity_mg}}
\begin{proof}
Let the total rounds of communication triggered by the threshold condition in any step $h \in [H]$ be given by $n_h$. Then, we have, by the communication criterion,
\begin{align}
    S^{n_h}< \frac{\det\left(\bLambda^T_h\right)}{\det\left(\bLambda^0_h\right)} \leq (1 + MT/d)^d.
\end{align}
Where the last inequality follows from Lemma~\ref{lem:variance_sum} and the fact that $\lVert \bPhi \rVert \leq \sqrt{M}$. This gives us that $n_h \leq d\log_S\left(1+MT/d)\right) + 1$. Furthermore, by noticing that $n \leq \sum_{h=1}^H n_h$, we have the final result.
\end{proof}

\newpage
\section{Weight Norm Bounds}
\begin{lemma}[Bound on Weights of Homogenous Value Functions, Lemma B.1 of~\citet{jin2020provably}]
\label{lem:bound_homo_policy_weight}
Under the linear MDP Assumption (Definition~\ref{def:linear_mdp}), for any fixed policy $\pi$, let $\left\{\w^\pi_h\right\}_{h \in [H]}$ be the weights such that $Q^\pi_h(x, a) = \langle \bphi(x, a), \w^\pi_h\rangle$ for all $(x, a, h) \in \cS \times \cA \times [H]$ and $m \in \cM$. Then, we have,
\begin{align*}
    \lVert \w^\pi_h \rVert_2 \leq 2H\sqrt{d}.
\end{align*}
\end{lemma}

\begin{lemma}[Linearity of weights in MMDP]
\label{lem:bound_mg_policy_weight}
Under the linear MMDP Assumption (Definition~\ref{def:linear_mg}), for any policy $\pi$ and $\bupsilon \in \bUpsilon$, there exists weights $\{\w^\pi_{\bupsilon, h}\}_{h \in [H]}$ such that $Q^\pi_{\bupsilon, h}(\x, \ba) = \bupsilon^\top \bPhi(\x, \ba)^\top \w^\pi_{\bupsilon, h}$ for all $(\x, \ba, h) \in \cS \times \cA \times [H]$, where $\lVert \w^\pi_{\bupsilon, h} \rVert_2 \leq 2H\sqrt{d}$.
\end{lemma}
\begin{proof}
By the Bellman equation and Proposition~\ref{prop:mg_bellman_optimal}, we have that for any MDP corresponding to the scalarization parameter $\bupsilon \in \bUpsilon$ and any policy $\bpi$, state $\x \in \cS$, joint action $\ba \in \cA$,
\begin{align}
    Q^\pi_{\bupsilon, h}(\x, \ba) &= \bupsilon^\top\br_h(\x, \ba) + \bbP_hV^\pi_{\bupsilon, h+1}(\x, \ba) \\
    &= \bupsilon^\top\left(\br_h(\x, \ba) + {\bm 1}_M\cdot\bbP_hV^\pi_{\bupsilon, h+1}(\x, \ba)\right) \\
    &= \bupsilon^\top\left(\bPhi(\x, \ba)^\top \begin{bmatrix}
    \btheta_h \\
    {\bm 0}_{d_2}
    \end{bmatrix} + \int V^\pi_{\bupsilon, h+1}(\x')\bPhi(\x, \ba)^\top \begin{bmatrix}
    {\bm 0}_{d_1} \\
    d\bmu_h(\x')
    \end{bmatrix}d\x'\right) \\
    &= \bupsilon^\top\bPhi(\x, \ba)^\top\w^\pi_{\bupsilon, h}.
\end{align}
Where $\w^\pi_{\bupsilon, h} =  \begin{bmatrix} \btheta_h \\ \int V^\pi_{\bupsilon, h+1}(\x') d\bmu(\x')d\x' \end{bmatrix}$. Therefore, since $\lVert \btheta_h \rVert \leq \sqrt{d}$ and $\lVert \int V^\pi_{\bupsilon, h+1}(\x') d\bmu(\x') \rVert \leq H\sqrt{d}$, the result follows. 
\end{proof}
\begin{lemma}[Bound on Weights of \texttt{\textbf{CoopLSVI}} Policy for MDPs]
\label{lem:bound_homo_algo_weight}
At any $t \in [T]$ for any $m \in \cM$ and all $h \in [H]$, we have that the weights $\w^t_{m, h}$ of Algorithm~\ref{alg:ind_homo} satisfy,
\begin{align*}
    \lVert \w^t_{m, h} \rVert_2 \leq 2H\sqrt{dMt/\lambda}.
\end{align*}
\end{lemma}
\begin{proof}
For any vector $\bv \in \bbR^d | \lVert \bv \rVert = 1$,
\begin{align}
    \left| \bv^\top\widehat\btheta^t_{m, h} \right| &= \left|\bv^\top\left(\bLambda^t_{m, h}\right)^{-1}\left(\sum_{\tau=1}^{U^m_h(t)}\left[\bphi(x_\tau, a_\tau)\left[r(x_\tau, a_\tau) + \max_{a} Q_{m, h+1}(x'_\tau, a)\right]\right]\right) \right| \\
    &\leq 2H\cdot\left|\bv^\top\left(\bLambda^t_{m, h}\right)^{-1}\left(\sum_{\tau=1}^{U^m_h(t)}\bphi(x_\tau, a_\tau)\right) \right| \\
    &\leq 2H\cdot\sqrt{\left|\left(\sum_{\tau=1}^{U^m_h(t)}\lVert \bv\rVert^2_{\left(\bLambda^t_{m, h}\right)^{-1}}\lVert  \bphi(x_\tau, a_\tau)\rVert^2_{\left(\bLambda^t_{m, h}\right)^{-1}}\right) \right|} \\
    &\leq 2H \lVert \bv \rVert \sqrt{dU^m_h(t)/\lambda} \leq 2H\sqrt{dMt/\lambda}.
\end{align}
The penultimate inequality follows from Lemma~\ref{lem:variance_sum} and the final inequality follows from the fact that $U^m_h(t) \leq Mt$. The remainder of the proof follows from the fact that for any vector $\w, \lVert \w \rVert = \max_{\bv: \lVert \bv \rVert = 1} |\bv^\top \w|$.
\end{proof}

\begin{lemma}[Bound on Weights in MMDP \textbf{\texttt{CoopLSVI}}]
\label{lem:bound_mg_algo_weights}
For any $t \in [T], h \in [H], \bupsilon \in \bUpsilon$, the weight $\w^t_{\bupsilon, h}$ in \textbf{\texttt{CoopLSVI}} in the multiagent MDP satisfies,
\begin{align*}
    \lVert \w^t_{\bupsilon, h} \rVert_2 \leq 2HM\sqrt{dt/\lambda}.
\end{align*}
\end{lemma}
\begin{proof}

For any vector $\bv \in \bbR^d | \lVert \bv \rVert = 1$,
\begin{align}
    \left| \bv^\top \w^t_{\bupsilon, h} \right| &= \left|\bv^\top\left(\bLambda^t_{ h}\right)^{-1}\left(\sum_{\tau=1}^{k_t}\left[\bPhi(\x_\tau, \ba_\tau)\left[\br_h(\x_\tau, \ba_\tau) + \max_{\ba \in \cA} Q_{\bupsilon, h+1}(\x'_\tau, \ba)\right]\right]\right) \right| \\
    &\leq \sqrt{k_t\cdot \sum_{\tau=1}^{k_t}\left(\bv^\top\left(\bLambda^t_{ h}\right)^{-1}\left[\bPhi(\x_\tau, \ba_\tau)\left[\br_h(\x_\tau, \ba_\tau) + \max_{\ba \in \cA} Q_{\bupsilon, h+1}(\x'_\tau, \ba)\right]\right]\right)^2} \\
    &\leq HM\sqrt{k_t\cdot \sum_{\tau=1}^{k_t}\left\lVert\bv^\top\left(\bLambda^t_{ h}\right)^{-1}\bPhi(\x_\tau, \ba_\tau)\right\rVert_2^2} \\
    &\leq 2HM\sqrt{k_t\cdot\sum_{\tau=1}^{k_t}\lVert \bv\rVert^2_{\left(\bLambda^t_{h}\right)^{-1}}\lVert  \bPhi(\x_\tau, \ba_\tau)\rVert^2_{\left(\bLambda^t_{h}\right)^{-1}}} \\
    &\leq 2HM \lVert \bv \rVert \sqrt{dk_t/\lambda} \leq 2HM\sqrt{dt/\lambda}.
\end{align}
The penultimate inequality follows from Lemma~\ref{lem:variance_sum} and the final inequality follows from the fact that $k_t \leq t$. The remainder of the proof follows from the fact that for any vector $\w, \lVert \w \rVert = \max_{\bv: \lVert \bv \rVert = 1} |\bv^\top \w|$.
\end{proof}

\section{Martingale Concentration Bounds}
\begin{lemma}[Lemma E.2 of~\cite{yang2020provably}, Lemma D.4 of~\cite{jin2018q}]
Let $\{x_\tau\}_{\tau=1}^\infty$ and $\{\bphi_\tau\}_{\tau=1}^\infty$ be an $\cS$-valued and an $\cH$-valued stochastic process adapted to filtration $\{\cF_\tau\}_{\tau=0}^\infty$ respectively, where we assume that $\lVert \bphi_\tau \rVert_2 \leq 1$ for all $\tau \geq 1$. Besides, for any $t \geq 1$, define $\bLambda_t : \cH \rightarrow \cH$ as $\bLambda_t = \lambda \bI_d + \sum_{\tau=1}^t \bphi_\tau\bphi_\tau^\top$ with $\lambda > 1$. Then, for any $\delta > 0$ with probability at least $1-\delta$, we have,
\begin{align*}
    &\sup_{V \in \cV}\left\lVert \sum_{\tau=1}^t \bphi_\tau\left\{ V(x_\tau) - \bbE[V(x_\tau) | \cF_{\tau-1}]\right\} \right\rVert^2_{\bLambda_t^{-1}}  \\ &\leq 4H^2\cdot\log\frac{\det\left(\bLambda_t\right)}{\det\left(\lambda\bI_d\right)} + 4H^2t(\lambda-1) + 8H^2\log\left(\frac{|\cN_\epsilon|)}{\delta}\right) + \frac{8t^2\epsilon^2}{\lambda}.
\end{align*}
\label{lem:self_normalized_single_task}
\end{lemma}

\subsection{Multi-task concentration bound \citep{chowdhury2020no}}
We assume the multi-agent kernel $\bGamma$ to be continuous relative to the operator norm on $\cL(\bbR^n)$, the space of bounded linear operators from $\bbR^n$ to itself (for some $n > 0$). Then the RKHS $\cH_\bGamma(\cX^n)$ associated with the kernel $\bGamma$ is a subspace of the space of continuous functions from $\cX^n$ to $\bbR^n$, and hence, $\bGamma$ is a Mercer kernel. Let $\mu$ be a measure on the (compact) set $\cX^n$. Since $\bGamma$ is a Mercer kernel on $\cX$ and $\sup_{\bX \in \cX^n} \lVert \bGamma(\bX, \bX) \rVert < \infty$, the RKHS $\cH_\bGamma(\cX^n)$ is a subspace of $L^2(\cX^n, \mu; \bbR^n)$, the Banach space of measurable functions $g : \cX^n \rightarrow \bbR^n$ such that $\int_{\cX^n} \lVert g(\bX) \rVert^2 d\mu(\bX) < \infty$, with norm $\lVert g \rVert_{L^2} = \left( \int_{\cX^n} \lVert g(\bX) \rVert^2 d\mu(\bX). \right)^{1/2}$. Since $\bGamma(\bX, \bX) \in \cL(\bbR^n)$ is a compact operator, by the Mercer theorem

We can therefore define a feature map $\Phi : \cX^M \rightarrow \cL(\bbR^n, \ell^2)$ of the multi-agent kernel $\Gamma$ by
\begin{align}
    \Phi(\bX)^\top \y = \left(\sqrt{\nu_1}\psi_1(\x_1)^\top\y, \sqrt{\nu_2}\psi_2(\x_2)^\top\y, ..., \sqrt{\nu_M}\psi_M(\x_M)^\top\y\right), \ \forall \bX \in \cX^M, \y \in \bbR^m.
\end{align}
We then obtain $F(\bX) = \Phi(\bX)^\top \btheta^\star$ and $\Gamma(\bX, \bX') = \Phi(\bX)^\top\Phi(\bX') \ \forall \ \bX, \bX' \in \cX^M$.

Define $\bS_t = \sum_{\tau=1}^t \Phi(\bX_\tau)^\top \varepsilon_\tau$, where $\varepsilon_1, ..., \varepsilon_t$ are the noise vectors in $\bbR^M$. Now consider $\cF_{t-1}$, the $\sigma$-algebra generated by the random variables $\{\bX_\tau, \varepsilon_\tau\}_{\tau=1}^{t-1}$ and $\bX_t$. We can see that $\bS_t$ is $\cF_t$-measurable, and additionally, $\bbE[\bS_t | \cF_{t-1}] = \bS_{t-1}$. Therefore, $\left\{\bS_{t}\right\}_{t \geqslant 1}$ is a martingale with outputs in $\ell^2$ space. Following \cite{chowdhury2020no}, consider now the map $\Phi_{\cX_t} : \ell^2 \rightarrow \bbR^{Mt}$:
\begin{align}
    \Phi_{\cX_t}\btheta = \left[\left(\Phi(\bX_1)^\top\btheta\right)^\top, \left(\Phi(\bX_1)^\top\btheta\right)^\top, ..., \left(\Phi(\bX_t)^\top\btheta\right)^\top\right]^\top, \ \forall \ \btheta \in \ell^2.
\end{align}
Additionally, denote $\bV_t := \Phi_{\cX_t}^\top\Phi_{\cX_t}$ be a map from $\ell^2$ to itself, with $\bI$ being the identity operator in $\ell^2$. We have the following result from~\cite{chowdhury2020no} that provides us with a self-normalized martingale bound.
\begin{lemma}[Lemma 3 of~\cite{chowdhury2020no}]
Let the noise vectors $\left\{\bepsilon_t\right\}_{t \geqslant 1}$ be $\sigma$-sub-Gaussian. Then, for any $\eta > 0$ and $\delta \in (0, 1]$, with probability at least $1-\delta$, the following holds uniformly over all $t \geqslant 1$:
\begin{align*}
    \left\lVert \bS_t \right\rVert_{\left(\bV_t + \eta \bI\right)^{-1}} \leqslant \sigma\sqrt{2\log(1/\delta) + \log\det(\bI + \eta^{-1}\bV_t)}.
\end{align*}
Alternatively stated, we have again that  with probability at least $1-\delta$, the following holds uniformly over all $t \geqslant 1$:
\begin{align*}
    \left\lVert \bEpsilon_t \right\rVert_{\left(\left(\bK_t + \eta \bI\right)^{-1} + \bI\right)^{-1}}^2 \leqslant 2\sigma^2\log\left[\frac{\sqrt{\det(\bI(1+\eta) + \bK_t)}}{\delta}\right].
\end{align*}
\label{lem:self_normalized_martingale_multi_task}
\end{lemma}

\section{Covering Number Bounds}
\begin{lemma}[Covering Number of the Euclidean Ball]
\label{lem:covering_euclidean_ball}
For any $\varepsilon > 0$, the $\varepsilon-$covering number of the Euclidean ball in $\bbR^d$ with radius $R > 0$ is less than $(1+2R/\varepsilon)^d$.
\end{lemma}

\begin{lemma}[Covering Number for UCB-style value functions, Lemma D.6 of~\cite{jin2020provably}]
\label{lem:covering_ind_homo}
Let $\cV$ denote a class of functions mapping from $\cS$ to $\bbR$ with the following parameteric form
\begin{align*}
    V(\cdot) = \min\left\{\max_{a \in \cA} \left[\w^\top\bphi(\cdot, a) + \beta\sqrt{\bphi(\cdot, a)^\top \bLambda^{-1} \bphi(\cdot, a)}\right], H \right\},
\end{align*}
where the parameters $(\w, \beta, \bLambda)$ are such that $\lVert \w \rVert \leq L$, $\beta \in (0, B]$, $\lVert \bphi(x, a) \rVert \leq 1 \ \forall (x, a) \in \cS \times \cA$, and the minimum eigenvalue of $\bLambda$ satisfies $\lambda_{\min}(\bLambda) \geq \lambda$. Let $\cN_\varepsilon$ be the $\varepsilon-$covering number of $\cV$ with respect to the distance $\text{dist}(V, V') = \sup_{x \in \cS} |V(x) - V'(x)|$. Then,
\begin{align*}
    \log \cN_\varepsilon \leq d \log\left(1+4L/\varepsilon\right) + d^2 \log\left(1 + 8d^{1/2}B^2/(\lambda\epsilon^2)\right).
\end{align*}
\end{lemma}

\begin{lemma}[Covering number for multiagent MDP UCB-style functions]
\label{lem:covering_mg}
Let $\cV$ denote a class of functions mapping from $\cS$ to $\bbR$ with the following parameteric form
\begin{align*}
    \bv_\bupsilon(\cdot) = {\bm 1}_M\cdot\min\left\{\max_{\ba \in \cA}\left[\langle\bupsilon, \bv(\cdot, \ba)\rangle + \beta\left\lVert\bPhi(\cdot, \ba)^\top \bLambda^{-1} \bPhi(\cdot, \ba)\right\rVert\right], H \right\}, \bv(\cdot, \ba) = \w^\top\bPhi(\cdot, \ba)
\end{align*}
where the parameters $(\w, \beta, \bLambda)$ are such that $\w \in \bbR^d, \lVert \w \rVert_2 \leq L$, $\beta \in (0, B]$, $\lVert \bPhi(\x, \ba) \rVert \leq \sqrt{M} \ \forall (\x, \ba) \in \cS \times \cA$, and the minimum eigenvalue of $\bLambda$ satisfies $\lambda_{\min}(\bLambda) \geq \lambda$. Let $\cN_\varepsilon$ be the $\varepsilon-$covering number of $\cV$ with respect to the distance $\text{dist}(\bv, \bv') = \sup_{\x \in \cS, \bupsilon \in \bUpsilon } |\bv_\bupsilon(\x) - \bv_\bupsilon'(\x)|$. Then,
\begin{align*}
    \log \left(\cN_\varepsilon\right) \leq d\cdot\log\left(1 + \frac{4LM^{2}}{\varepsilon}\right) + d^2\log\left(1 + \frac{8Md^{1/2}B^2}{\lambda\varepsilon^2}\right).
\end{align*}
\end{lemma}
\begin{proof}
We have that for two matrices $\bA_1 = \beta^2\bLambda_1^{-1}, \bA_2=  \beta^2\bLambda_2^{-1}$ and weight matrices $\w_1, \w_2 \in \bbR^{d}$, by a strategy similar to that of Lemma~\ref{lem:covering_ind_homo},
\begin{align}
    &\sup_{\bupsilon \in \bUpsilon, \x \in \cS}\left| \bv_\bupsilon(\x)-\bv'_\bupsilon(\x) \right|_1 \\
    &=  M\cdot\sup_{\x \in \cS, \bupsilon \in \bUpsilon}\left| \bupsilon^\top\bv(\x)-\bupsilon^\top\bv'(\x) \right|\\
    &\leq  M\cdot\sup_{\x \in \cS}\left| \bv(\x)-\bv'(\x) \right|_1\\
    &\leq  M\cdot\sup_{\x \in \cS, \ba \in \cA}\left| \left[\w_1^\top\bPhi(\cdot, \ba) + \left\lVert\bPhi(\cdot, \ba)^\top \bA_1 \bPhi(\cdot, \ba)\right\rVert_2\right]-\left[\w_2^\top\bPhi(\cdot, \ba) + \left\lVert\bPhi(\cdot, \ba)^\top \bA_2 \bPhi(\cdot, \ba)\right\rVert_2\right] \right|_1\\
    &\leq  M\cdot\sup_{\x \in \cS, \ba \in \cA}\left| \left(\w_1-\w_2\right)^\top\bPhi(\cdot, \ba) + \left\lVert\bPhi(\cdot, \ba)^\top \bA_1 \bPhi(\cdot, \ba)\right\rVert_2- \left\lVert\bPhi(\cdot, \ba)^\top \bA_2 \bPhi(\cdot, \ba)\right\rVert_2 \right|_1\\
    &\leq  M\cdot\sup_{\x \in \cS, \ba \in \cA}\left| \left(\w_1-\w_2\right)^\top\bPhi(\cdot, \ba)\right|_1 +M\cdot\sup_{\x \in \cS, \ba \in \cA}\left|  \left\lVert\bPhi(\cdot, \ba)^\top \bA_1 \bPhi(\cdot, \ba)\right\rVert_2- \left\lVert\bPhi(\cdot, \ba)^\top \bA_2 \bPhi(\cdot, \ba)\right\rVert_2 \right|\\
    &\leq  M\cdot\sup_{\x \in \cS, \ba \in \cA}\left| \left(\w_1-\w_2\right)^\top\bPhi(\cdot, \ba)\right|_1 +M\cdot\sup_{\x \in \cS, \ba \in \cA}\left\lVert\bPhi(\cdot, \ba)^\top \left(\bA_1-\bA_2\right) \bPhi(\cdot, \ba)\right\rVert_2 \\
    &\leq  M^{3/2}\cdot\sup_{\bPhi : \lVert \bPhi \rVert \leq \sqrt{M}} \left[\left\lVert \left(\w_1 - \w_2\right)^\top\bPhi\right\rVert_2\right] + M\cdot\sup_{\bPhi : \lVert \bPhi \rVert \leq \sqrt{M}}\left\lVert\bPhi^\top \left(\bA_1-\bA_2\right) \bPhi\right\rVert_2 \\
    &\leq  M^{2}\cdot\left\lVert \w_1 - \w_2\right\rVert_2 + M^{2}\left\lVert\bA_1 - \bA_2\right\rVert_2 \\
    &\leq  M^{2}\cdot\left\lVert \w_1 - \w_2\right\rVert_2 + M^{2}\left\lVert\bA_1 - \bA_2\right\rVert_F \\
\end{align}
Now, let $\cC_{\w}$ be an $\varepsilon/(2M^{2})$ cover of $\left\{\w \in \bbR^{d} \ \big| \ \lVert \w \rVert_2 \leq L\right\}$ with respect to the Frobenius-norm, and $\cC_\bA$ be an $\varepsilon^2/4$ cover of $\left\{ \bA \in \bbR^{d \times d} | \lVert \bA \rVert_F \leq (M^2d)^{1/2}B^2\lambda^{-1}\right\}$ with respect to the Frobenius norm. By Lemma~\ref{lem:covering_euclidean_ball} we have,
\begin{align}
    |\cC_{\w}| \leq (1 + 4LM^{2}/\varepsilon)^{d}, |\cC_\bA| \leq (1 + 8(M^2d)^{1/2}B^2/(\lambda\varepsilon^2))^{d^2}.
\end{align}
Therefore, we can select, for any $\bv_\bupsilon(\cdot)$, corresponding weight $\w \in \cC_{\w}$, and matrix $\bA \in \cC_\bA$. Therefore, $\cN_\varepsilon \leq |\cC_\bA| \cdot |\cC_{\w}| $. This gives us,
\begin{align}
    \log \left(\cN_\varepsilon\right) \leq d\cdot\log\left(1 + \frac{4LM^{2}}{\varepsilon}\right) + d^2\log\left(1 + \frac{8Md^{1/2}B^2}{\lambda\varepsilon^2}\right).
\end{align}
\end{proof}
\section{Auxiliary Results}
\begin{lemma}[Elliptical Potential, Lemma 3 of~\citet{abbasi2011improved}]
Let $\bm x_1, \bm x_2, ..., \bm x_n \in \mathbb R^d$ be vectors such that $\lVert \bm x \rVert_2 \leq L$. Then, for any positive definite matrix $\bm U_0 \in \mathbb R^{d \times d}$, define $\bm U_t := \bm U_{t-1} + \bm x_t \bm x_t^\top$ for all $t$. Then, for any $\nu > 1$,
\begin{align*}
    \sum_{t=1}^n \lVert \bm x_t \rVert_{\bm U_{t-1}^{-1}}^2 \leq 2d\log_\nu\left(\frac{\text{tr}(\bm U_0) + nL^2}{d \det^{1/d}(\bm U_0)}\right).
\end{align*}
\label{lem:variance_sum}
\end{lemma}

\bibliographystyle{icml2021}
\bibliography{refs}

\end{document}